\documentclass[twoside]{article}

\usepackage[accepted]{aistats2026}
\usepackage[authoryear,round]{natbib}
\setcitestyle{semicolon}              % (A, 2020; B, 2021)
\let\cite\citep
\usepackage[utf8]{inputenc} % allow utf-8 input
\usepackage[T1]{fontenc}    % use 8-bit T1 fonts
\usepackage{hyperref}       % hyperlinks
\usepackage{url}            % simple URL typesetting
\usepackage{booktabs}       % professional-quality tables
\usepackage{amsfonts}       % blackboard math symbols
\usepackage{nicefrac}       % compact symbols for 1/2, etc.
\usepackage{microtype}      % microtypography
\usepackage{xcolor}         % colors
\usepackage{graphicx}
\usepackage{subfig}
\usepackage[inkscapelatex=false]{svg}
\usepackage{multirow}
\usepackage{colortbl} % For coloring tables
\usepackage{xcolor}   % For color management
\usepackage{booktabs} % For professional table formatting
\usepackage{adjustbox} % Optional for advanced resizing
\usepackage{float}
\usepackage{xspace}
\definecolor{rowcolor1}{HTML}{FBE4E4} % Light pink
\newcommand{\INDSTATE}[1][1]{\STATE\hspace{#1em}}
\newcolumntype{C}{>{\cellcolor{red!10}}c}
\usepackage{wrapfig}

\newcommand{\proj}{TSA\xspace}

\usepackage{amsmath,amsfonts,bm, mathtools}

\def\eqref#1{equation~\ref{#1}}
\def\1{\bm{1}}

\def\rmA{{\mathbf{A}}}

\def\rmX{{\mathbf{X}}}
\def\rmY{{\mathbf{Y}}}

\def\mA{{\bm{A}}}

\def\mI{{\bm{I}}}

\def\mX{{\bm{X}}}

\DeclareMathAlphabet{\mathsfit}{\encodingdefault}{\sfdefault}{m}{sl}
\SetMathAlphabet{\mathsfit}{bold}{\encodingdefault}{\sfdefault}{bx}{n}

\def\gE{{\mathcal{E}}}

\def\gG{{\mathcal{G}}}
\def\gH{{\mathcal{H}}}

\def\gL{{\mathcal{L}}}

\def\gN{{\mathcal{N}}}

\def\gS{{\mathcal{S}}}
\def\gT{{\mathcal{T}}}
\def\gU{{\mathcal{U}}}
\def\gV{{\mathcal{V}}}

\def\gX{{\mathcal{X}}}
\def\gY{{\mathcal{Y}}}

\newcommand{\predy}{\hat{Y}}
\newcommand{\prob}{\mathbb{P}}
\newcommand{\probu}{\mathbb{P}^\mathcal{U}}
\newcommand{\probs}{\mathbb{P}^\mathcal{S}}
\newcommand{\probt}{\mathbb{P}^\mathcal{T}}
\newcommand{\probm}{\mathbb{P}^\mathcal{M}}

\newcommand{\errors}{\varepsilon^\mathcal{S}}
\newcommand{\errort}{\varepsilon^\mathcal{T}}
\newcommand{\node}{\mathcal{V}}

\newcommand{\nodes}{\mathcal{V}^\mathcal{S}}
\newcommand{\nodet}{\mathcal{V}^\mathcal{T}}

\newcommand{\neighbor}{\mathcal{N}}
\newcommand{\rvY}{Y}

\newcommand{\gammau}[1]{\gamma^{\mathcal{U}}_{#1}}
\newcommand{\gammas}[1]{\gamma^{\mathcal{S}}_{#1}}
\newcommand{\gammat}[1]{\gamma^{\mathcal{T}}_{#1}}
\newcommand{\rvW}{W}

\newcommand{\omegau}[1]{\omega^{\mathcal{U}}_{#1}}
\newcommand{\omegas}[1]{\omega^{\mathcal{S}}_{#1}}
\newcommand{\omegat}[1]{\omega^{\mathcal{T}}_{#1}}
\newcommand{\omegam}[1]{\omega^{\mathcal{M}}_{#1}}

\newcommand{\piu}[1]{\pi^{\mathcal{U}}_{#1}}
\newcommand{\pis}[1]{\pi^{\mathcal{S}}_{#1}}
\newcommand{\pit}[1]{\pi^{\mathcal{T}}_{#1}}

\newcommand{\balerrorrate}{\textnormal{BER}^{\mathcal{S}}\infdivx}

\DeclarePairedDelimiter{\abs}{\lvert}{\rvert}

\DeclarePairedDelimiter{\rpar}{(}{)}
\DeclarePairedDelimiter{\spar}{[}{]}
\DeclarePairedDelimiter{\cbrace}{\{}{\}}
\DeclarePairedDelimiterX{\infdivx}[2]{(}{)}{  #1\;\delimsize\|\;#2}
\DeclarePairedDelimiter{\set}{\{}{\}}

\newcommand{\E}{\mathbb{E}}

\newcommand{\R}{\mathbb{R}}

\newcommand{\mgamma}{\boldsymbol{\gamma}}

\usepackage{amsmath}
\usepackage{amssymb}
\usepackage{mathtools}
\usepackage{amsthm}
\usepackage{thmtools, thm-restate}
\usepackage{subcaption} % For subfigures
\usepackage{algorithm}
\usepackage{algorithmic}

\theoremstyle{plain}
\newtheorem{theorem}{Theorem}[section]

\newtheorem{lemma}[theorem]{Lemma}

\theoremstyle{definition}
\newtheorem{definition}[theorem]{Definition}

\theoremstyle{remark}

\renewcommand{\paragraph}[1]{\textbf{#1}~~}

\begin{document}


% \runningauthor{Surname 1, Surname 2, Surname 3, ...., Surname n}

\twocolumn[
\aistatstitle{Structural Alignment Improves Graph Test-Time Adaptation}

% \aistatsauthor{\hspace{1cm} Hans Hao-Hsun Hsu\textsuperscript{*}$ \ ^1$ \And Shikun Liu\textsuperscript{*}$ \ ^1$ \And  Han Zhao$\ ^2$ \And Pan Li$\ ^1$}
\aistatsauthor{ Hans Hao-Hsun Hsu\textsuperscript{*} \And Shikun Liu\textsuperscript{*} \And  Han Zhao \And Pan Li}

\aistatsaddress{Georgia Tech \And Georgia Tech \And UIUC \And Georgia Tech} 
]

\begin{abstract}
  Graph-based learning excels at capturing interaction patterns in diverse domains like recommendation, fraud detection, and particle physics. However, its performance often degrades under distribution shifts, especially those altering network connectivity. Current methods to address these shifts typically require retraining with the source dataset, which is often infeasible due to computational or privacy limitations. We introduce Test-Time Structural Alignment (TSA), a novel algorithm for Graph Test-Time Adaptation (GTTA) that adapts a pretrained model to align graph structures during inference without the cost of retraining. Grounded in a theoretical understanding of graph data distribution shifts, TSA employs three synergistic strategies: uncertainty-aware neighborhood weighting to accommodate neighbor label distribution shifts, adaptive balancing of self-node and aggregated neighborhood representations based on their signal-to-noise ratio, and decision boundary refinement to correct residual label and feature shifts. Extensive experiments on synthetic and real-world datasets demonstrate TSA's consistent outperformance of both non-graph TTA methods and state-of-the-art GTTA baselines. Code is available at \href{https://github.com/Graph-COM/TSA}{https://github.com/Graph-COM/TSA}.
\end{abstract}

\vspace{-2mm}
\begin{figure}[t]
    \centering  % Centers the image on the page
    \includegraphics[trim={0.6cm 0 0 0},clip, width=0.95\linewidth]{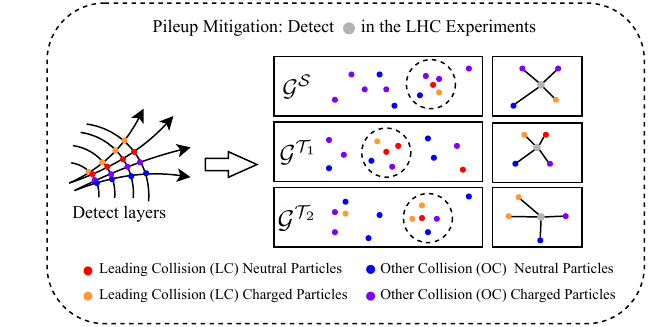} 
    \vspace{-2mm}
    \caption{Example of the distribution shifts of neighborhood information due to different experimental conditions.
    % in the LHC experiments.
    % The goal of pileup mitigation is to detect LC neutral particles.
    % The particles are modeled by kNN graphs (dashed circles in the figure) to use nearby particles for drawing inference.
    Pileup mitigation aims to detect LC neutral particles using kNN graphs (dashed circles) that leverage nearby particles for inference.
    The model is trained on $\gG^\gS$ but needs to generalize to $\gG^{\gT_1}$ and $\gG^{\gT_2}$. 
    The inferred nodes within the circles are the LC neutral particles, but their neighborhood node label ratios change in $\gG^{\gT_1}$ and $\gG^{\gT_2}$.
    % In  $\gG^{\gT_1}$ the homophily ratio changes as one of the neighbor node is an LC neutral particle.
    % In $\gG^{\gT_3}$ the neighborhood node label ratio remains unchanged but the amount of neighboring node changes.
    % Real world distribution shifts in graph data typically present as a combination of these three cases.
    Both cases represent \emph{neighborhood shift}, which this work aims to address.
    A more formal definition of these shifts is discussed in Sec.~\ref{sec:Test Error Analysis}.}
    \label{fig:pileup_example}
    % Link for Pileup figure
    %https://drive.google.com/file/d/1m7N-_RqEChCZ-vo_qYoqkUVxR24-hdoG/view?usp=sharing
    % Link for Pileup neurips https://drive.google.com/file/d/1uNGb4IW_Kj_7Y1kEZcCPoacOebewOkaO/view?usp=sharing
    % shikun drawio link https://app.diagrams.net/#G1qnY3a9sP0vJ37JYVrVi_6bd22wDj-fjG#%7B%22pageId%22%3A%22KIHKL_-jXcYF3asLFJqd%22%7D
    \vspace{-5mm}
\end{figure}

\section{INTRODUCTION}
\label{introduction}
\vspace{-2mm}

%%%%%%%%%%%%%%%%%%%%%%%%%%%%%%%%%%%%%%%%%%%%%%%%%%%%%%%%%%%%%%%%%%%%%%%%%%%%%%%%%%%%%%%%%%%%%%%%%%%%%%%%%%%%%%%%%%%%%%%%%%%%%%%%%%%%%%%%%%%%%%%%%%%%%%
\begin{figure*}[t]
\centering

\includegraphics[trim={0.4cm 0 0 0},clip, width=0.95\textwidth]{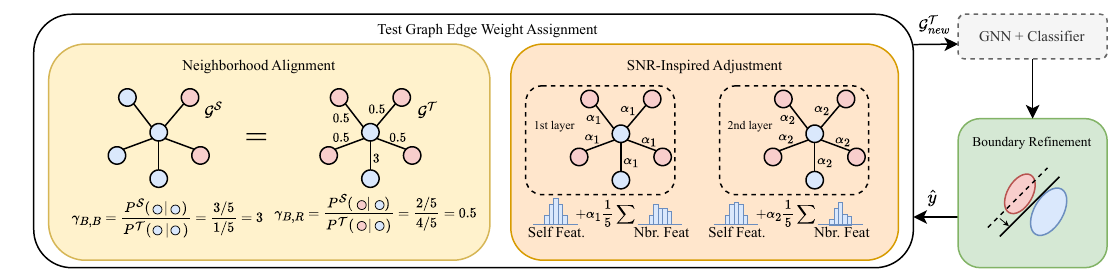}
\vspace{-2mm}
\caption{TSA utilizes neighborhood alignment to address neighborhood shift and SNR-inspired adjustment to mitigate SNR shift. It further adjusts the decision boundary to get refined predictions $\hat{y}$. 
The refined soft pseudo-labels are used to estimate the parameter $\mgamma$ for neighborhood alignment and to optimize $\alpha$ for combining self representations with neighborhood-aggregated representations.
 % The refined soft pseudo-labels $\hat{y}$ provide a more reliable estimation of the parameter $\mgamma$.
 % Reciprocally, better alignment of neighborhood information  $\gG^{\gT}_{new}$ can further refine the decision boundary.
%The pseudo codes are presented in Algorithm \ref{alg:example}.
% https://drive.google.com/file/d/1zDMWqsX2Lwd77fAItY363V0mJtAD5Wag/view?usp=sharing
}
\label{fig:tsa}
\vspace{-4mm}
\end{figure*}
%%%%%%%%%%%%%%%%%%%%%%%%%%%%%%%%%%%%%%%%%%%%%%%%%%%%%%%%%%%%%%%%%%%%%%%%%%%%%%%%%%%%%%%%%%%%%%%%%%%%%%%%%%%%%%%%%%%%%%%%%%%%%%%%%%%%%%%%%%%%%%%%%%%%%%

Graph-based methods have become indispensable in handling structured data across a wide range of real-world applications \cite{duvenaud2015convolutional, bronstein2017geometric, zhang2019deep, stokes2020deep}, achieving significant success when training and testing data originate from similar distributions.
However, these methods struggle to generalize to test-time data in a different domain, where variations in time, location, or experimental conditions result in distinct graph connection patterns. 
Although some literature on graph domain adaptation (GDA)~\cite{wu2020unsupervised,you2023graph,zhu2021shift,liu2023structural} seeks to bridge these gaps by aligning labeled source distributions with target distributions, such approaches are often infeasible in practice due to the cost of retraining.
% and the limited availability of source data at the test time.

Real-world applications may operate under constraints such as lightweight computation, limited storage, and strict privacy requirements, which make reprocessing of source data generally infeasible~\cite{wang2021gnnadvisor, wu2021fedgnn}. For example, Fig.~\ref{fig:pileup_example} illustrates pileup mitigation based on graph learning at the Large Hadron Collider (LHC)~\cite{shlomi2020graph, highfield2008large,miao2024locality}, where the connections between particles differ as experimental conditions change. 
To address this, the detector is expected to be capable of adapting on-the-fly~\cite{li2023semi, komiske2017pileup}, potentially on embedded hardware.
Similarly, in fraud detection for financial networks~\cite{clements2020sequential, wang2021review}, privacy-sensitive raw training data may not be accessed during testing, while transaction patterns may change between different regions~\cite{wang2019semi, dou2020enhancing}.
These scenarios necessitate graph test-time adaptation (GTTA), which enables models to adjust to new test domains without 
% re-accessing the massive source dataset during adaptation or 
incurring the overhead of full retraining.

Accurate node classification in graph-structured data relies heavily on effectively leveraging neighborhood information. Graph neural networks (GNNs) \cite{kipf2016semi, velivckovic2017graph, hamilton2017inductive, hamilton2020graph} have achieved significant success across a wide range of applications by utilizing this neighborhood context, yet they remain sensitive to distribution shifts in node neighborhoods \cite{gui2022good, ji2023drugood}. In the above LHC example illustrated by Fig.~\ref{fig:pileup_example}, distribution shifts lead to distinct neighborhood connections among center particles in the same class, causing a neighborhood shift. %\footnote{\vspace{-6mm} We use \emph{neighborhood shift} for intuition; its formal definition is provided in Def.~\ref{def:strucshift}.}).
Existing methods for test-time adaptation (TTA) often focus on adjusting only the classifier’s final layer \cite{iwasawa2021test}, modifying its outputs \cite{boudiaf2022parameter}, or adapting normalization statistics \cite{wang2020tent, niu2023towards}, primarily targeting image-based tasks with independently labeled data points. Consequently, they may not adequately address neighborhood shifts for node classification in graphs. While recent works on GTTA attempt to address the structural changes using graph augmentation \cite{jin2022empowering}, self-supervised learning \cite{zhang2024fully, mao2024source}, these approaches mainly rely on augumentation heuristics and homophily consistency, failing to fully address changes in neighboring node connections in principle.

In this work, we propose a model-agnostic framework with test-time structural alignment, named \proj, as illustrated in Fig.~\ref{fig:tsa}, and supported by theoretical analysis. We begin by examining the generalization gap in GTTA for a pretrained source model. Our theoretical findings, further validated by empirical studies, 
indicate the importance of aligning the 
% label distributions of neighboring nodes among the same class of nodes 
neighborhood distributions
across the source and target domains.
% 
% Our contributions are summarized as follows:
% 
% In the context of training-time GDA, PairAlign \cite{liu2024pairwise} employs a neighborhood weighting strategy to recalibrate the influence of neighboring nodes during message aggregation for a similar alignment.
% However, extending this approach to GTTA necessitates our more nuanced investigation, leading to the following contributions:
% \shikun{If we want to weaken the contribution of PairAlign, we might not need to mention it here. summarize your 4.1 first paragraph saying structure alignment can be handled through reweighting, but the weight assignment is harder in test time case and then connect to the first bullet point.}
% 
% Firstly, 
To this end, \proj employs a neighborhood weighting strategy to recalibrate the influence of neighboring nodes during message aggregation.
% , thereby enabling structural alignment under GTTA. 
This strategy is a form of first-order alignment, as it corrects for shifts in the neighborhood connection expectation.
However, the weight assignment relies on the knowledge of node labels, which are unavailable in the target graph. Assigning weights based on pseudo labels can
% , in fact, 
degrade performance. To mitigate this, \proj introduces an uncertainty-aware assignment strategy that aligns only node pairs with more reliable test-time pseudo labels.

Additionally, \proj goes beyond aligning first-order statistics by optimizing the test-time combination of self-node representations and neighborhood-aggregated representations based on their second-order signal-to-noise ratio (SNR). 
Our experiments show that neighborhood-aggregated representations tend to have higher SNR and are better denoised when the neighbor set size and the GNN layer depth increase.
Under both conditions, \proj adjusts the combination by assigning greater weight to neighborhood-aggregated representations than to self-node representations to enhance model performance.
% Under both conditions, \proj assigns greater weight to neighborhood-aggregated representations over self-node representations to enhance performance.

% For instance, when the target graph is sparse, \proj reduces reliance on neighborhood messages due to their high variance and low SNR. Similarly, as layer depth increases, it assigns greater weight to progressively denoised neighborhood representations.

Lastly, \proj integrates non-graph TTA methods to refine the decision boundary, mitigating mismatches caused by additional label and feature shifts once neighborhood shifts have been addressed.

We conduct extensive experiments on synthetic and four real-world datasets, including those from high-energy physics and citation networks, with multiple GNN backbones, to demonstrate the effectiveness of TSA. TSA outperforms non-graph TTA baselines by up to 21.25\% on synthetic datasets and surpasses all non-graph baselines. Compared to existing GTTA baselines, TSA achieves an average improvement of 10\% on all real-world datasets.

\vspace{-2mm}

\section{PRELIMINARIES}
\vspace{-2mm}
\subsection{Notations and Problem Setup}
\vspace{-2mm}

We use upper-case letters, such as $Y$ to represent random variables, and lower-case letters, such as $y$ to represent their realizations.
Calligraphic letters, such as $\gY$ denote the domain of random variables.
We use bold capital letters such as $\rmY$ to represent the vectorized counterparts, i.e., collections of random variables.
The probability of a random variable $Y$ for a realization is expressed as $\prob(Y=y)$.

% CheatSheet
% \rx random variable x
% \rvx random vector x
% \ervx element of random vector x
% \rmX random matrix X
% \ermX element of random matrix X
% \vx vector x
% \evx elements of vector x
% \mX matrix X

\paragraph{Graph Neural Networks (GNNs).} 
We let $\gG = (\gV,\gE)$ denote an undirected and unweighted graph with the symmetric adjacency matrix $\mA\in\R^{N\times N}$ and the node feature matrix $\mX=[x_1,\dots,x_N]^T$.
GNNs utilize neighborhood information by encoding $\mA$ and $\mX$ into node representations $\{h_{v}^{(k)}, v \in \neighbor_u\}$.
With $h_u^{(1)}=x_u$, the message passing in standard GNNs for node $u$ and each layer $k \in [L] \coloneqq \{1,\dots,L\}$ can be written as
\begin{equation}
    h_u^{(k+1)} = \text{UPT}(h_u^{(k)}, \text{AGG}(\{h_{v}^{(k)}, v \in \neighbor_u\})) \label{eq:gnn}
\end{equation}
where $\neighbor_u$ denotes the set of neighbors of node $u$, and $|\neighbor_u|$ represents the node degree $d_u$.
The AGG function aggregates messages
from the neighbors, and the UPT function updates the node
representations. 

\paragraph{Test-Time Adaptation (TTA).}
Assume the model consists of a feature encoder $\phi:\gX\rightarrow \gH$ and a classifier $g: \gH \rightarrow \gY$ and each domain $\gU \in \{\gS,\gT\}$ consists of a joint feature and label distribution $\probu(X,Y)$.
In TTA, the model is first trained on a labeled source dataset $\{(x_i,y_i)\}_{i=1}^{N}$, IID sampled from the source domain.
In the target dataset $\{x_i\}_{i=1}^{M}$, the objective is to minimize the test error $\errort(g\circ\phi)=\probt(g(\phi(X))\neq Y)$ 
given the source-pretrained model.
% accessing the source dataset \emph{during adaptation}.

% \paragraph{Graph Test-Time Adaptation (GTTA).} 
% Assume the model consists of a feature encoder $\phi:\gX\rightarrow \gH$ and a a classifier $g: \gH \rightarrow \gY$ and each domain $\gU \in \{\gS,\gT\}$ consists of a joint feature and label distribution $\probu(X,Y)$.
% In TTA, the model is first trained on the source domain $\probs(X,Y)$. The objective is to adapt the model to minimize the target test error $\errort(g\circ\phi)=\probt(g(\phi(X))\neq Y)$ \emph{without} accessing the source domain.

% We focus on the node classification tasks in GTTA. 
% Similar to TTA, the model is trained on source graph $\gG^{\gS}=(\gV^{\gS}, \gE^{\gS})$ and the goal is to enhance model performance on test graph $\gG^{\gT}=(\gV^{\gT}, \gE^{\gT})$.
% The encoder $\phi$ is a switched to graph based method such as a GNN.
% For node classification tasks, we aim to minimize the test error $\errort(g\circ\phi)=\probt(g(\phi(X_u, \rmA))\neq Y_u)$ \emph{without} accessing $\gG^{\gS}$.

\paragraph{Graph Test-Time Adaptation (GTTA).} 
Assume the model consists of a GNN feature encoder $\phi:\gX\rightarrow \gH$ and a classifier $g: \gH \rightarrow \gY$. The model is trained on source graph $\gG^{\gS}=(\gV^{\gS}, \gE^{\gS})$ with node labels $\mathbf{y}^\gS$ and the goal is to enhance model performance on a test graph $\gG^{\gT}=(\gV^{\gT}, \gE^{\gT})$ with distribution shifts that will be defined in Sec.~\ref{sec:shift}. For node classification tasks, we aim to minimize the test error $\errort(g\circ\phi)=\probt(g(\phi(X_u, \rmA))\neq Y_u)$ 
given the source-pretrained model.

%%%%%%%%%%%%%%%%%%%%%%%%%%%%%%%%%%%%%%%%%%%%%%%%%%%%%%%%%
\begin{figure*}[t]
\centering
\begin{tabular}{c|ccc}
(a) Reliability of $\gamma$ 
% &\multicolumn{2}{|c|}{ (b) Decision Boundary Refinement}
&(b) Source &(c) Target with Stru. Shift
& (d) Nbr. Align   \\
\includegraphics[width=0.22\textwidth]{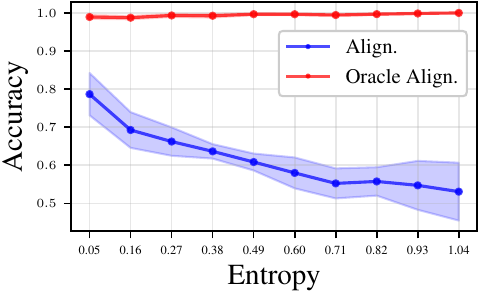} &
\includegraphics[width=0.21\textwidth]{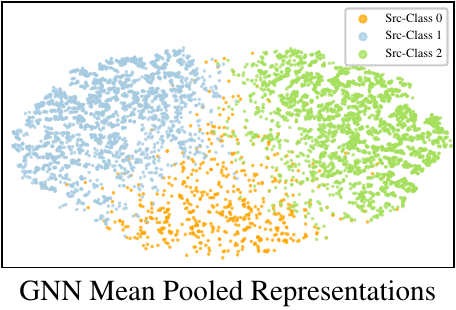} &
\includegraphics[width=0.21\textwidth]{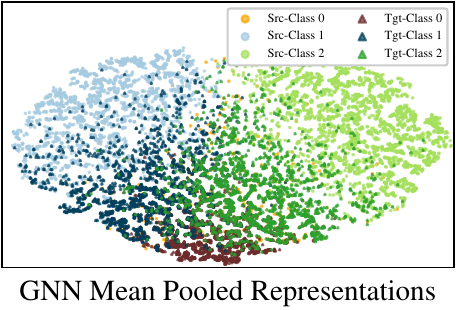} &
\includegraphics[width=0.21\textwidth]{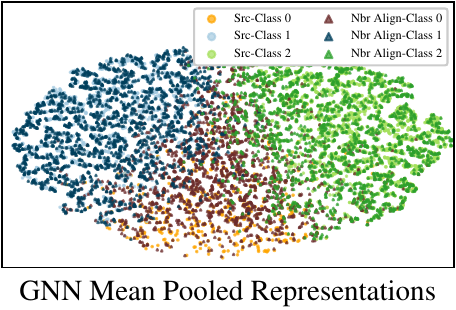} \\

\includegraphics[width=0.22\textwidth]{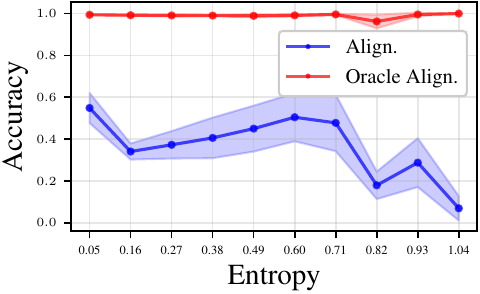} &
\includegraphics[width=0.21\textwidth]{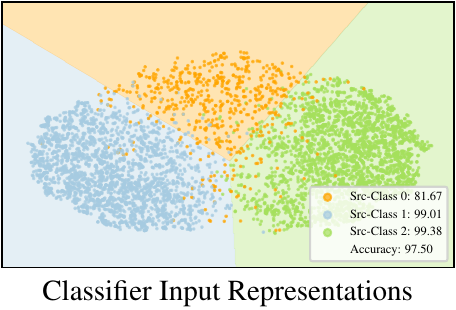} &
\includegraphics[width=0.21\textwidth]{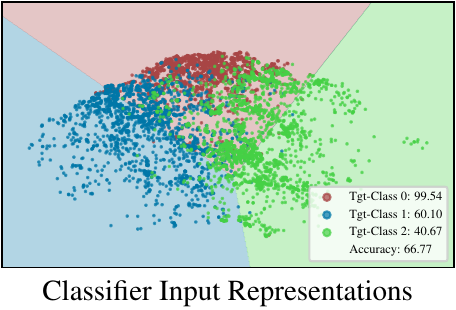} &
\includegraphics[width=0.21\textwidth]{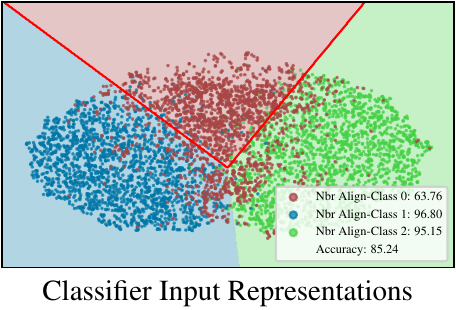} \\
\end{tabular} 
\vspace{-2mm}
\caption{(a) Comparison of neighborhood alignment with $\mgamma$ from model prediction and Oracle on the CSBM graphs~\cite{deshpande2018contextual}. The top (or bottom) subfigures represent the assignment under neighbor shift (or neighbor shift plus label shift, respectively). Nodes are grouped by the entropy of their soft pseudo labels and the y axis shows the accuracy after assigning $\mgamma$. Ideally, a correct assignment (red) would lead to near-perfect accuracy.
However, the assignment based on pseudo labels is far from optimal (blue).
%(b) and (c) top \pan{} represent the t-SNE visualization of the node representations. (b) and (c) bottom visualize the node representations before passing through the classifier with its decision boundary on CSBM. 
Figure (b) to (d) present t-SNE visualizations of node representations during GTTA.
The model is trained on the source domain in (b) (CSBM with label distribution $[0.1, 0.3, 0.6]$) and evaluated on the target domain in (c) (CSBM with label distribution $[0.3, 0.3, 0.3]$). Figure (d) illustrates the qualitative result of the neighborhood alignment.
% Note that class 0 is 
Despite well-aligned neighborhood representations, a mismatched decision boundary (red) leads to class 0 misclassifications.
Node colors represent the ground-truth labels. The top subfigures show the output given by the GNN encoder, while the bottom subfigures show the classifier decision boundaries. 
}
\label{fig:pa_vis}
\vspace{-3mm}
\end{figure*}
\vspace{-3mm}

%%%%%%%%%%%%%%%%%%%%%%%%%%%%%%%%%%%%%%%%%%%%%%%%%%%%%%%%%
\subsection{Related Works}
% \hans{Feature Adaptation to do:}
% Baseline to choose
% (1) ActMAD CVPR 2023
% (2) 

\paragraph{Test-time Adaptation. }
%We categorize TTA into three categories based on their problem setup.
Test-time training (TTT) approaches necessitate modifications to both the training pipeline and network architecture to accommodate customized self-supervised tasks, such as rotation prediction~\cite{sun2020test}, contrastive learning~\cite{liu2021ttt, bartler2022mt3}, and clustering~\cite{hakim2023clust3}. These requirements, however, can limit their broader applicability. In contrast, our proposed setup aligns with the test-time adaptation (TTA) category, which does not alter the original training pipeline~\cite{liang2024comprehensive}. One prominent TTA method, Tent~\cite{wang2020tent}, adapts batch normalization parameters by minimizing entropy. This approach has inspired subsequent research focusing on normalization layer adaptation~\cite{gong2022note, zhao2023delta, lim2023ttn, niu2023towards}. Other TTA strategies directly modify the classifier's predictions. For instance, LAME~\cite{boudiaf2022parameter} refines the model's soft predictions by applying regularization in the feature space. T3A~\cite{iwasawa2021test} classifies test data based on its distance to pseudo-prototypes derived from pseudo-labels. Building upon T3A, TAST~\cite{jang2022test} and PROGRAM~\cite{sunprogram} aim to construct more reliable prototype graphs. However, the effectiveness of these methods can be compromised when significant domain shifts occur~\cite{mirza2023actmad, Kang_2023_ICML}. 
To address larger domain shifts, some methods leverage stored, lightweight source statistics, such as label distributions, to perform test-time alignment~\cite{mirza2023actmad, niu2024test, kojima2022robustvit, eastwood2021source, adachi2024test, zhou2023ods}. ActMAD~\cite{mirza2023actmad}, for example, aligns target feature statistics with stored source statistics across multiple layers. More recently, FOA~\cite{niu2024test} focuses on aligning the statistics of classification tokens in vision transformers. Note that all the aforementioned methods are primarily designed for image-based applications and may not adequately address shifts in the neighborhood information characteristic of graphs.

\paragraph{Graph Test-time Adaptation.} Studies on GTTA are in two categories - node and graph classification.
% Existing work on graph classification 
Graph classification problems can treat each graph as an i.i.d.\ input, allowing more direct extension of image-based TTA techniques to graphs \cite{chen2022graphtta, wang2022test}.
Our work focuses on node classification \cite{jin2022empowering, zhang2024fully, mao2024source, ju2024graphpatcher}. %, where the connection patterns introduce unique challenges.
GTrans \cite{jin2022empowering} proposes to augment the target graph at test time by optimizing a contrastive loss by generating positive views from DropEdge \cite{rong2019dropedge} and negative samples from the features of the shuffled nodes \cite{velivckovic2018deep}. 
HomoTTT~\cite{zhang2024fully} combines self-supervised learning with homophily-based DropEdge and prediction selection.
% GraphPatcher \cite{ju2024graphpatcher} learns to generate virtual neighbors to improve low-degree nodes classification. 
SOGA \cite{mao2024source} designs a self-supervised loss that relies on mutual information maximization and the homophily assumption.
However, these works address the symptoms of the distribution shift (i.e., poor representations) by transforming the test graph or model parameters using a self-supervised contrastive loss, which relies on the choice of augmentation heuristics. In contrast, TSA is the first method that directly targets the root cause (the statistical change in connectivity patterns) to address test-time graph structure shift in a principled way, grounded in formal error analysis.

\vspace{-3mm}
\section{TEST ERROR ANALYSIS}
\label{sec:Test Error Analysis}
\vspace{-3mm}

% In this section, we first formally define the distribution shift on graphs, categorizing it into feature shift, label shift, and conditional structure shift (CSS).
% Addressing CSS can be reduced to addressing neighborhood shift in terms of first-order alignment \pan{too technical, in the expectation sense?, or not saying this. Later, say the missing part} under GNN mean pooling.
% Next, we examine how these shifts impact the generalization gap in GTTA given a pretrained source model.
% Lastly, we conduct an empirical investigation to corroborate and provide deeper insights into our theoretical findings, discussing additional second-order adjustment \pan{variance} strategies that are necessary but not covered in the theoretical analysis.
In this section, we characterize the generalization error between the source and target graphs and explicitly attribute it to different types of shifts in graph data, which further motivate our \proj algorithm. % to minimize the across-graph generalization error. 
% how different shifts between the source and target graph can impact the generalization gap. Our error decomposition echoes the findings in previous GDA works, but in a more formal bound \han{Bound is always formal. Can we claim our bound to be tighter? More general?}, which further motivates the \proj algorithm design. 

\subsection{Distribution Shifts on Graphs} 
\label{sec:shift}
Distribution shifts on graphs have been formally studied in previous GDA works \cite{wu2020unsupervised, you2023graph, zhu2021shift, wu2022handling, liao2021information, zhu2023explaining}.  
Following their definitions, we categorize shifts in graphs into two types: feature shift and structure shift. For simplicity, our analysis is based on a data generation process: $\rmX \leftarrow \rmY \rightarrow \rmA$,  
where the graph structure and node features are both conditioned on node labels.

\begin{definition}
\label{def:attrshift}
(Feature Shift). Assume node features $x_u, u \in \gV$ are i.i.d sampled given labels $y_u$, then we have $\prob(\rmX|\rmY)=\prod_{u\in\gV}\prob(X_u|Y_u)$. We then define feature shift as $\probs(X_u|Y_u)\neq \probt(X_u|Y_u)$.
\end{definition}

\begin{definition}
\label{def:strucshift}
(Structure Shift). As the graph structure involves the connecting pattern between nodes including their labels, we consider the joint distribution of the adjacency matrix and labels $\prob(\rmA,\rmY)$, where \emph{structure shift}, denoted by $\probs(\rmA,\rmY) \neq \probt(\rmA,\rmY)$, can be decomposed into \emph{conditional structure shift (CSS)} $\probs(\rmA|\rmY) \neq \probt(\rmA|\rmY)$ and \emph{label shift (LS)} $\probs(\rmY) \neq \probt(\rmY)$. Here, CSS  reflects the shift in neighborhood information (neighborhood shift) in Fig.~\ref{fig:pileup_example}.

\end{definition}

% While feature shift and label shift have been widely studied  in non-graph literature, CSS is specific to graph data as it manifest the connection discrepancy given the same label. 
% Notably, in real-world datasets, these shifts often coexist \cite{li2022out, zhang2024survey}.

% According to Theorem 3.3 in \cite{liu2024pairwise}, the sufficient condition for mitigating the influence of CSS under GNNs lies in aligning \emph{degree shift} $\probs(D_u | Y_u)\neq\probt(D_u | Y_u)$, where the node degrees within the
% same label differs, and \emph{neighborhood shift} $\probs(Y_v|Y_u, v \in\neighbor_u)\neq\probt(Y_v|Y_u, v \in\neighbor_u)$, where the neighboring node label connections
% within the same class differs.

\subsection{Theoretical Analysis}

Let $\errors(g \circ \phi)$ denote the error of a pretrained GNN on the source domain and $\errort(g \circ \phi)$ the error of the model when applied to a test graph in the target domain.
Inspired by \citet{tachet2020domain}, we provide an upper bound on the error gap between the source and target domains, showing how a pretrained GNN (e.g., ERM) can be influenced. 
% \vspace{-1mm}
%\begin{definition} 
We denote the \emph{balanced error rate} of a pretrained node predictor $\predy_u$ on the source domain as
%\label{sber definition}
    %\begin{align*}
        $\balerrorrate{\predy_u}{\rvY_u} \coloneq \max_{j \in \mathcal{Y}} \probs(\predy_u\neq\rvY_u | \rvY_u=j).$
    %\end{align*}
%\end{definition}

% \vspace{-1mm}
\begin{restatable}[Error Decomposition Theorem (Informal)]{theorem}{decomposeerror}
\label{theory:decomposeerror}
Suppose $\gG^{\gS}$ and $\gG^{\gT}$, we can decouple both graphs into independent ego-networks (center nodes and 1-hop neighbors).
For any classifier $g$ with a GNN encoder $\phi$ in node classification tasks, we have the following upper bound on the error gap between source and target under feature shift and structure shift: 
{\small
\begin{align*}
&\abs{\errors(g \circ \phi) - \errort(g \circ \phi)} \\ 
&\leq   
\balerrorrate{\predy_u}{\rvY_u}\cdot
    \cbrace{\underbrace{
    \rpar{2\text{TV}(\probs(\rvY_u), \probt(\rvY_u)}
    }_{\text{Label Shift}}
   \\
    &+\underbrace{\E_{Y_u}^\gT\spar{\max_{k\in \gY} \abs{1-\frac{\probt(Y_v=k|Y_u,v\in \neighbor_u)}{\probs(Y_v=k|Y_u,v\in \neighbor_u)}}}  }_{\textbf{Neighborhood shift}}}
    + \underbrace{\Delta_{CE}}_{\text{Feature shift}}
\end{align*}
}
\normalsize
where 
    $TV(\probs(\rvY_u), \probt(\rvY_u))$ % \coloneq \frac{1}{2}\sum_{j\in\gY}\abs{\probs(\rvY_u=j) - \probt(\rvY_u=j)}$ 
    is the total variation distance between the source and target label distributions. % of source and target
    and 
    $\Delta_{CE}$ is the error gap that exists if and only if feature shift exists.
\end{restatable}

% \pan{first term should be in total variation; it is not good to mix sum and expectation; why there is a node association between source and target domain; the degree is a distribution, does max over d make sense? what is the neighbor shift here?; where is i,j in the feature shift?; Does the above bound assume mean pooling or anything?}

Our bound 
% aligns with the findings of \citet{liu2024pairwise}, which 
highlights the impact of conditional structure shift, label shift, and feature shift on generalization to the target graph. 

\section{TEST-TIME STRUCTURAL ALIGNMENT}
\label{sec:ttsa}

\vspace{-3mm}

Motivated by our theoretical analysis, we propose Test-Time Structural Alignment (\proj) to address GTTA in a principled way. % mitigate variations in neighborhood information for GTTA principally.
To address structure shift, \proj first conducts neighborhood alignment via weighted aggregation and then properly balances self and neighboring node information combination based on the SNRs of these representations. Lastly, as a generic approach for structure shift, \proj can be combined with non-graph TTA methods to refine the decision boundary. 
% to address the remaining feature shift. 

% Motivated by our theoretical analysis, we propose \proj to add

% \hans{neighborhood shift: difference compared to pairalign
% 1.compute inverse (1st order)
% 2.trainable $\rightarrow$ TTA is not trainable $\rightarrow$ adjust classifier, SNR (2nd order)
% 3.stable usage $\rightarrow$ why gamma not accurate $\rightarrow$ entropy
% }

\subsection{Neighborhood Alignment} 
\label{subsec:1stalign}

% Neighborhood shift alters the node label ratios in the aggregated neighborhood information, causing the shift in the GNN-encoded representations. 
% In training-time GDA, PairAlign \cite{liu2024pairwise} leverages such a technique by assigning edge weights to align the distribution of the source neighborhood with the target domain.

% %\pan{the original version of the next paragraph}

% Inspired by this idea, for GTTA, we aim to achieve a similar goal but in a different direction. Based on Theorem \ref{theory:decomposeerror}, we align the target neighborhood distribution with the source domain, leveraging the fact that the pre-trained model is optimized for the source distribution. Specifically, the neighborhood distribution determines the ratio of messages passed from a neighboring class $j$ to center class $i$. To adjust for distributional differences, this ratio can be rescaled by assigning weights to edges from class $j$ to class $i$, effectively acting as an up/down-sampling mechanism during message aggregation. To ensure that message aggregation in the target domain aligns with the expected behavior in the source domain, \proj incorporates the following adjustment:

% Neighborhood shift 
The conditional structure shift (CSS) alters the node label ratios in the aggregated neighborhood information, causing a shift in the GNN-encoded representations. Based on Theorem \ref{theory:decomposeerror}, we align the target neighborhood distribution with the source domain, leveraging the fact that the pre-trained model is optimized for the source distribution. Specifically, the neighborhood distribution determines the ratio of messages passed from a neighboring class $j$ to the center class $i$. To adjust for distributional differences, this ratio can be rescaled by assigning weights to edges from class $j$ to class $i$, effectively acting as an up/down-sampling mechanism during message aggregation. To ensure that message aggregation in the target domain aligns with the expected behavior in the source domain, \proj incorporates the following adjustment:

% This ratio can be rescaled by assigning weights on $j$ to $i$ edges, which can be interpreted as up/down-sampling of messages during message aggregation.  

% Hence, \proj incorporates defined as follows to make the target domain message aggregation match the source domain in expectation:

\begin{definition}
Let $\probt(Y_v=j|Y_u=i,v\in \neighbor_u)>0, \forall i,j \in \gY$, we have $\mgamma \in \R^{|\gY|\times|\gY|}$ as:
% \label{def:gamma}
\begin{equation} 
    [\mgamma]_{i,j} = \frac{\probs(Y_v=j|Y_u=i,v\in \neighbor_u)}{\probt(Y_v=j|Y_u=i,v\in \neighbor_u)}, \: \forall i,j \in \gY
\label{eq:gamma_ratio}
\end{equation}
\end{definition}
\vspace{-2mm}

% Though it may seem trivial that applying the same alignment strategy with the ratio is simply the inverse of the ratio used in GDA.
% To estimate $\mgamma$, we assume that the source summary statistics $\probs(Y_v=j|Y_u=i,v\in \neighbor_u)$ are recorded and available at test time; otherwise, alignment would not be feasible. Storing  $\probs(Y_v=j|Y_u=i,v\in \neighbor_u)$ incurs minimal cost, as it is merely an $|\gY|\times|\gY|$ matrix. Beyond this, no additional information from the source domain is required. For $\probt(Y_v=j|Y_u=i,v\in \neighbor_u)$, we estimate it based on target pseudo labels. Note that PairAlign~\cite{liu2024pairwise} enhances estimation accuracy that relies pseudo labels~\cite{liu2023structural} by leveraging a least-squares constrained optimization. However, in GTTA, the absence of source graphs and the potential need for real-time adaptation render this approach impractical.

% To estimate $\mgamma$, we assume that the source summary statistics $\probs(Y_v=j|Y_u=i,v\in \neighbor_u)$ are recorded and available at test time.
% % ; otherwise, alignment would not be feasible. 
% Storing  $\probs(Y_v=j|Y_u=i,v\in \neighbor_u)$ incurs minimal cost, as it is merely an $|\gY|\times|\gY|$ matrix. Beyond this, no additional information from the source domain is required. For $\probt(Y_v=j|Y_u=i,v\in \neighbor_u)$, we estimate it based on target pseudo labels. 

To estimate $\mgamma$, we estimate $\probt(Y_v=j|Y_u=i,v\in \neighbor_u)$ based on target pseudo labels. $\probs(Y_v=j|Y_u=i,v\in \neighbor_u)$ can be computed directly from the source graph during model training and added as extra model parameters for the test time direct usage. During test time, we weight the message with a properly assigned $[\mgamma]_{i,j}$ computed based on the ratio of the two quantities. This idea aligns with the general principle of test-time alignment~\cite{mirza2023actmad, niu2024test, kojima2022robustvit}, but is novel in the aspect of dealing with the structure shift for graph data. Assigning edge weights to reduce structure shifts was also studied in training-time GDA by PairAlign~\cite{liu2024pairwise}. 
However, PairAlign is inapplicable in the GTTA setting as it assumes concurrent access to both the labeled source graph and the unlabeled target graph to train a model from scratch. This allows the model to learn and adjust to the weighted messages simultaneously, leading to better coordination between the model and the weights.
In contrast, TSA addresses the more restrictive and practical inference-time GTTA problem. The model is pre-trained and frozen.
To compensate for this inability to co-train, TSA needs a more reliable assignment strategy specified in the following.

% TSA differs from PairAlign in that $\mgamma$ estimation and model training can be done together in PairAlign, which better coordinates the two sides as it does training-time GDA (the model better fits the weighted message passing, and the weights are adjusted based on the improved models). 
% However, this cannot be done at test time. To compensate for such a loss, TSA needs a more reliable assignment strategy specified in the following.

%we assume the source statistics $\probs(Y_v=j|Y_u=i,v\in \neighbor_u)$ are stored, but we do not require access to the source graph $\gG^{\gS}$ during adaptation.
% ; otherwise, alignment would not be feasible. 
%Storing  $\probs(Y_v=j|Y_u=i,v\in \neighbor_u)$ incurs minimal cost, as it is merely an $|\gY|\times|\gY|$ matrix. Therefore, it does not violate the computational or privacy constraints imposed by GTTA.
%It is worth noting that the use of source statistics at test time is common when label shift occurs. Some recent works \cite{zhou2023ods, li2021imbalanced} rely on balanced source training, which can be interpreted as implicitly recording source statistics.
% Beyond this, no additional information from the source domain is required.

\paragraph{Reliable Assignment of $\mgamma$.}  
The ratio $\mgamma$ should be assigned to edge weights based on the label pairs of the central and neighboring nodes.
In training-time GDA, this assignment is straightforward, as it relies on the ground-truth labels of the source graph. However, in GTTA, the node labeling information is unavailable. A naive approach is to use target pseudo labels, i.e., $[\mgamma]_{i,j}$ for the message from a node with pseudo label $j$ to a node with pseudo label $i$, but this often results in significant mismatches. \proj addresses this by quantifying the uncertainty of target pseudo labels~\cite{zhang2019bayesian,stadler2021graph,hsu2022makes, hsu2022graph}. In particular, TSA assigns $[\mgamma]_{i,j}$ only to node pairs $v\rightarrow u$ where both of their soft target pseudo labels $\hat{y}$ have low entropy $H(\hat{y})=-\sum_{i\in \gY}[\hat{y}]_i\ln([\hat{y}]_i)\leq\rho_1 \cdot \ln ( |\gY|)$. Here, $\rho_1$ is a hyperparameter and $ \ln( |\gY|)$ is the maximum entropy for a $|\gY|$ class prediction. In Fig.~\ref{fig:pa_vis} (a), nodes with low-entropy soft predictions are more reliable, resulting in higher accuracy after the assignment of $\mgamma$.

\begin{figure*}[t]
\centering
\begin{tabular}{cc|c}
(a) SNR Source (CN) &(b) SNR Target (US) 
% &(c) SNR CN$\rightarrow$US
% &\multicolumn{2}{|c|}{ (b) Decision Boundary Refinement}
& (c) SNR $\alpha$ CN$\rightarrow$US  \\
\includegraphics[width=0.23\textwidth]{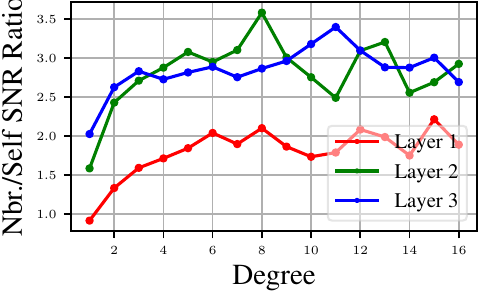}
&\includegraphics[width=0.23\textwidth]{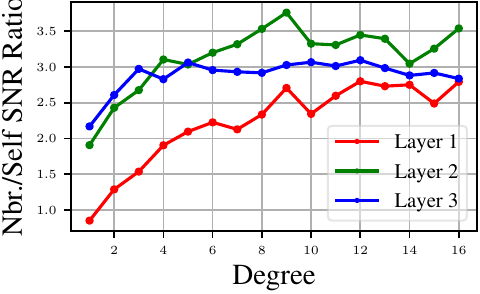}
% &\includegraphics[width=0.23\textwidth]{figures/SNR_analysis_CNStoUS_diff.pdf} 
&\includegraphics[width=0.23\textwidth]{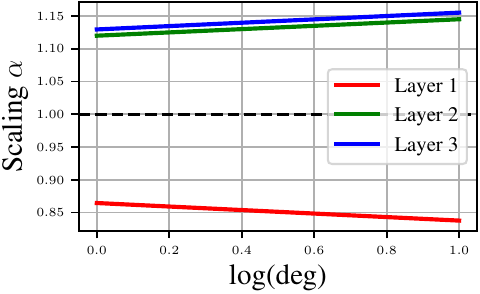} \\

% \includegraphics[width=0.23\textwidth]{figures/SNR_analysis_CN.png}
% &\fbox{\rule{0pt}{0.12\textwidth} \rule{0.2\textwidth}{0pt}}
% &\fbox{\rule{0pt}{0.12\textwidth} \rule{0.2\textwidth}{0pt}}
% &\includegraphics[width=0.23\textwidth]{figures/cnus_alpha.pdf} \\
\end{tabular}
\vspace{-4mm}
\caption{(a) and (b) compare the SNR of neighborhood-aggregated representations to that of self-node representations across node degrees and GNN layers, based on the ratio of the two quantities. A higher ratio indicates that the neighborhood-aggregated representations exhibit higher SNR.
(c)~Analysis of the SNR adjustment $\alpha$ across different layers and node degrees.
Details are discussed in Sec.~\ref{subsec:result_analysis}.
% The adjustment of $\alpha$ is consistent with the SNR ratio from target to source.
}
\label{fig:snr}
\vspace{-2mm}
\end{figure*}

% On top of the neighborhood alignment, we introduce additional techniques that help better adaptation under the GTTA setting, taking into account the constraints of starting with a pretrained source model and lacking the opportunity for substantial retraining.

\vspace{-2mm}
\subsection{SNR-Inspired Adjustment}
\vspace{-2mm}
\label{subsec:snr}
Going beyond aligning the first-order statistics, we further optimize the signal-to-noise ratio (SNR) of node representations to enhance performance. The SNR is defined as the ratio of inter-class representation variance to intra-class representation variance \cite{yuan2019signal}, formally expressed as $\text{SNR} \coloneqq {\sigma_{inter}^2} /{\sigma_{intra}^2}$ with

\vspace{-6mm}
{\small
\begin{align*} 
    % \text{SNR} \coloneqq \frac{\sigma_{inter}^2}{\sigma_{intra}^2}, 
    % \quad 
    \sigma_{inter}^2 = \sum_{i\in\gY} \frac{n_i \left\| \mu_i - \mu_* \right\|_2^2} {|\gY|}, 
    \quad 
    \sigma_{intra}^2= \sum_{i\in \gY} \sum_{h \in \mathcal{H}_i}  \frac{\left\| h - \mu_i \right\|_2^2 }{|\gY|} 
\label{eq:snr}
\end{align*}
}
\vspace{-6mm}

where $n_i$ is the number of samples in class $i$, $\mu_i$ is the centroid representation of class $i$, $\mu_*$ is the overall centroid of all nodes, and $\mathcal{H}_i$ represents the set of representations in class $i$. 
A higher inter-class variance $\sigma_{inter}^2$ indicates informative and well-separated representations, while a higher $\sigma_{intra}^2$ indicates a higher noise in the representations.

Optimizing SNR complements the neighborhood alignment approach. Even if neighborhood label distributions are perfectly aligned, variations in neighbor set sizes between the source and target graphs can impact the SNR of aggregated neighboring node representations.
To compare the SNR of neighborhood-aggregated representations to that of self-node representations, we compute $\text{SNR ratio} = \frac{\text{Nbr. Aggr. SNR}}{\text{Self-Node SNR}}$ which exhibits different trends as a function of node degree (aka neighbor set size) and layer depth, as shown in Fig.~\ref{fig:snr} (a) and (b). 
In both figures, deeper layers and higher degree values generally lead to a higher SNR ratio, indicating aggregated representations are better denoised. However, the rate and shape of SNR ratio growth vary across source and target domains. Consequently, the combination of these aggregated representations with self-node representations in a typical GNN pipeline (Eq.~\ref{eq:gnn}) should be properly adjusted as well in test-time adaptation. To implement the SNR-inspired adjustment, we introduce a parameter to perform the weighted combination of self-node representations and neighborhood-aggregated representations at each layer, adapting to node degrees as follows:

% \hans{is this appropriate to use definition as it is not something derived from principle?}
\begin{definition}
Let $\tilde{d_u}=\left(\frac{\ln(d_u +1)}{\ln(\max_{v\in\node} d_v+1)}\right)$ denote log-normalized degree of node $u$ and let $\text{MLP}^{(k)}$ and $b^{(k)}$ to be learnable parameters for adjusting $k$-th layer combination, we define the weights for combination $\alpha \in \R^{L}$ as:

\vspace{-3mm}
\begin{equation}
    [\alpha]_{k} =
\sigma(\text{MLP}^{(k)}(\tilde{d_u}))-0.5+ b^{(k)}, \: \forall k \in [L]
\label{eq:alpha_ratio}
\end{equation}
\end{definition}
\vspace{-3mm}

where $\sigma$ is the sigmoid function. 
%The term $b^{(k)}$ accounts for the global tendency, while $\text{MLP}^{(k)}$ gives a local interpretation with respect to node degrees, both at the $k$-th layer. 
During initialization $[\alpha]_{k}$ is set to $1$.
Degree values are taken in the logarithmic domain to handle their often long-tailed distribution. 

Together, $[\alpha]_{k}\cdot[\mgamma]_{i,j}$ is used to reweight the GNN messages, adjusting the influence from the node with a highly certain pseudo label $j$ to the node with a highly certain pseudo label $i$.
%Since $[\alpha]_{k}$ is constant with respect to the GNN aggregation function, it can be considered as a bootstrapping ratio for non-self-loop pairs.

% Moreover, the SNR of node self representations may also vary across different GNN layers, since in general deep layers help with reducing the variance. 

 \textbf{Remark.} Neighborhood alignment alone does not address potential shifts in SNR. This is because the alignment approach, inspired by Thm.~\ref{theory:decomposeerror}, focuses solely on population-level performance, whereas SNR also incorporates the noise (variance) induced by empirical sampling from the population. Thus, these two aspects complement each other.

\begin{algorithm}[t]
   \caption{Test-Time Structural Alignment (TSA)}
   \label{alg:example}
\begin{algorithmic}[1]
   \STATE {\bfseries Input:} A GNN $\phi$ and a classifier $g$ pretrained on source graph $\gG^{\gS}$; Test-time target graph $\gG^{\gT}$; Source statistics $\probs(Y_v|Y_u,v\in \neighbor_u)\in \R_+^{|\gY|\times|\gY|}$
   \STATE Initialize $b^{(k)}\leftarrow 1$ and $\text{MLP}^{(k)}\text{ parameters} \leftarrow0$ for the $k$-th layer.
    \STATE Perform boundary refinement based on embeddings from $g\circ \phi(\gG^{\gT})$ and get soft pseudo-labels $\hat{y}$
    \STATE {\bfseries Get} $\gG^{\gT}_{new}$ {\bfseries by assigning edge weights:}
    \INDSTATE Compute $\gamma$ using  $\hat{y}$ via Eq. \ref{eq:gamma_ratio}
    \INDSTATE  Assign $\gamma$ only if node pairs $H(\hat{y})\leq \rho_1 \cdot \ln(|\mathcal{Y}|)$
        \INDSTATE Assign the parameterized $\alpha$ via Eq. \ref{eq:alpha_ratio}
    \STATE Update $\alpha$'s parameters $b^{(k)}$ and $\text{MLP}^{(k)}$ via Eq. \ref{eq:loss}
    %\STATE Perform boundary refinement again %using $\hat{y}'$ where $H(\hat{y}')\leq \rho_2 \cdot \ln(|\mathcal{Y}|)$ as supervision 
   % \ENDFOR
   \STATE {\bfseries return} $\hat{y}_{\mathrm{final}}$ after another boundary refinement
   % $\hat{Y}_{final}\leftarrow\mathit{ClsTTA}(g(\phi(\gG^{\gT}_{new})))$
   % \REPEAT
   % \STATE Initialize $noChange = true$.
   % \FOR{$i=1$ {\bfseries to} $m-1$}
   % \IF{$x_i > x_{i+1}$}
   % \STATE Swap $x_i$ and $x_{i+1}$
   % \STATE $noChange = false$
   % \ENDIF
   % \ENDFOR
   % \UNTIL{$noChange$ is $true$}

\end{algorithmic}
% \vspace{-1mm}
\end{algorithm}
\vspace{-1mm}

\begin{table*}[t]
\vspace{-2mm}
\scriptsize

\caption{Synthetic CSBM results (accuracy). \textbf{Bold} indicates improvements in comparison to the corresponding non-graph TTA baselines. \underline{Underline} indicates the best model. First six: imbalanced source training. Last two: balanced source training. }
\vspace{-4mm}
\label{table:syn}
\begin{center}
\resizebox{0.8\textwidth}{!}{%
\small
\begin{tabular}{lcccccccc}
\toprule
& \multicolumn{2}{c}{Nbr. Shift} &\multicolumn{2}{c}{Nbr.+ SNR Shift} &\multicolumn{2}{c}{Struct. Shift (Imbal.$\rightarrow$ Bal.)} &\multicolumn{2}{c}{Struct. Shift (Bal.$\rightarrow$ Imbal.)}\\
\midrule
ERM &82.70$\pm$4.45 &61.11$\pm$10.81 &77.03$\pm$3.99 &61.93$\pm$6.44 &50.41$\pm$4.88 &39.12$\pm$4.71 &68.27$\pm$5.00 &61.39$\pm$1.89 
\\
GTrans &86.67$\pm$3.59 &72.37$\pm$4.06 &79.55$\pm$1.23 &68.69$\pm$3.27 &59.24$\pm$2.12 &47.42$\pm$3.73 &79.29$\pm$2.71 &66.77$\pm$2.52  

\\
HomoTTT
&85.08$\pm$2.87 &65.79$\pm$9.07 &78.84$\pm$3.42 &65.31$\pm$5.69 &52.34$\pm$4.69 &40.68$\pm$5.24 
 &70.11$\pm$5.08 &62.19$\pm$2.51 

\\
SOGA &86.09$\pm$3.89 &70.39$\pm$7.96 &79.75$\pm$3.20 &69.00$\pm$5.28 &56.60$\pm$3.88 &44.09$\pm$5.32 &73.52$\pm$5.30 &63.76$\pm$3.45  

\\
ActMAD
&84.95$\pm$4.04 &65.17$\pm$10.19 &78.66$\pm$3.70 &64.95$\pm$5.99 &54.78$\pm$3.94 &42.22$\pm$4.83 &72.41$\pm$5.67 &63.39$\pm$3.33
\\

TENT
&87.48$\pm$2.86 &77.14$\pm$4.64 &81.04$\pm$2.72 &72.51$\pm$3.39 &76.48$\pm$6.21 &59.12$\pm$5.06 &77.21$\pm$5.53 &62.36$\pm$6.83  

\\
LAME &83.96$\pm$5.35 &61.44$\pm$11.33 &77.56$\pm$4.92 &62.58$\pm$7.24 &50.33$\pm$4.78 &39.05$\pm$4.67  &68.33$\pm$5.02 &61.26$\pm$1.97 

\\
T3A &77.05$\pm$7.10 &59.83$\pm$10.50 &71.44$\pm$6.11 &56.56$\pm$7.50 &48.13$\pm$5.64 &38.19$\pm$3.72 &68.50$\pm$4.81 &61.63$\pm$1.81

\\

\midrule
TSA-TENT  & \textbf{88.78$\pm$1.37} & \textbf{80.51$\pm$2.39} & \textbf{83.19$\pm$1.46} & \textbf{76.41$\pm$1.25} & \underline{\textbf{88.68$\pm$4.99}} & \underline{\textbf{66.25$\pm$7.75}} & \textbf{81.20$\pm$8.18} & \underline{\textbf{70.15$\pm$2.30}} 
\\
TSA-LAME & \textbf{88.96$\pm$1.66} & \textbf{80.02$\pm$5.44} & \textbf{83.51$\pm$0.55} & \underline{\textbf{79.56$\pm$1.82}} & \textbf{65.09$\pm$2.34} & \textbf{52.90$\pm$6.11} & \textbf{82.20$\pm$5.17} & \textbf{63.15$\pm$2.58}
\\
TSA-T3A  & \underline{\textbf{89.96$\pm$1.33}} & \underline{\textbf{81.08$\pm$2.73}} & \underline{\textbf{84.23$\pm$1.24}} & \textbf{76.89$\pm$2.02} & \textbf{65.59$\pm$2.57} & \textbf{52.34$\pm$7.19} & \underline{\textbf{82.55$\pm$5.06}} & \textbf{64.88$\pm$3.30} 
\\
\bottomrule
\end{tabular}
}
\end{center}
\vspace{-2mm}
\end{table*}

\begin{table*}[t]
\vspace{-2mm}
\scriptsize
\caption{MAG results (accuracy). \textbf{Bold} indicates improvements in comparison to the corresponding non-graph TTA baselines. \underline{Underline} indicates the best model.}
\vspace{-4mm}
\label{table:mag}
\begin{center}
\begin{adjustbox}{width = 0.9\textwidth}
\begin{tabular}{lcccccccccc}
\toprule
\textbf{Method} & US$\rightarrow$CN & US$\rightarrow$DE & US$\rightarrow$JP & US$\rightarrow$RU & US$\rightarrow$FR & CN$\rightarrow$US &CN$\rightarrow$DE & CN$\rightarrow$JP & CN$\rightarrow$RU & CN$\rightarrow$FR\\
\midrule
ERM &31.86$\pm$0.83 &32.22$\pm$1.16 &41.77$\pm$1.27 &29.22$\pm$1.64 &24.80$\pm$0.88 &37.41$\pm$1.01 &21.54$\pm$0.79 &30.12$\pm$0.72 &19.19$\pm$1.12 &16.92$\pm$0.58 
\\
GTrans  &31.77$\pm$0.91 &32.14$\pm$1.05 &41.55$\pm$1.23 &29.74$\pm$1.57 &25.03$\pm$0.85 &36.17$\pm$0.89 &21.07$\pm$0.93 &29.08$\pm$0.82 &19.68$\pm$1.14 &16.78$\pm$0.62  \\

HomoTTT &25.89$\pm$1.98 &29.65$\pm$1.41 &33.49$\pm$1.78 &29.84$\pm$1.68 &27.38$\pm$0.62 &24.39$\pm$2.12 &17.36$\pm$2.82 &21.00$\pm$1.90 &23.20$\pm$1.34 &15.02$\pm$1.42 

\\
SOGA &21.54$\pm$2.52 &25.48$\pm$0.93 &36.24$\pm$3.31 &29.07$\pm$4.14 &24.34$\pm$0.91 &38.95$\pm$3.35 &25.75$\pm$1.14 &38.25$\pm$1.42 &29.86$\pm$1.71 &23.50$\pm$0.67  \\
ActMAD &31.38$\pm$0.74 &31.93$\pm$1.18 &41.53$\pm$1.22 &29.68$\pm$1.73 &25.09$\pm$0.93 &36.04$\pm$0.97 &20.59$\pm$0.78 &29.01$\pm$0.67 &19.37$\pm$1.07 &16.30$\pm$0.53 

\\

TENT &26.72$\pm$1.33 &32.73$\pm$0.63 &40.80$\pm$0.91 &32.26$\pm$0.95 &28.32$\pm$0.66 &27.21$\pm$0.88 &15.66$\pm$0.86 &24.62$\pm$0.48 &21.37$\pm$0.73 &13.84$\pm$0.62 
\\ 
LAME &35.75$\pm$0.85 &33.64$\pm$1.70 &44.97$\pm$1.15 &30.19$\pm$1.64 &24.17$\pm$1.84 &40.08$\pm$1.13 &22.64$\pm$1.14 &33.00$\pm$1.48 &17.80$\pm$0.55 &17.43$\pm$0.93 
 \\
T3A  &41.47$\pm$1.15 &45.36$\pm$2.15 &50.34$\pm$0.94 &46.41$\pm$0.84 &40.26$\pm$1.69 &46.50$\pm$1.26 &38.62$\pm$1.03 &46.10$\pm$0.38 &43.11$\pm$0.76 &29.95$\pm$1.36  
\\
\midrule
TSA-TENT
& \textbf{27.30$\pm$1.61} & \textbf{32.84$\pm$0.78} & \textbf{40.82$\pm$0.99} & \textbf{32.53$\pm$0.91} & \textbf{28.62$\pm$0.63} & \textbf{27.89$\pm$0.98} & \textbf{16.22$\pm$1.17} & \textbf{24.87$\pm$0.56} & \textbf{22.05$\pm$0.84} & \textbf{14.20$\pm$0.71}
\\
TSA-LAME
& \textbf{37.95$\pm$0.97} & \textbf{36.29$\pm$1.65} & \textbf{46.86$\pm$1.13} & \textbf{32.86$\pm$2.23} & \textbf{27.22$\pm$1.48} & \textbf{44.83$\pm$0.88} & \textbf{28.51$\pm$0.44} & \textbf{39.80$\pm$0.99} & \textbf{24.54$\pm$0.87} & \textbf{22.39$\pm$0.30}
\\
TSA-T3A
& \underline{\textbf{41.65$\pm$0.99}} & \underline{\textbf{47.01$\pm$2.08}} & \underline{\textbf{51.65$\pm$0.90}} & \underline{\textbf{46.61$\pm$0.88}} & \underline{\textbf{43.45$\pm$0.81}} & \underline{\textbf{48.09$\pm$0.60}} & \underline{\textbf{39.18$\pm$1.87}} & \underline{\textbf{46.50$\pm$0.25}} & \underline{\textbf{43.70$\pm$1.38}} & \underline{\textbf{30.89$\pm$2.13}}
\\
\bottomrule
\end{tabular}
\end{adjustbox}
\end{center}
\vspace{-5mm}
\end{table*}

\begin{table*}[t]
\vspace{-2mm}
\scriptsize
\caption{Pileup results (f1-scores). \textbf{Bold} indicates improvements in comparison to the corresponding non-graph TTA baselines. \underline{Underline} indicates the best model.}
\vspace{-4mm}
\label{table:pileup}
\begin{center}
\resizebox{0.8\textwidth}{!}{%
\small
\begin{tabular}{lcccccccc}
\toprule
\textbf{Method} & PU10$\rightarrow$30 & PU30$\rightarrow$10 & PU10$\rightarrow$50 & PU50$\rightarrow$10 & PU30$\rightarrow$140 & PU140$\rightarrow$30 &gg$\rightarrow$qq & qq$\rightarrow$gg \\
\midrule
ERM &57.98$\pm$0.66 &65.40$\pm$2.17 &47.66$\pm$1.47 &67.81$\pm$1.70 &19.42$\pm$2.59 &57.49$\pm$3.02 &69.35$\pm$0.81 &67.90$\pm$0.46 
\\
GTrans  &57.37$\pm$1.49 &63.66$\pm$2.43 &48.13$\pm$2.11 &65.74$\pm$1.95 &28.41$\pm$4.01 &57.65$\pm$1.97 &69.17$\pm$0.82 &67.37$\pm$0.56  
\\
HomoTTT
&58.29$\pm$1.34 &62.50$\pm$3.35 &49.07$\pm$1.32 &65.18$\pm$1.76 &28.63$\pm$2.34 &56.11$\pm$3.83 &70.31$\pm$0.58 &\underline{68.85$\pm$0.29}
\\
SOGA &60.13$\pm$0.75 &69.71$\pm$0.92 &51.68$\pm$0.83 &67.47$\pm$1.42 &37.16$\pm$1.15 &56.59$\pm$3.86 &\underline{70.83$\pm$0.66}
&68.84$\pm$0.97  
\\
ActMAD &58.54$\pm$0.81 &64.59$\pm$2.38 &49.39$\pm$1.35 &67.12$\pm$1.74 &20.88$\pm$2.66 &50.97$\pm$5.14 &69.20$\pm$0.79 &68.04$\pm$0.46
\\
TENT &58.25$\pm$1.85 &58.73$\pm$4.80 &48.27$\pm$2.58 &58.41$\pm$7.89 &30.59$\pm$3.07 &0.04$\pm$0.08 &68.72$\pm$0.48 &68.00$\pm$0.60 
\\
LAME &57.77$\pm$0.66 &66.29$\pm$2.42 &47.31$\pm$1.58 &68.24$\pm$1.61 &11.96$\pm$3.08 &58.81$\pm$2.00 &69.25$\pm$0.52 &67.94$\pm$0.66 
\\
T3A &58.74$\pm$1.09 &70.02$\pm$2.33 &49.61$\pm$1.03 &70.85$\pm$0.64 &30.33$\pm$3.66 &54.79$\pm$1.52 &69.45$\pm$0.97 &68.45$\pm$0.61  
\\
\midrule
TSA-TENT & \textbf{58.27$\pm$1.84} & \textbf{60.11$\pm$3.67} & \textbf{48.44$\pm$2.67} & \textbf{58.79$\pm$7.01} & \textbf{36.22$\pm$2.04} & \textbf{6.13$\pm$5.93} & \textbf{69.26$\pm$0.54} & \textbf{68.80$\pm$0.78}
\\
TSA-LAME & \textbf{60.76$\pm$0.71} & \textbf{66.94$\pm$1.62} & \textbf{50.72$\pm$1.49} & \textbf{68.53$\pm$1.55} & \textbf{36.43$\pm$2.57} & \underline{\textbf{59.36$\pm$2.53}} & \textbf{69.85$\pm$0.66} & \textbf{68.31$\pm$0.54}
\\
TSA-T3A
& \underline{\textbf{61.03$\pm$1.19}} & \underline{\textbf{70.42$\pm$1.66}} & \underline{\textbf{52.39$\pm$0.86}} & \underline{\textbf{70.92$\pm$0.64}} & \underline{\textbf{37.33$\pm$2.59}} & \textbf{57.73$\pm$1.98} & \textbf{69.73$\pm$1.09} & \textbf{68.75$\pm$0.66}
\\
\bottomrule
\end{tabular}
}
\end{center}
\vspace{-2mm}
\end{table*}

\begin{table*}[t]
\vspace{-2mm}
\scriptsize
\caption{Arxiv and DBLP/ACM results (accuracy). \textbf{Bold} indicates improvements in comparison to the corresponding non-graph TTA baselines. \underline{Underline} indicates the best model.}
\vspace{-4mm}
\label{table:arxiv}
\begin{center}
\resizebox{0.8\textwidth}{!}{%
\small
\begin{tabular}{lcccccccc}
\toprule
& \multicolumn{2}{c}{1950-2007} &\multicolumn{2}{c}{1950-2009} &\multicolumn{2}{c}{1950-2011} &\multicolumn{2}{c}{DBLP \& ACM}\\
\textbf{Method} &2014-2016 &2016-2018 &2014-2016 &2016-2018 &2014-2016 &2016-2018 & D$\rightarrow$A & A$\rightarrow$D\\
\midrule
ERM &41.04$\pm$0.50 &40.48$\pm$1.45 &44.80$\pm$1.96 &42.38$\pm$3.64 &53.40$\pm$1.14 &51.68$\pm$1.76 
&28.95$\pm$4.50 &52.45$\pm$6.81 

\\
GTrans &40.92$\pm$0.32 &40.25$\pm$1.69 &45.31$\pm$1.99 &43.83$\pm$3.15 &53.68$\pm$0.95 &52.57$\pm$1.23 
&34.47$\pm$2.99
&49.68$\pm$5.44 
\\
HomoTTT
&36.33$\pm$0.95 &32.83$\pm$1.15 &41.22$\pm$1.15 &38.86$\pm$1.42 &48.18$\pm$1.38 &44.61$\pm$2.90 
&31.25$\pm$3.96 &52.90$\pm$6.25 
\\
SOGA &34.11$\pm$2.91 &28.94$\pm$4.68 &41.59$\pm$2.03 &39.61$\pm$3.22 &50.12$\pm$1.38 &43.03$\pm$3.70 
&35.74$\pm$5.20 &\underline{58.59$\pm$8.35}
\\
ActMAD &41.02$\pm$0.53 &40.43$\pm$1.44 &44.88$\pm$1.94 &42.39$\pm$3.45 &53.43$\pm$1.13 &51.78$\pm$1.78 
&29.47$\pm$4.46 &52.75$\pm$6.89 
\\

TENT&40.77$\pm$0.32 &39.58$\pm$1.30 &45.74$\pm$1.26 &43.98$\pm$1.90 &54.45$\pm$1.03 &53.19$\pm$1.31 
&36.59$\pm$4.26 &53.38$\pm$6.60 
\\
LAME  &40.91$\pm$0.88 &41.02$\pm$1.81 &45.13$\pm$2.49 &43.23$\pm$4.42 &53.63$\pm$1.21 &52.07$\pm$1.62 
&28.02$\pm$5.08 &52.74$\pm$7.08 
\\
T3A &39.57$\pm$1.06 &38.98$\pm$1.78 &43.34$\pm$1.51 &41.17$\pm$2.88 &50.10$\pm$1.70 &48.00$\pm$2.82 
&35.29$\pm$5.58 &52.50$\pm$6.83 
\\
\midrule
TSA-TENT & \textbf{40.92$\pm$0.34} & \textbf{39.78$\pm$1.21} & \underline{\textbf{46.68$\pm$1.24}} & \underline{\textbf{45.08$\pm$1.71}} & \underline{\textbf{54.78$\pm$0.80}} & \underline{\textbf{53.61$\pm$1.24}} & \underline{\textbf{37.06$\pm$3.93}} & \textbf{54.06$\pm$6.63}
\\
TSA-LAME & \underline{\textbf{41.34$\pm$0.83}} & \underline{\textbf{41.23$\pm$1.70}} & \textbf{45.42$\pm$2.45} & \textbf{43.71$\pm$4.99} & \textbf{54.05$\pm$1.05} & \textbf{52.76$\pm$1.47} & \textbf{28.66$\pm$5.27} & \textbf{52.90$\pm$7.30}
\\
TSA-T3A & \textbf{40.03$\pm$1.03} & \textbf{39.70$\pm$1.49} & \textbf{44.09$\pm$1.53} & \textbf{42.17$\pm$2.91} & \textbf{50.96$\pm$1.68} & \textbf{49.21$\pm$2.67} & \textbf{35.30$\pm$5.57} & \textbf{52.55$\pm$6.92}
\\
\bottomrule
\end{tabular}
}
\end{center}
\vspace{-5mm}
\end{table*}

\subsection{Decision Boundary Refinement}
\label{subsec:boundary}

% Although Fig. \ref{fig:pa_vis} top (d) shows an aligned latent node representation space, label shift results in a mismatch of the decision boundary between the source and target domains, 
% In GTTA, label shift can result in a mismatch of the decision boundary, even after addressing structure shift and obtaining high-SNR node representations. This is illustrated in  Fig.~\ref{fig:pa_vis} (d). 
In GTTA, even when CSS is fully addressed and high-SNR node representations are obtained, the decision boundary may still suffer from a mismatch due to label and feature shifts (Fig.~\ref{fig:pa_vis} (d)).

% The primary issue arises due to data imbalance during source training. 
%As shown in Fig.~\ref{fig:pa_vis}(b), when the source model is trained on label-imbalanced data with class proportions $[0.1, 0.3, 0.6]$, the decision boundary for the minority class (class $0$) tends to be located in high-density regions. 
%Consequently, even when the neighborhood distribution is perfectly aligned with Oracle assignment, the predefined boundary from the source model can still be affected by label shift. % where the target class ratio is $[0.3, 0.3, 0.3]$ (Fig.~\ref{fig:pa_vis}(c)).

% \label{subsec:boundary_refine}
% Section \ref{Sec:Test Error Analysis} theoretically and empirically shows that simply addressing neighborhood shift in GNNs leads to suboptimal target risk $\errort(g\circ\phi)$. The predefined decision boundary in the source model can easily suffer from the shift in label distribution during test time. Recall the example that label imbalance in the source domain causes the boundary to lie in high density regions for small portion classes, which cannot be solved by simply aligning the neighborhood distribution.
% Under label shift, accuracy may be further hindered as the portion of small-portion classes increases, leading to a decrease in overall accuracy (see Fig. \ref{fig:pa_vis} (c)).
A straightforward approach to refining the decision boundary at test time is to adjust the classifier's batch normalization parameters (TENT~\cite{wang2020tent}) or directly modify its output (T3A~\cite{iwasawa2021test} and LAME~ \cite{boudiaf2022parameter}). 
We integrate these techniques into our framework for two folds: (1) their refined pseudo-labels provide a more reliable assignment of $\mgamma$ and can supervise the update of SNR adjustment. (2) Reciprocally, better alignment of neighborhood information can further refine the decision boundary.

Matcha~\cite{bao2024Matcha} is a very recent work on GTTA that also integrates non-graph TTA methods with an approach for addressing structure shift. However, it considers only degree shift and homophily shift, where homophily shift is merely a special case of CSS in our context. Consequently, Matcha does not introduce the parameter $\mgamma$ and does not fully address the structure shift. % therefore, lacks the ability to properly align neighborhood aggregated representations when the neighboring label distribution shifts. 

% The final adaptation from the the classifier-adapting TTA further enhances prediction accuracy by reducing CSS discrepancy.

% We integrate these techniques into our framework. \pan{after classifier, y , y}

% leverage non-graph TTA 

% without changing the GNN encoder. 
% Non-graph TTA has widely explored this by adjusting the classifier's batch norm parameters \cite{wang2020tent} or its output directly \cite{iwasawa2021test,boudiaf2022parameter}.
% By integrating with non-graph TTA, TSA effectively handles both label shift and feature shift.

% \shikun{we may want to include more insights in how and why previous works can adjust decision boundary. how they are trained, provide some brief summarization, list two works in 3-4 sentences. }

%%%%%%%%%%%%%%%%%%%%%%%%%%%%%%%%%%%%%%%%%%%%%%%%%%%%%%%%%%%%%%%%%%%%%%%%%%%%%%%%%%%%%%%%%%%%%%%%%%%%%%%%%%%%%%%%%%%%%%%%%%%%%%%%%%%%%%%%%%%%%%%%%%%%%%%%%%%%%%%%%%%%%%%%%%%%%%%%%%%%%%%

\textbf{TSA Overview.}
%\label{Sec:overview}
 %Since $\bar{\gamma}_{u,v}$ relies on the quality of the predictions, TSA utilizes more accurate prediction from a classifier-adapting TTA to get a more reliable estimation on $\probt(Y_v,Y_u | v\in \neighbor_u)$ and pseudo-label assignments.
Note that $\mgamma$ is obtained from the initial pseudo labels.
Only the parameters in $\text{MLP}^{(k)}$ and $b^{(k)}$ for $\alpha$ estimation in Eq.~\ref{eq:alpha_ratio} need to be optimized at test time according to Eq.~\ref{eq:loss}.

% The learnable bootstrapping $\alpha$ and parameters in the MLP are optimized by the cross-entropy loss:

% The parameters in $\text{MLP}^{(k)}$ and $b^{(k)}$ will be optimized in the test time via Eq.~\ref{eq:loss} supervised by hard pseudo labels.

\vspace{-5mm}
\begin{equation}
    \gL_{CE} = \frac{1}{|\gV^{\gT}|}\sum_{u\in\gV^{\gT}}\text{cross-entropy}(y_u',\hat{y}_u) 
    \label{eq:loss}
\end{equation}
\vspace{-2mm}

where $y_u'$ is the hard pseudo label refined by the procedure in Sec.~\ref{subsec:boundary} and $\hat{y}_u$ is the soft prediction from the original model.
After updating $\alpha$, TSA makes predictions on the newly weighted graphs and then further adapts the boundaries as described in Sec. \ref{subsec:boundary}. %This procedure will iterate. 
Our algorithm is summarized in Alg. \ref{alg:example}. % to adapt predictions on the newly weighted graph.

\vspace{-3mm}

\section{EXPERIMENTS} 
\vspace{-3mm}

\label{sec:Experiments}
We evaluate \proj 
%by comparing with non-graph TTA methods and existing GTTA methdos the combination with the three non-graph TTA methods and compared to basedlines from TTA
on synthetic datasets and 5 real-world datasets. Additional results and discussions (e.g., hyperparameter sensitivity, ablation studies, computation efficiency) are provided in Appendix.
% ~\ref{app:additional results}.
\vspace{-2mm}

\subsection{Datasets and Baselines}
\vspace{-2mm}

\textbf{Synthetic Data.} 
We use the CSBM model~\cite{deshpande2018contextual} to generate three-class datasets, focusing on structure shifts while keeping feature distributions unchanged. Specifically, we evaluate performance under three conditions: (1) CSS, (2) CSS plus SNR (induced by degree changes) shift, and (3) full structure shift (CSS + label shift) plus SNR shift. This experimental setup is motivated by Thm.~\ref{theory:decomposeerror} and the observations in Sec.\ref{subsec:snr}. For each shift scenario, we examine two levels of severity, with the left column in Table~\ref{table:syn} corresponding to the smaller shift.

\textbf{MAG~\cite{liu2024pairwise}} is a citation network extracted by OGB-MAG \cite{hu2020open}.
Distribution shifts arise from partitioning papers into distinct graphs based on their countries of publication. The task is to classify the publication venue of each paper. Our model is pretrained on graphs from the US and China and subsequently adapted to graphs from other countries.

\textbf{Pileup Mitigation~\cite{liu2023structural}} is a dataset curated for the data-denoising step in high-energy physics~\cite{bertolini2014pileup}.
Particles are generated by proton-proton collisions in LHC experiments. The task is to classify leading-collision (LC) neutral particles from other-collision (OC) particles. Particles are connected via kNN graphs if they are close in the $\eta-\phi$ space shown in Fig.~\ref{fig:pileup_example}.
Distribution shifts arise from pile-up (PU) conditions (primarily structure shift), where PU level indicates the number of OC in the beam, and from the particle generation processes gg$\rightarrow$qq and qq$\rightarrow$gg (primarily feature shift).

\textbf{Arxiv~\cite{hu2020open}} is a citation network between all Computer Science (CS) arXiv papers. 
Distribution shifts originate from different time periods.
Our model is pretrained on the earlier time span 1950 to 2007/ 2009/ 2011 and tested on later periods 2014-2016 and 2016-2018.

\textbf{DBLP and ACM~\cite{tang2008arnetminer, wu2020unsupervised}} are two citation networks. 
The model is trained on one network and adapted to the other to predict the research topic of a paper (node).

\textbf{Baselines.} We compare TSA with six baselines. For non-graph TTA methods, we include TENT \cite{wang2020tent}, LAME \cite{boudiaf2022parameter}, T3A \cite{iwasawa2021test}, and ActMAD \cite{mirza2023actmad}.
% Note that TENT is applied only to the classifier, as GNNs typically do not include batch normalization layers. 
GTrans \cite{jin2022empowering}, HomoTTT \cite{zhang2024fully}, SOGA \cite{mao2024source}, and Matcha \cite{bao2024Matcha} are direct comparisons in GTTA.
Matcha is limited to GPRGNN due to its design. We present the results for GraphSAGE \cite{hamilton2017inductive} in the main paper, while the results for GPRGNN \cite{chien2020adaptive} (including Matcha), GCN \cite{kipf2016semi} are provided in the Appendix.
% \proj results are presented for one epoch of optimization for Eq. \ref{eq:loss}.

\vspace{-3mm}
\subsection{Result Analysis}
\label{subsec:result_analysis}
\vspace{-2mm}

% Compare with baselines (GTTA)
%     - They use self-supervised learning
%     - They dont explore CSS and rely more on homophily assumptions
% Different variant of \proj
%  - \proj consistent improves

% \pan{compare ours with baselines, explain why baselines work bad in particular those gtta baselines. Compare between different variants of our methods and explain why one is better than the other. Focus on how we may handle shifts in graph data.}

\textbf{Comparison with Baselines.} \proj achieves strong performance across all datasets and GNN backbones. In all cases, the best results are obtained by GTTA methods (\underline{underlined} across Tables~\ref{table:syn}–\ref{table:arxiv}), highlighting the importance of addressing structure shift in graph data.
In the synthetic dataset (Table~\ref{table:syn}), \proj outperforms GTTA baselines under plain CSS, demonstrating neighborhood alignment as an effective approach to handling CSS. The same table also showcases that dataset imbalance severely affects source training, which aligns with the worst-case error described in Theorem~\ref{theory:decomposeerror}.
In real-world datasets, \proj achieves an average improvement of 10\% over existing GTTA baselines. 
GTrans relies on self-supervised learning to augment graphs. Although it shows improvements on CSBM, its performance is limited on real-world datasets (Tables~\ref{table:mag}–\ref{table:arxiv}). SOGA and HomoTTT depend on the homophily assumption, which does not always hold in real-world graphs~\cite{luan2024heterophilic}, leading to unstable performance. 
In contrast, \proj makes no assumptions and consistently delivers improvements by structure shifts.

\textbf{Different \proj-version.}
Overall, \proj-variants consistently yield performance gains of up to 21.25\% over the corresponding non-graph TTA baselines (\textbf{bold} across Table~\ref{table:syn}- \ref{table:arxiv}).
In particular, the performance of the \proj variants differs in how the boundary refinement addresses the label shift. 
\proj-T3A is less sensitive to the dominance of the majority class, as it refines the boundary based on distances to prototypes, offering increased robustness~\cite{zhang2022divide}. Therefore, \proj-T3A works well on imbalanced datasets such as MAG (Table~\ref{table:mag}) and Pileup (Table~\ref{table:pileup}).
We observe \proj-TENT and \proj-LAME perform well on Arxiv and ACM/DBLP (Table~\ref{table:arxiv}), as they directly maximize the most probable class, with \proj-LAME introducing a regularization overhead.

\textbf{Analysis of $\alpha$.}
% \label{subsec:alpha}
The learned $\alpha$ depends on node degrees and GNN layer depth. The trend of $\alpha$ in Fig.~\ref{fig:snr} (c) aligns with our expectations: earlier GNN layers require less attention to neighborhood-aggregated representations (small $\alpha$), while deeper layers increasingly rely on them. 
% Additionally, the plot reveals a correlation with node degrees, where nodes with higher degrees (and thus higher SNR ratios) place greater weight on neighborhood-aggregated messages. 
Additionally, the plot reveals that nodes with higher degrees (and thus higher SNR ratios) place greater weight on neighborhood-aggregated messages.
However, the overall effect of different GNN layers appears to be greater than that of varying node degrees. 
Lastly, the learned  $\alpha$ exhibits a trend similar to the source SNR ratio shown in Fig.~\ref{fig:snr} (a), assigning lower weights in layer 1 and nearly identical weighting patterns in layer 2 and 3.
\section{LIMITATIONS}
\label{app:lim}
\vspace{-3mm}

% While \proj can address feature shift to some extent, we observe that \proj’s benefits can be marginal when feature shift dominates significantly over structure shift, as evidenced in the last two columns of Pileup. This is primarily because the neighborhood alignment strategy is not robust in such scenarios. Moreover, our current estimation of $\mgamma$ heavily depends on the accuracy of refined soft pseudo-labels. While confusion-matrix-based methods~\cite{lipton2018detecting, azizzadenesheli2019regularized, alexandari2020maximum} can provide more robust estimations, they often require additional adversarial training, which is not compatible with GTTA constraints. We encourage future work to explore more robust approaches to estimating $\mgamma$ to further enhance adaptation under severe feature shifts.

While \proj effectively addresses distribution shifts in many settings, we identify specific limitations where its assumptions or mechanisms face challenges. We openly discuss these boundaries to clarify the method's scope to inspire future work.

\textbf{Severe Feature Shifts.} We observe that \proj's benefits can be marginal when feature shift significantly dominates over structure shift, as evidenced in the last two columns of the Pileup dataset. This is primarily because the neighborhood alignment strategy is not robust in such extreme scenarios. Furthermore, our current estimation of $\mgamma$ heavily depends on the accuracy of refined soft pseudo-labels. While confusion-matrix-based methods~\cite{lipton2018detecting, azizzadenesheli2019regularized, alexandari2020maximum} can provide more robust estimations, they often require additional adversarial training, which is incompatible with GTTA constraints. Exploring robust, GTTA-compatible approaches to estimating $\mgamma$ under severe feature shifts remains a crucial direction for future work.

\textbf{Violation of the Markov Assumption.} \proj calculates edge weights based on local pairwise conditional probabilities. If the graph exhibits strong long-range dependencies or subgraph correlations that extend beyond the 1-hop neighborhood, the pairwise term becomes an insufficient statistic. In such cases, the model fails to account for latent variables or hidden structures that co-determine the labels of interconnected nodes, leading to biased alignment weights.

\textbf{Non-Homogeneity of Neighborhood Distributions.} The current formulation estimates a single global $\mgamma$ for the entire graph. If the target graph is heavily non-homogeneous---meaning different clusters exhibit distinctly different structural shifts---this global expectation will be insufficient. A single global parameter fails to capture local variations in neighborhood structures, which results in improved alignment within dominant clusters while inadvertently degrading performance in minority or structurally distinct regions.

\section{BROADER IMPACT}
\vspace{-2mm}
\label{app:broader}
Our work studies structure shift and proposes \proj for graph test-time adaptation (GTTA). The goal of GTTA is to enhance model performance under distribution shifts with computation and privacy constraints. We do not foresee any direct negative societal impacts.

\section{CONCLUSION}
\vspace{-2mm}
In this work, we theoretically analyze the impact of various graph shifts on the generalization gap in GTTA and highlight their empirical effects on test-time performance degradation. Based on these insights, we propose \proj, a method that addresses structure shift through neighborhood alignment, SNR-inspired adjustment, and boundary refinement. Extensive experiments validate the effectiveness and efficiency of \proj.

\section{ACKNOWLEDGMENT}
\vspace{-2mm}
H. Hsu, S. Liu, and P. Li are partially supported by NSF awards PHY-2117997, IIS-2239565, IIS-2428777, and CCF-2402816; DOE award DE-FOA-0002785; JPMC faculty awards; and Microsoft Azure Research Credits for Generative AI.

{
% \small
\bibliographystyle{bib_style}
\bibliography{main}
}

%%%%%%%%%%%%%%%%%%%%%%%%%%%%%%%%%%%%%%%%%%%%%%%%%%%%%%%%%%%%
\section*{Checklist}

% % %%% BEGIN INSTRUCTIONS %%%
% The checklist follows the references. For each question, choose your answer from the three possible options: Yes, No, Not Applicable.  You are encouraged to include a justification to your answer, either by referencing the appropriate section of your paper or providing a brief inline description (1-2 sentences). 
% Please do not modify the questions.  Note that the Checklist section does not count towards the page limit. Not including the checklist in the first submission won't result in desk rejection, although in such case we will ask you to upload it during the author response period and include it in camera ready (if accepted).

% \textbf{In your paper, please delete this instructions block and only keep the Checklist section heading above along with the questions/answers below.}
% % %%% END INSTRUCTIONS %%%

\begin{enumerate}

  \item For all models and algorithms presented, check if you include:
  \begin{enumerate}
    \item A clear description of the mathematical setting, assumptions, algorithm, and/or model. [Yes] Please see Section \ref{sec:ttsa}.
    \item An analysis of the properties and complexity (time, space, sample size) of any algorithm. [Yes] We provide a runtime complexity analysis in Appendix.
    \item (Optional) Anonymized source code, with specification of all dependencies, including external libraries. [Yes] We provide anonymized source code in the supplementary materials.
  \end{enumerate}

  \item For any theoretical claim, check if you include:
  \begin{enumerate}
    \item Statements of the full set of assumptions of all theoretical results. [Yes] Please see Theorem \ref{theory:decomposeerror}.
    \item Complete proofs of all theoretical results. [Yes] Please see Appendix.
    \item Clear explanations of any assumptions. [Yes] Please see Theorem \ref{theory:decomposeerror}.
  \end{enumerate}

  \item For all figures and tables that present empirical results, check if you include:
  \begin{enumerate}
    \item The code, data, and instructions needed to reproduce the main experimental results (either in the supplemental material or as a URL). [Yes] We provide anonymized source code in the supplementary materials.
    \item All the training details (e.g., data splits, hyperparameters, how they were chosen). [Yes] Please see Appendix Experimental Setup.
    \item A clear definition of the specific measure or statistics and error bars (e.g., with respect to the random seed after running experiments multiple times). [Yes] Please see Appendix Experimental Setup and Section \ref{sec:Experiments}.
    \item A description of the computing infrastructure used. (e.g., type of GPUs, internal cluster, or cloud provider). [Yes] Please see Appendix Experimental Setup.
  \end{enumerate}

  \item If you are using existing assets (e.g., code, data, models) or curating/releasing new assets, check if you include:
  \begin{enumerate}
    \item Citations of the creator If your work uses existing assets. [Yes] We include citations for existing assets.
    \item The license information of the assets, if applicable. [Yes]
    \item New assets either in the supplemental material or as a URL, if applicable. [Yes] We provide anonymized source code in the supplementary materials.
    \item Information about consent from data providers/curators. [Yes]
    \item Discussion of sensible content if applicable, e.g., personally identifiable information or offensive content. [Not Applicable]
  \end{enumerate}

  \item If you used crowdsourcing or conducted research with human subjects, check if you include:
  \begin{enumerate}
    \item The full text of instructions given to participants and screenshots. [Not Applicable]
    \item Descriptions of potential participant risks, with links to Institutional Review Board (IRB) approvals if applicable. [Not Applicable]
    \item The estimated hourly wage paid to participants and the total amount spent on participant compensation. [Not Applicable]
  \end{enumerate}

\end{enumerate}

% % % % %%%%%%%%%%%%%%%%%%%%%%%%%%%%%%%%%%%%%%%%%%%%%%%%%%%%%%%%%%%%

\clearpage
\appendix
\onecolumn
\aistatstitle{Appendix for Structural Alignment Improves \\Graph Test-Time Adaptation}
% \shikun{Paper writing timeline
% DDL: 1/30 AOE
% \begin{itemize}
%     \item 1/10 main experimental results+first draft on everything else except theory and experimental discussion
%     \item 1/17 theory part finalize, second round of paper revision, finish ablation exps and exp result writing
%     \item 1/22 paper polish and finish most images (1/20 Abstract submission)
%     \item 1/26 appendix + continue polish
%     \item 1/28 code base cleanup, everything almost check
% \end{itemize}
% }

\section{OMITTED PROOFS}

% \subsection{CSS Decomposition}
% \decomposeCSS*

% \begin{proof} 
 
% \begin{align*}
%     \prob(\rmA|\rmY)
%     &=\prod_{u,v}\prob(A_{u,v} | Y_u, Y_v)  \\
%     %  \\
%     % &=
%     %  \prod_{u,v}\sum_{\gD}\prob(D_{u,v} | Y_u, Y_v) \prob\prob(e | Y_u, Y_v, D_{u,v})
%     %  \\
%     %  &=
%     %  \\
%     &=\prod_{u\in\gV}\prob(E_{u\leftarrow \neighbor_u} | Y_u) \tag{a}
%     \\
%     &=\prod_{u\in\gV}\sum_{\gD}\prob(D_u | Y_u)\prob(E_{u\leftarrow\neighbor_u} | Y_u, D_u)
%     \\    
%     &=\prod_{u\in\gV}\sum_{\gD}\prob(D_u | Y_u)\prod_{v\in \neighbor_u}\prob(E_{u\leftarrow v} | Y_u, D_u) \tag{b}
%     \\    
%     &=\prod_{u\in\gV}\sum_{\gD}\prob(D_u | Y_u)\prod_{v \in \neighbor_u}(\sum_{\gY}\prob(E_{u\leftarrow v} | Y_u, Y_v, D_u)\prob(Y_v|Y_u, v\in \neighbor_u))
%     \tag{c}
%     \\
% \end{align*}
% (a) We define $E_{u\leftarrow \neighbor_u}$ as a random variable indicates node $u$ receiving connections.
% Given our edges are independently sampled from labels, the conditional structure probability can be rewritten as the joint probability of $\prob(E_{u\leftarrow \neighbor_u} | Y_u)$ with respect to nodes.
% (b) The probability $\prob(E_{u\leftarrow v} | Y_u, D_u)$ represents the probability that the center node $u$ receive an edge from a node $v$ conditioned on node label and degree.
% (c) Based on our assumption, $\prob(E_{u\leftarrow v} | Y_u, Y_v, D_u)$ should be a constant among all data generation processes. However $\prob(D_u | Y_u)$ and $\prob(Y_v|Y_u, v\in \neighbor_u)$ may change in different doamin.
% \end{proof}

\subsection{Proof of Theorem \ref{theory:decomposeerror}}
\label{app:thm33_proof}

\begin{lemma}
\label{error decomposition lemma}
    Assume $\probm \coloneq 
    \alpha_j 
    % \omegas{u_k|u_j}
    \probs(\predy_u=i| \rvY_u=j)
    + 
    \beta_j 
    \probt(\predy_u=i| \rvY_u=j)$ to be a mixture conditional probability of source and target domain. The mixture weight is given by $\alpha_j$ and $\beta_j$ where $\forall \alpha_j,\beta_j \geq 0$ and $\alpha_j+\beta+j=1$, the following upper bound holds:
    \begin{align*}
    &\abs{\gammas{u_j}\probs(\predy_u=i| \rvY_u=j) - \gammat{u_j}\probt(\predy_u=i| \rvY_u=j} \\
    &\leq\abs{\gammas{u_j} - \gammat{u_j}} \cdot 
    \probm
    + (\gammas{u_j}\beta_j+\gammat{u_j}\alpha_j)\abs{ \probs(\predy_u=i| \rvY_u=j) -  \probt(\predy_u=i| \rvY_u=j)}
    \end{align*}

\begin{proof}
    To simplify the derivation we first abuse the notation by letting $\probs$ denote $\probs(\predy_u=i| \rvY_u=j, \rvW_u=k)$ and $\probt$ denote $\probt(\predy_u=i| \rvY_u=j, \rvW_u=k)$.
    \begin{align*}
    &\abs{\gammas{u_j}\probs - \gammat{u_j}\probt} 
    - \abs{\gammas{u_j}- \gammat{u_j}} \cdot \probm \\
    &\leq \abs{\gammas{u_j}\probs - \gammat{u_j}\probt - 
    \rpar{\gammas{u_j} - \gammat{u_j}}\probm}\\
    &=\abs{\gammas{u_j}\rpar{\probs-\probm} - \gammat{u_j}\rpar{\probt-\probm}} \\
    &\leq \gammas{u_j}\abs{\probs-\probm} + \gammat{u_j} \abs{\probt-\probm}
    \end{align*}
    In order to simplify the first term we substitute the definition of $\probm \coloneq \alpha_j \probs
    + \beta_j \probt$:
    \begin{align*}
        \abs{\probs-\probm} = \abs{\probs - \alpha_j \probs -\beta_j\probt} 
        =  \abs{ \beta_j\probs -\beta_j \probt}
    \end{align*}
    Similarly for the second term:

    \begin{align*}
        \abs{\probt-\probm} = \abs{\probt - \alpha_j \probs -\beta_j \probt} = \abs{\alpha_j \probt -\alpha_j \probs}
    \end{align*}
    Using the two identities, we can proceed with the derivation:
    \begin{align*}
         &\abs{\gammas{u_j}\probs - \gammat{u_j}\probt} - \abs{\gammas{u_j}
         - \gammat{u_j}} \cdot \probm \\
    &\leq \gammas{u_j}\beta_j\abs{ \probs - \probt} + \gammat{u_j}\alpha_j\abs{\probt-\probs} \\
    &= (\gammas{u_j}\beta_j+\gammat{u_j}\alpha_j)\abs{ \probs - \probt}
    \end{align*}
\end{proof}
\end{lemma}

\decomposeerror*

\begin{proof}
     
    We start by converting the average error rate into individual error probabilities. For sufficiently large $|\node|$ we have $|\node|\approx |\node|-1$. Given $|\nodes|\rightarrow\infty$ and $|\nodet|\rightarrow\infty$, we can have $|\node|\approx |\nodes| \approx |\nodet|$. By applying the triangle inequality and assuming that the graphs are sufficiently large, we obtain the following inequality:
    \begin{align*}
        &\abs{\errors(h \circ g) - \errort(h \circ g)} \\
       = &\abs{\probs(g(\phi(X_u, \rmA))\neq Y_u) - \probt(g(\phi(X_u, \rmA))\neq Y_u)}\\
       =&\abs{ \probs(\predy_u \neq \rvY_u) -  \probt(\predy_u \neq \rvY_u)} \\
    \end{align*}
To simplify the notation, define $\gammau{u_j}\coloneq \probu(\rvY_u=j)$, $\piu{v_k|u_j} \coloneq \probu(\rvY_{v}=k | \rvY_u=j, v\in \neighbor_u )$, and $\omegau{u_i|u_j,v_k}\coloneq \probu(\predy_u=i| \rvY_u=j, \rvY_{v}=k, v\in \neighbor_u)$ for $\mathcal{U}) \in \cbrace*{\mathcal{S}, \mathcal{T}}$. Using the the law of total probability and assuming that the label prediction depends on the distribution of neighborhood labels, we can derive the following identity:

\begin{align*}
    \probu(\predy_u \neq \rvY_u) &= \sum_{i\neq j}\probu(\predy_u=i, \rvY_u=j)= \sum_{i\neq j}\probu(\rvY_u=j) \probu(\predy_u=i| \rvY_u=j)\\
    % &=\sum_{i\neq j}\probu(\rvY_u=j) \sum_{d \in \mathcal{D}} \probu(|\neighbor_u|=d | \rvY_u=j)
    % \probu(\predy_u=i| \rvY_u=j, |\neighbor_u|=d)
    % \\
    % &\begin{aligned}
    %     =\sum_{i\neq j}\probu(\rvY_u=j) \sum_{d \in \mathcal{D}} \probu(|\neighbor_u|=d | \rvY_u=j) \bigg(\prod^{d}_{t=1} &\sum_{k \in \mathcal{Y}}  \probu(\rvY_{v_t}=k | \rvY_u=j, v_t\in\neighbor_u)  \\
    %     &\probu(\predy_u=i| \rvY_u=j, |\neighbor_u|=d, \set{\rvY_{v}}=k^d) \bigg)
    % \end{aligned}
    \\
      &=\sum_{i\neq j}\probu(\rvY_u=j)  \bigg(\sum_{k\in\mathcal{Y}}  \probu(\rvY_{v}=k | \rvY_u=j, v\in \neighbor_u )
        \probu(\predy_u=i| \rvY_u=j, \rvY_{v}=k, v\in \neighbor_u) \bigg)\\
    &=\sum_{i\neq j}\gammau{u_j}\sum_{k\in\mathcal{Y}}\piu{v_k|u_j}\omegau{u_i|u_j,v_k}
\end{align*}

% Plugging the identities into the above, we can show that
% \begin{align*}
%      &|\errors(h \circ g) - \errort(h \circ g)| \\
%      &\leq \frac{1}{|\node|}\sum_{u \in \node} |\sum_{i\neq j} \gammas{u_j} \sum_{k\in \mathcal{W}} \omegas{u_k|u_j}\probs(\predy_u=i| \rvY_u=j, \rvW_u=k)
%      - \sum_{i\neq j}\gammat{u_j} \sum_{k\in \mathcal{W}} \omegat{u_k|u_j}\probt(\predy_u=i| \rvY_u=j, \rvW_u=k)| \\
%      &\leq \frac{1}{|\node|}\sum_{u \in \node}\sum_{i\neq j}\sum_{k\in \mathcal{W}}\abs{\gammas{u_j}\omegas{u_k|u_j}\probs(\predy_u=i| \rvY_u=j, \rvW_u=k) - \gammat{u_j}\omegat{u_k|u_j}\probt(\predy_u=i| \rvY_u=j, \rvW_u=k)}
% \end{align*}

Most GNNs  marginalize the effect of degree magnitude, meaning that only the expectation of the neighboring nodes' label distribution and the label distribution of the central node itself influence the prediction.
We use Lemma \ref{error decomposition lemma} to bound the terms above.
Since $\forall j \in \mathcal{Y},
\gammas{u_j} \text{ and } \gammat{u_j} \in [0,1]$ and we have $\alpha_j + \beta_j = 1$, we obtain:
\begin{align*}
    &|\errors(h \circ g) - \errort(h \circ g)| \\
    &\leq 
    % \frac{1}{|\node|}
    % \sum_{u \in \node}\abs{
    \sum_{i\neq j}\gammas{u_j}\probs(\predy_u=i| \rvY_u=j) - \sum_{i\neq j}\gammat{u_j}\probt(\predy_u=i| \rvY_u=j)
    % } 
    \\
    &\leq \sum_{i\neq j}
    \abs{\gammas{u_j}\probs(\predy_u=i| \rvY_u=j) - \gammat{u_j}\probt(\predy_u=i| \rvY_u=j)} 
    \tag{a}
    \\
    &\leq \sum_{i\neq j}
    \rpar*{\abs{\gammas{u_j} - \gammat{u_j}} \cdot 
    \probm
    +(\gammas{u_j}\beta_j + \gammat{u_j}\alpha_j )
    \abs{ \probs - \probt}  
    }
    \tag{b}
    \\
    &= \sum_{i\neq j}
    \rpar*{\abs{\gammas{u_j} - \gammat{u_j}} \cdot\probs(\predy_u=i| \rvY_u=j) + \gammat{u_j}
    \abs{\probs(\predy_u=i| \rvY_u=j) - \probt(\predy_u=i| \rvY_u=j)}
    }
    \tag{c}
    \\
    &= \sum_{i\neq j}\abs{\gammas{u_j} - \gammat{u_j}} \cdot\probs(\predy_u=i| \rvY_u=j) + \sum_{i\neq j}\gammat{u_j}
    \abs{\probs(\predy_u=i| \rvY_u=j) - \probt(\predy_u=i| \rvY_u=j)}
    \\
    &\leq 
    \rpar*{\sum_{j\in\gY}\abs{\gammas{u_j} - \gammat{u_j}} }
    \cdot\balerrorrate{\predy_u}{\rvY_u} 
    + 
    \sum_{i\neq j}\gammat{u_j}
        \abs{\probs(\predy_u=i| \rvY_u=j) - \probt(\predy_u=i| \rvY_u=j)}
    \tag{d}
    \\
    &=2TV\rpar*{\probs(\rvY_u=j), \probt(\rvY_u=j)}\cdot\balerrorrate{\predy_u}{\rvY_u}+ 
    \sum_{i\neq j}\gammat{u_j}
        \abs{\probs(\predy_u=i| \rvY_u=j) - \probt(\predy_u=i| \rvY_u=j)}
    \tag{e} 
\end{align*}
By applying the triangle inequality, (a) is obtained.
Following Lemma \ref{error decomposition lemma} we have (b).
(c) is obtained by choosing $\alpha_j=1$ and $\beta_j=0, \forall j \in \mathcal{Y}$.
(d) is derived by applying Holder's inequality.
(e) is based on the definition of total variation distance where 
    $TV(\probs(\rvY_u=j), \probt(\rvY_u=j))=\frac{1}{2}\sum_{j\in\gY}\abs{\probs(\rvY_u=j) - \probt(\rvY_u=j)}$ is the total variation distance between the label distribution of source and target.

Similarly, the second term can be decomposed by Lemma  \ref{error decomposition lemma} while setting $\omegam{u_i|u_j,v_k} \coloneq 
    \alpha_j 
     \probs(\predy_u=i| \rvY_u=j, \rvY_{v}=k, v\in \neighbor_u)
    + 
    \beta_j 
     \probt(\predy_u=i| \rvY_u=j, \rvY_{v}=k, v\in \neighbor_u)$, $\omegas{u_i|u_j,v_k}\coloneq \probs(\predy_u=i| \rvY_u=j, \rvY_{v}=k, v\in \neighbor_u)$, and $\omegat{u_i|u_j,v_k}\coloneq \probt(\predy_u=i| \rvY_u=j, \rvY_{v}=k, v\in \neighbor_u)$. 
The decomposition is instead with respect to $\piu{v_k|u_j} \coloneq \probu(\rvY_{v}=k | \rvY_u=j, v\in \neighbor_u )$.

\begin{align*}
    &\sum_{i\neq j}\gammat{u_j}
        \abs{\probs(\predy_u=i| \rvY_u=j) - \probt(\predy_u=i| \rvY_u=j)}
    \\
    &=\sum_{i\neq j}\gammat{u_j} 
    \begin{aligned}[t]
        &\abs{ \sum_{k\in\mathcal{Y}}  
        \probs(\rvY_{v}=k | \rvY_u=j, v\in \neighbor_u )
        \probs(\predy_u=i| \rvY_u=j, \set{\rvY_{v}:v\in \neighbor_u}) \\
        &\quad - \sum_{k\in\mathcal{Y}}  
        \probt(\rvY_{v}=k | \rvY_u=j, v\in \neighbor_u )
        \probt(\predy_u=i| \rvY_u=j, \set{\rvY_{v}:v\in \neighbor_u})
        }
    \end{aligned}
    \\
    % & =\sum_{i\neq j}\probu(\rvY_u=j)  \bigg(\sum_{k\in\mathcal{Y}}  \probu(\set{\rvY_{v}:v\in \neighbor_u} | \rvY_u=j)
    %     \probu(\predy_u=i| \rvY_u=j, \set{\rvY_{v}:v\in \neighbor_u}) \bigg)
    %     \abs*{\prod^{d}_{t=1} \sum_{k \in \mathcal{Y}}  \probs(\rvY_{v}=k | \rvY_u=j, v\in\neighbor_u)
    %     \probt(\predy_u=i| \rvY_u=j, \set{\rvY_{v}}=k^d)
    %     -
    %     \prod^{d}_{t=1} \sum_{k \in \mathcal{Y}} \probt(\rvY_{v}=k | \rvY_u=j, v\in\neighbor_u) \probt(\predy_u=i| \rvY_u=j, \set{\rvY_{v}}=k^d)
    %     }
    &\leq\sum_{i\neq j} \gammat{u_j}\sum_{k\in\mathcal{Y}}
        \abs{  \pis{v_k|u_j}
        \omegas{u_i|u_j,v_k} 
        -
        \pit{v_k|u_j} \omegat{u_i|u_j,v_k}
        }
    \\
    &\leq\sum_{i\neq j} \gammat{u_j}\sum_{k\in\mathcal{Y}}\rpar*{
    \abs{\pis{v_k|u_j} -\pit{v_k|u_j} }\cdot \omegam{u_i|u_j,v_k} 
    + 
    (\pis{v_k|u_j}  \beta_j+\pit{v_k|u_j}  \alpha_j)
    \abs{\omegas{u_i|u_j,v_k}
    - 
    \omegat{u_i|u_j,v_k}}}
    \tag{a}
    \\
     &=\sum_{i\neq j} \gammat{u_j} \sum_{k\in\mathcal{Y}}
    \abs{\pis{v_k|u_j} -\pit{v_k|u_j} }\cdot \omegas{u_i|u_j,v_k} 
    + 
    \sum_{i\neq j} \gammat{u_j}\sum_{k\in\mathcal{Y}}
   \pit{v_k|u_j}\
    \abs{\omegas{u_i|u_j,v_k}
    - 
    \omegat{u_i|u_j,v_k}}
    \tag{b}
    \\
     &=\sum_{i\neq j} \gammat{u_j}  \sum_{k\in\mathcal{Y}}
    \pis{v_k|u_j}\abs{1 -\frac{\pit{v_k|u_j}}{\pis{v_k|u_j}} }\cdot \omegas{u_i|u_j,v_k} 
    + 
    \sum_{i\neq j} \gammat{u_j}\sum_{k\in\mathcal{Y}}
   \pit{v_k|u_j}\
    \abs{\omegas{u_i|u_j,v_k}
    - 
    \omegat{u_i|u_j,v_k}}
    \\    
    &\leq \sum_{i\neq j} \gammat{u_j}  \rpar*{\max \abs{1-\frac{\probt(Y_v=k|Y_u=j,v\in \neighbor_u)}{\probs(Y_v=k|Y_u=j,v\in \neighbor_u)}}} \cdot \probs(\predy_u=i| \rvY_u=j) 
    + 
    \sum_{i\neq j} \gammat{u_j}\sum_{k\in\mathcal{Y}}
   \pit{v_k|u_j}\
    \abs{\omegas{u_i|u_j,v_k}
    - 
    \omegat{u_i|u_j,v_k}}    
    \tag{c}\\
    &\leq \E^\gT\spar*{\rpar*{\max \abs{1-\frac{\probt(Y_v=k|Y_u=j,v\in \neighbor_u)}{\probs(Y_v=k|Y_u=j,v\in \neighbor_u)}}}}  \balerrorrate{\predy_u}{\rvY_u}
    + 
    \sum_{i\neq j} \gammat{u_j}\sum_{k\in\mathcal{Y}}
   \pit{v_k|u_j}\
    \abs{\omegas{u_i|u_j,v_k}
    - 
    \omegat{u_i|u_j,v_k}}
    \\
    &\leq \E_{Y_u}^\gT\spar*{\rpar*{\max \abs{1-\frac{\probt(Y_v=k|Y_u=j,v\in \neighbor_u)}{\probs(Y_v=k|Y_u=j,v\in \neighbor_u)}}}}  \balerrorrate{\predy_u}{\rvY_u}
    + 
    \max_{i\
    \neq j}\E^\gT\spar*{|\omegas{u_i|u_j,v_k} - \omegat{u_i|u_j,v_k} |}
    \\
    &\begin{aligned}[t]
    \leq \E_{Y_u}^\gT&\spar*{\rpar*{\max \abs{1-\frac{\probt(Y_v=k|Y_u=j,v\in \neighbor_u)}{\probs(Y_v=k|Y_u=j,v\in \neighbor_u)}}}}  \balerrorrate{\predy_u}{\rvY_u}\\
    &+ \max_{i\neq j}\E^\gT_{\rvY_{v}}\spar{|\probs(\predy_u=i| \rvY_u=j, \rvY_{v}=k, v\in \neighbor_u)
    - \probt(\predy_u=i| \rvY_u=j, \rvY_{v}=k, v\in \neighbor_u) |}
    \end{aligned}
    \\
    &\leq \E_{Y_u}^\gT\spar*{\rpar*{\max \abs{1-\frac{\probt(Y_v=k|Y_u=j,v\in \neighbor_u)}{\probs(Y_v=k|Y_u=j,v\in \neighbor_u)}}}}  \balerrorrate{\predy_u}{\rvY_u} + \Delta_{CE} \tag{d}
\end{align*}

% \hans{add explanation (b)}
(a) comes from Lemma \ref{error decomposition lemma}. (b) comes from setting $\alpha_j=1$ and $\beta_j=0, \forall j \in \mathcal{Y}$
(c) The inequality comes from Holder's inequality, where we take the maximum-norm for $k$ and take the $L_1$ norm to marginalize $Y_v$: $\probs(\predy_u=i| \rvY_u=j) =\sum_{k\in\mathcal{Y}}  \probu(\rvY_{v}=k | \rvY_u=j, v\in \neighbor_u )\probu(\predy_u=i| \rvY_u=j, \rvY_{v}=k, v\in \neighbor_u)$
(d) The last term in the previous line corresponds to the conditional error gap in \cite{tachet2020domain} but in a standard GNN's message passing fashion.

% The last term $\probu(\predy_u=i| \rvY_u=j, \rvY_{v}=k, v\in \neighbor_u)$ we introduce a random variable of the one-hop structure $\rmA \sim \probu_j(N)$ that samples from the histogram of neighborhood distribution N given the center node being $j$.
The term $\probu(\predy_u=i| \rvY_u=j, \rvY_{v}=k, v\in \neighbor_u)$ can be interpreted as the probability of error prediction using a GNN encoder.
By definition, $Y_v$ represents a random variable of the neighborhood label ratio and $Y_u$ denotes a random variable of self class label, their realization corresponds to the expectation of the neighboring nodes’ label distribution and the label distribution of the central node itself.
As a result, $\Delta_{CE}$ only exists if and only if feature shift exists. Putting everything together, we have

% Knowing that $Y_v$ represents a random variable of the neighborhood label ratio and $Y_u$, we can conclude to know the random vector $N$. 

% Hence we can rewrite the term to $\probu(\predy_u=i| \rvY_u=j, \rmA)$.

% % \begin{align*}
% %  \Delta_{CE}\coloneq\max_{i\neq j}\E^\gT_{\rvY_{v}}\spar{|\probs(\predy_u=i| \rvY_u=j, \rmA)
% %     - \probt(\predy_u=i| \rvY_u=j, \rmA) |}
% % \end{align*}

% As $\probu(\predy_u=i| \rvY_u=j, \rmA)$  conditioned on its self label and the graph structure, $\Delta_{CE}$ only exists if and only if feature shift exists. Putting everything together, we have

\vspace{-6mm}
\begin{align*}
&\abs{\errors(g \circ \phi) - \errort(g \circ \phi)} 
\\
&\leq  
\balerrorrate{\predy_u}{\rvY_u}\cdot
    \cbrace{\underbrace{
    \rpar{2TV(\probs(\rvY_u), \probt(\rvY_u)}
    }_{\text{Label Shift}}
    +\underbrace{\E_{Y_u}^\gT\spar{\max_{k\in \gY} \abs{1-\frac{\probt(Y_v=k|Y_u,v\in \neighbor_u)}{\probs(Y_v=k|Y_u,v\in \neighbor_u)}}}  }_{\textbf{Conditional Structure Shift}}}
    + \underbrace{\Delta_{CE}}_{\text{Feature Shift}}
    \vspace{-2mm}
\end{align*}

\end{proof}

%%%%%%%%%%%%%%%%%%%%%%%%%%%%%%%%%%%%%%%%%%%%%%%%%%%%%%%%%%%

\subsection{Finite-Sample Guarantees of Theorem 3.3}
\label{app:finite_sample}

To extend the theoretical foundation of Theorem 3.3, we relax the asymptotic assumption ($|\mathcal{V}| \rightarrow \infty$) and provide a finite-sample guarantee. We demonstrate that for finite graphs consisting of independent ego-networks, the empirical error gap is bounded by the theoretical shift terms plus a statistical variance term that decays with sample size.

Let $\epsilon^{\mathcal{S}}$ and $\hat{\epsilon}^{\mathcal{S}}$ denote the expected and empirical risk on the source, respectively (and similarly for target $\mathcal{T}$). By decomposing the error gap via triangle inequality, we have:
\begin{align*}
    |\hat{\epsilon}^{\mathcal{S}} - \hat{\epsilon}^{\mathcal{T}}| \le \underbrace{|\hat{\epsilon}^{\mathcal{S}} - \epsilon^{\mathcal{S}}|}_{\text{Source Est. Error}} + \underbrace{|\epsilon^{\mathcal{S}} - \epsilon^{\mathcal{T}}|}_{\text{Distribution Shift (Thm 3.3)}} + \underbrace{|\epsilon^{\mathcal{T}} - \hat{\epsilon}^{\mathcal{T}}|}_{\text{Target Est. Error}}
\end{align*}

The middle term, $|\epsilon^{\mathcal{S}} - \epsilon^{\mathcal{T}}|$, is strictly bounded by the Label, Neighborhood, and Feature shift components derived in Theorem 3.3. We bound the estimation error terms using McDiarmid's Inequality. Following the formulation in Theorem 3.3, we assume the graph $\mathcal{G}$ consists of $M$ independent samples (ego-networks) $\{Z_1, Z_2, ..., Z_M\}$, where $M = |\mathcal{V}|$. Let the empirical error be a function $f: \mathcal{Z}^M \rightarrow [0, 1]$ defined as the average 0-1 loss over the nodes:
\begin{align*}
    f(Z_1, ..., Z_M) = \frac{1}{M}\sum_{i=1}^{M} \mathbb{I}(g(Z_i) \ne y_i)
\end{align*}

Since the loss function $\mathbb{I}(\cdot)$ is bounded within $[0, 1]$, changing a single ego-network $Z_i$ to $Z_i^\prime$ changes the value of $f$ by at most $\frac{1}{M}$. Thus, applying McDiarmid's Inequality, with probability at least $1-\delta$, the estimation error is bounded by:
\begin{align*}
    \eta = \sqrt{\frac{\ln(2/\delta)}{2|\mathcal{V}|}}
\end{align*}

Applying this bound to both the source and target estimation errors, we conclude that with probability at least $1 - 2\delta$:
\begin{align*}
    |\hat{\epsilon}^{\mathcal{S}} - \hat{\epsilon}^{\mathcal{T}}| \le \underbrace{\text{Bound}_{\text{Thm 3.3}}}_{\text{Shift Components (CSS, Label, Feature)}} + \underbrace{\sqrt{\frac{\ln(2/\delta)}{2|\mathcal{V}^{\mathcal{S}}|}} + \sqrt{\frac{\ln(2/\delta)}{2|\mathcal{V}^{\mathcal{T}}|}}}_{\text{Finite Sample Noise}}
\end{align*}
This shows that for finite graphs, the error decomposition in Theorem 3.3 holds, subject to a standard complexity term $\mathcal{O}(1/\sqrt{|\mathcal{V}|})$ that vanishes as graph size increases.

\subsection{Justification of the Ego-Network Assumption}
\label{app:ego_network}

Theorem 3.3 assumes that the joint distribution of the entire graph can be factorized into the product of local marginal distributions centered at each node. Namely, it treats the graph as a ``bag'' of independent training examples. While we agree that this assumption is not fully satisfied in real-world graphs, it is grounded in practical graph learning properties:
\begin{itemize}
    \item \textbf{Decaying Correlation:} In many physical and social networks, the correlation between nodes decays as the graph distance increases. The influence of a distant node on a center node is statistically negligible compared to the influence of immediate neighbors.
    \item \textbf{Local Dominance:} The assumption captures the core characteristic of the data generation process for node classification: the immediate structural context (the ego-network) contains the highest density of discriminatory signals. Therefore, modeling the data as a collection of local environments is a justifiable approximation for analyzing distribution shifts, even if it ignores weaker high-order dependencies.
\end{itemize}
Our empirical results demonstrate that while the independence assumption is a simplification, it successfully captures the dominant factors driving distribution shift in practice on real graphs.

\subsection{Clarification on Feature Shift ($\Delta_{CE}$)}
\label{app:feature_shift}

\textbf{Feature Shift Definition:} We define Feature Shift as a divergence in the conditional feature distribution given the node label, expressed as $\mathbb{P}^{\mathcal{S}}(X_u|Y_u) \ne \mathbb{P}^{\mathcal{T}}(X_u|Y_u)$.

\textbf{Correspondence to $\Delta_{CE}$:} In Appendix \ref{app:thm33_proof}, $\Delta_{CE}$ is derived as the discrepancy in prediction probabilities between domains, conditioned on both the center label $Y_u$ and the neighborhood labels $Y_v$. Because we condition on the label topology and the pre-trained model parameters are frozen, the only variable remaining that can alter the prediction probability is the input feature distribution $X$. Thus, $\Delta_{CE}$ isolates the error gap caused specifically by feature shift.

% \scaledbootstrap*

% \begin{proof} Since $\alpha^{(k)}$ is a constant with respect to the aggregation function, by distributivity, we have  $h_u^{(k+1)} = \text{UPT}(h_u^{(k)}, \text{AGG}(\alpha \{\{h_{v}^{(k)} \mid v \in \mathcal{N}_u\}\}))$, which denotes sampling the elements $h_v^{(k)}$ by a ratio of $\alpha$. In practice, this sampling strategy can be interpreted as aggreagating over weighted edges, where self-loop edges have a weight of 1 and the rest have a weight of $\alpha$.
    
% \end{proof}

% \section{Edge Weights $\bar{\gamma}_{u,v}$}
% If the mean pooling follows symmetric normalization, $\text{AGG}(\{h_{v}^{(k)}, v \in \neighbor_u\}) \coloneqq
% \frac{1}{\sqrt{|\neighbor_u|}\sqrt{|\neighbor_v|}}\sum_{v \in \neighbor_u}h_v$, the edge weights are assigned to $\bar{\gamma}_{u,v}$ following:

% \begin{equation}
%     \bar{\gamma}_{u,v} = \frac{[\gamma]_{i,j}\sqrt{|\neighbor_u|}\sqrt{|\neighbor_v|}}{\sum_{v \in \neighbor_u} \sqrt{[\gamma]}_{i,j}  {\sum_{u \in \neighbor_v} \sqrt{[\gamma]}_{i,j}}}, \, \text{where $Y_u = i$ and $Y_v=j$}
% \end{equation}

% \newpage
\section{DISCUSSION OF USING SOURCE DOMAIN INFORMATION IN GTTA}

In this section, we delve deeper into the role of source domain information within GTTA setting. We argue that a truly "source-agnostic" approach is fundamentally limited and that effective adaptation necessitates leveraging, at minimum, summary statistics or assumptions from the source domain. We will illustrate this point through conceptual examples, theoretical arguments, and empirical results, demonstrating how explicitly accounting for source-side properties leads to more robust and reliable adaptation.

\subsection{The Necessity of Source Information as an Anchor for Adaptation}
We clarify a fundamental point: a purely \emph{source-agnostic} approach to GTTA is inherently limited. Effective adaptation requires anchoring the process with either explicit source statistics or an implicit assumption about the source domain (e.g., a clustering assumption). Attempting adaptation without such an anchor cannot sufficiently correct for the inherent biases learned during pre-training.

\paragraph{Illustrative Example (Fraud Detection).}
Consider two different models for financial fraud detection, both adapting to the same target domain of \emph{new mobile payment transactions}:
\begin{itemize}
  \item \textbf{Model A} is pre-trained on \emph{Source A (E-commerce Data)}, where fraud is typically characterized by \emph{high-frequency, low-value} transactions. This model is thus biased toward that pattern.
  \item \textbf{Model B} is pre-trained on \emph{Source B (Wire Transfer Data)}, where fraud more often appears as \emph{single, high-value} anomalous transfers. This model carries a different bias.
\end{itemize}
A single, source-agnostic TTA rule would apply the same adaptation logic to both models and therefore fail to correct their distinct pre-existing biases. In contrast, a method like \proj that uses explicit source statistics performs \emph{tailored} adaptation, counterbalancing each model’s specific bias. This intuition is formalized by our analysis and validated empirically.

\paragraph{Theoretical Demand.}
Theorem~\ref{theory:decomposeerror} decomposes the generalization error and makes explicit its dependence on both source and target shifts. The bound implies that to reduce this error, adaptation must account for properties of the source distribution.

% \paragraph{A Principled, Minimal Requirement.}
% To correct graph-specific class-structure shift (CSS), our method requires only the source domain’s \emph{conditional probability table}. This principle is validated in our synthetic and real experiments (Tables~\ref{tab:main}-\ref{tab:more}), where mitigating CSS substantially improves performance.

\paragraph{Lightweight Information.}
The required source summary is extremely lightweight: For \proj, a single table of probabilities that consumes negligible space and can be shipped alongside model parameters.

\subsection{The Implicit Source Assumptions in "Source-Agnostic" Methods}

\begin{table*}[h!]
\centering
\caption{GTTA performance when the source cluster assumption is satisfied (well‐separated clusters) and when it is violated (highly imbalanced source clusters) on synthetic datasets, with the target domain held constant.}
\vspace{-4mm}
\label{tab:assum_violate}
  \begin{center}
  \resizebox{0.9\textwidth}{!}{%
    \small
      \begin{tabular}{lccccccc}
        \toprule
        & ERM & GTrans & HomoTTT& SOGA
        & TSA-TENT & TSA-LAME & TSA-T3A \\
        \midrule
         Assumption \textbf{satisfied}&
        $82.70\pm4.45$ & $86.67\pm3.59$ & $85.08\pm2.87$ & $86.09\pm3.89$ &
        $88.78\pm1.37$ & $88.96\pm1.66$ & $\mathbf{89.96\pm1.33}$ \\
        \midrule
         Assumption \textbf{violated}&
        $64.38\pm4.25$ & $62.30\pm3.20$ & $60.44\pm4.95$ & $60.07\pm6.16$ &
        $\mathbf{74.97\pm5.41}$ & $70.77\pm4.76$ & $\mathbf{74.97\pm3.41}$ \\
        \bottomrule
      \end{tabular}
    }
  \end{center}
\end{table*}

As demonstrated in our experiments (Table~\ref{tab:assum_violate}), methods often regarded as source-agnostic (e.g., GTrans, HomoTTT, SOGA) rely critically on implicit structural assumptions about the source domain. In contrast, $\textsc{TSA}$ is more robust when those assumptions do not hold. Every method fails to transfer (accuracy decreases) except TSA.

\subsection{\proj Robustness in the Absence of Explicit Source Statistics}

\begin{table*}[h!]
\centering
\caption{Performance on MAG for US$\rightarrow$CN (top) and CN$\rightarrow$US (bottom) when using a uniform distribution for the source prior. \textbf{Bold} indicates improvements in comparison to the corresponding non-graph TTA baselines. \underline{Underline} indicates the best model.}
\vspace{-2mm}
\label{tab:no_source}
\resizebox{0.9\textwidth}{!}{%
\small
\begin{tabular}{lcccccccccc}
\toprule
ERM & GTrans & HomoTTT& SOGA
& TENT &LAME &T3A
& TSA-TENT & TSA-LAME & TSA-T3A \\
\midrule
$31.86\pm0.83$ 
& $31.77\pm0.91$ & $25.89\pm1.98$ & $21.54\pm2.52$ 
&26.72$\pm$1.33 &35.75$\pm$0.85 &41.47$\pm$1.15

&\textbf{27.24$\pm$1.51} & \textbf{37.83$\pm$1.30} & \underline{\textbf{41.56$\pm$1.04}} \\
\midrule
% CN->US (bottom) -- from the image
$37.41\pm1.01$
& $36.17\pm0.89$ & $24.39\pm2.12$ & $38.95\pm3.35$
& $27.21\pm0.88$ & $40.08\pm1.13$ & $46.50\pm1.26$
& \textbf{27.82$\pm$0.88} & \textbf{44.72$\pm$0.73} & \underline{\textbf{47.98$\pm$0.89}} \\

\bottomrule
\end{tabular}}
\end{table*}

Even for third-party models where source statistics are unknown, \proj remains effective. Assuming a simple uniform source prior, our method still delivers consistent improvements on real-world datasets, as shown in Table~\ref{tab:no_source}. This is because the alignment is then driven primarily by target-graph statistics, providing useful regularization for an SNR-inspired adjustment.

Together, these arguments support the central claim: \emph{effective GTTA requires lightweight source-domain information or, equivalently, explicit assumptions about it}. Methods that pretend to be fully source-agnostic inevitably encode—and are limited by—implicit source assumptions; making such information explicit enables more reliable, bias-aware adaptation.

\section{ADDITIONAL RESULTS}
\label{app:additional results}

\subsection{TSA Ablation Studies}
\label{app:tsa_ablation_study}

We conduct an ablation study to examine the individual contributions of neighborhood alignment and SNR-based representation adjustments. For this analysis, we use synthetic datasets with known structural shifts and the real-world datasets exhibiting the highest conditional structure shift (CSS), as detailed in Table~\ref{table:magstats}. 

The results, presented in Table~\ref{table:ablation_all}, consistently show that incorporating both modules yields the best performance. Removing neighborhood alignment leads to a significant performance drop of up to 19.9\% on the synthetic dataset and 3.16\% on the real-world dataset. The larger drop on the synthetic CSBM data is expected, given the more significant structural shift modeled. Meanwhile, removing SNR adjustments results in a performance decrease of up to 2.43\% and 2.97\%, respectively. These findings underscore that both modules are crucial for handling structure shift.

% , and a more detailed discussion is provided in Section~\ref{subsec:result_analysis}.

% In Table~\ref{table:ablation_all}, we examine the gains of neighborhood alignment and SNR-based representation adjustments on both synthetic and real-world datasets. The results consistently show that incorporating both modules yields the best performance.
% Removing \emph{neighborhood alignment} leads to a performance drop of up to 19.9\% on the synthetic dataset and 3.16\% on the real-world dataset, while removing \emph{SNR adjustments} results in a drop of up to 2.43\% and 2.97\%, respectively. 
% % These findings indicate that both modules are crucial components for handling structure shift.
% Removing neighborhood alignment causes a larger drop in CSBM due to the more significant shift used in the data modeling.

% In the ablation studies, we examine the effects of neighborhood alignment and SNR-based representation adjustments. We select synthetic datasets with structure shift, as well as real-world datasets that exhibit the most conditional structure shift (CSS), according to Table~\ref{table:magstats}. A detailed discussion of the ablation results is provided in Section~\ref{subsec:result_analysis}.

\begin{table}[h]
\scriptsize
\caption{TSA Ablation studies on CSBM and MAG.}
\label{table:ablation_all}
% \resizebox{\textwidth}{!}{%
\begin{center}
\resizebox{\textwidth}{!}{%
\begin{tabular}{c|l|cc|cc}
\toprule 
% Method &Ablation & US$\rightarrow$CN & US$\rightarrow$DE & US$\rightarrow$JP & US$\rightarrow$RU & US$\rightarrow$FR & CN$\rightarrow$US &CN$\rightarrow$DE & CN$\rightarrow$JP & CN$\rightarrow$RU & CN$\rightarrow$FR \\
Method &Ablation 
% & \multicolumn{2}{c}{Nbr. Shift} &\multicolumn{2}{c}{Nbr.+ SNR Shift} 
&\multicolumn{2}{c}{Struct. Shift (Imbal.$\rightarrow$ Bal.)} 
& US$\rightarrow$RU & US$\rightarrow$FR\\
 \midrule
\multirow{3}{*}{TSA-TENT} &\cellcolor{red!10}Base   
&\cellcolor{red!10}88.68$\pm$4.99
&\cellcolor{red!10}66.25$\pm$7.75
&\cellcolor{red!10}32.53$\pm$0.91 &\cellcolor{red!10}28.62$\pm$0.63
\\
&w/o Nbr. Align 
% &87.48$\pm$2.83 &77.18$\pm$4.64 &81.04$\pm$2.70 &72.54$\pm$3.38 
&78.45$\pm$5.51 &61.12$\pm$4.86 
&32.26$\pm$0.95 &28.41$\pm$0.59
 \\
&w/o SNR
% &88.76$\pm$1.37 &80.50$\pm$2.38 &83.20$\pm$1.46 &76.41$\pm$1.26 
&86.93$\pm$5.05 &63.82$\pm$8.06
 &32.52$\pm$0.90 &28.57$\pm$0.59

\\
 \midrule
\multirow{3}{*}{TSA-LAME} &\cellcolor{red!10}Base  
% &\cellcolor{red!10}88.96$\pm$1.66 
% &\cellcolor{red!10}80.02$\pm$5.44
% &\cellcolor{red!10}83.51$\pm$0.55
% &\cellcolor{red!10}79.56$\pm$1.82
&\cellcolor{red!10}65.09$\pm$2.34
&\cellcolor{red!10}52.89$\pm$6.11
&\cellcolor{red!10}32.86$\pm$2.23 
 &\cellcolor{red!10}27.22$\pm$1.48
\\
&w/o Nbr. Align 
&48.45$\pm$5.30 &37.31$\pm$3.95
 &32.70$\pm$2.29 &26.49$\pm$1.76
 \\
 &w/o SNR 
  &63.65$\pm$2.65 &51.18$\pm$5.84
   &29.89$\pm$1.94 &24.41$\pm$1.43 
\\
 \midrule
\multirow{3}{*}{TSA-T3A} &\cellcolor{red!10}Base
&\cellcolor{red!10}65.59$\pm$2.57
&\cellcolor{red!10}52.34$\pm$7.19
 &\cellcolor{red!10}46.61$\pm$0.88 &\cellcolor{red!10}43.45$\pm$0.81
 \\
 &w/o Nbr. Align 
  &45.69$\pm$5.44 &36.75$\pm$3.09
   &46.44$\pm$0.80 &42.09$\pm$1.26
 \\
 &w/o SNR 
  &64.17$\pm$2.11 &51.04$\pm$6.79
   &46.51$\pm$0.94 &41.35$\pm$1.09
 \\
% \midrule
%  \multicolumn{2}{c|}{w/o Boundary Refinement}&88.49$\pm$1.65 &77.82$\pm$5.84 &83.70$\pm$1.41 &77.65$\pm$2.73 
%  &66.18$\pm$2.90 &53.30$\pm$6.34
%  &82.29$\pm$5.22 &63.54$\pm$2.44  
% \\
\bottomrule
\end{tabular}
}
\end{center}
\end{table}

\newpage
\subsection{Visualization of Distribution Shift}

Figure~\ref{fig:diff_shift} (a)-(e) investigates the impact of different types of distribution shifts on GNNs.
We use a one-layer GraphSAGE followed by an MLP classifier to visualize node representations on synthetic datasets generated by the CSBM model.

\begin{itemize}
\item \textbf{Feature Shift:} Figure~\ref{fig:diff_shift} (b) shows that the representations deviate from the source domain, but the clusters remain separable.
\item \textbf{Conditional Structure Shift (CSS):} Figure~\ref{fig:diff_shift} (c) shows that the embeddings largely overlap, making the representation space less discriminative.
\item \textbf{Label Shift:} Figure~\ref{fig:diff_shift} (d) illustrates a biased decision boundary, which leads to misclassification of the minority class.
\item \textbf{SNR Shift:} Figure~\ref{fig:diff_shift} (e) shows that the node representations remain roughly aligned with the source, but are more dispersed, as the SNR shift introduces noise into the representations.
\end{itemize}

\begin{figure*}[h]
\centering
\begin{tabular}{c|cccc}
(a) Source &(b) Feature Shift & (c) CSS  &(d) Label Shift & (e) SNR Shift \\
\includegraphics[width=0.17\textwidth]{figures/aggre_src.pdf} &
\includegraphics[width=0.17\textwidth]{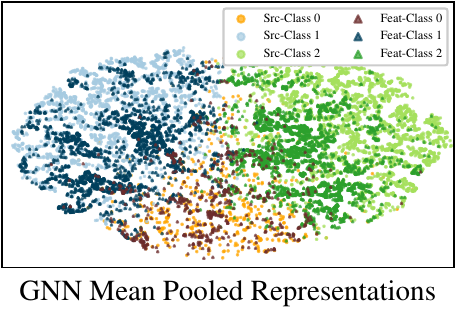} &
\includegraphics[width=0.17\textwidth]{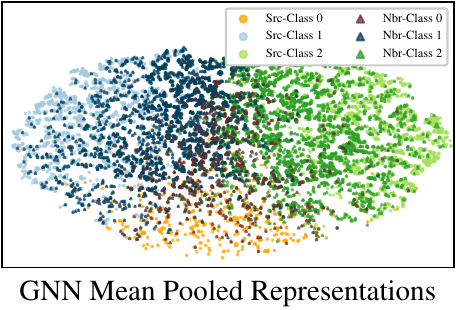} &
\includegraphics[width=0.17\textwidth]{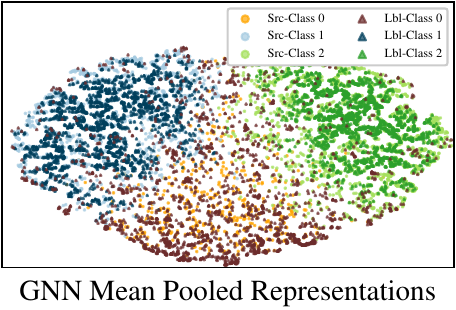} &
\includegraphics[width=0.17\textwidth]{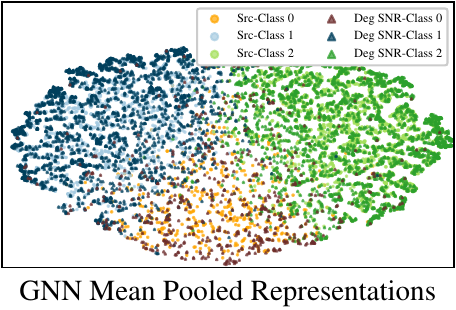} \\

\includegraphics[width=0.17\textwidth]{figures/output_src.pdf} &
\includegraphics[width=0.17\textwidth]{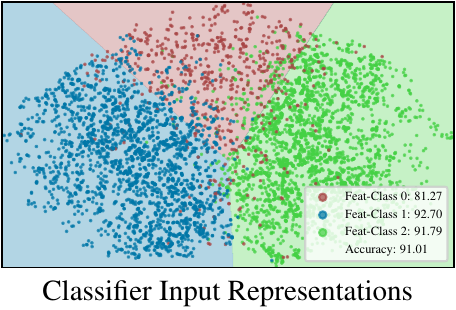} &
\includegraphics[width=0.17\textwidth]{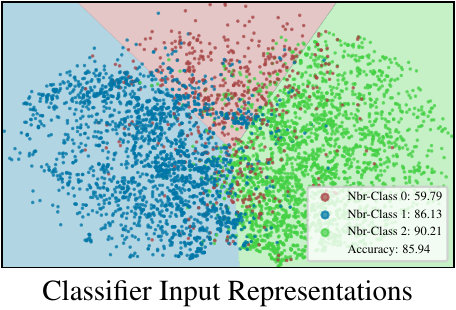} &
\includegraphics[width=0.17\textwidth]{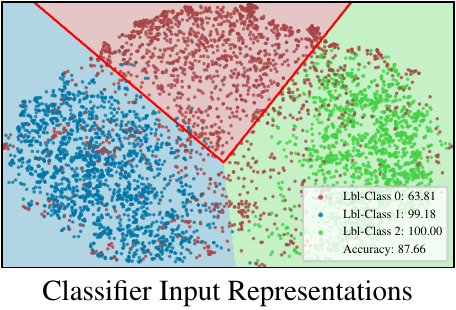} &
\includegraphics[width=0.17\textwidth]{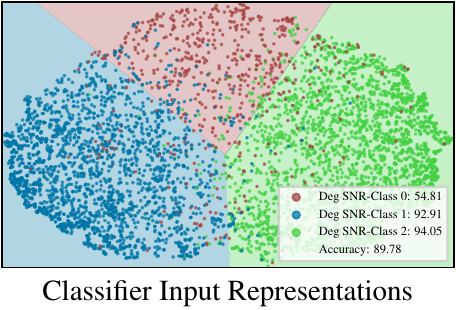} \\
\end{tabular}
\caption{The t-SNE visualization of a one-layer GraphSAGE under different distribution shifts: (top) the output representations given by the GNN encoder and (bottom) the node representations with the classifier decision boundaries. 
(a) indicates the source domain where the model is pretrained. 
(b), (c), and (d) stand for feature shift, conditional structure shift (CSS), and label shift. 
(e) represent the impact of SNR shift induced by degree changes.
The color of the nodes represents the ground-truth labels.
The legends in the bottom plots show the accuracy and recall scores for each class.
The red contours in the bottom (d) highlight the decision boundary of the minority class.
% A detailed discussion is provided in Sec. \ref{subsec:empirical_invest}.
}
\label{fig:diff_shift}
\end{figure*}

% \newpage
\subsection{Computational Efficiency}

\begin{table}[ht]
\vspace{-2mm}
\scriptsize
\caption{Computation time on MAG (CN $\rightarrow$ US). TTA methods are recorded in adaptation time per epoch. The Wall-Clock Time (seconds) are measured on a single NVIDIA RTX A6000 GPU.}
\label{tab:computation_time}
\begin{center}

\begin{adjustbox}{width = 0.8\textwidth}
\begin{tabular}{llcc}
\toprule
Method  & Stage & Computation Time (sec) & Additional Time (\%) \\
\midrule
 --  & Inference   & 0.518$\pm$0.038   & --                    \\
\midrule
% T3A         & Adaptation       & 0.368$\pm$0.046      & --                                \\
% \midrule
TSA-T3A     & Adaptation       
& 1.138$\pm$0.072
& 219.69\%                        \\
&  \hspace{1mm} - Neighborhood Alignment           
&0.037$\pm$0.011    
& 7.14\%                     
\\
&  \hspace{1mm} - SNR Adjustment   &0.729$\pm$0.023      
& 140.73\%                        
\\
&  \hspace{1mm} - Boundary Refinement (T3A)     
&0.368$\pm$0.046      
& 71.04\%                        \\
\midrule
GTrans      & Adaptation   & 4.309$\pm$0.171    & 831.85\%                         \\
\midrule
HomoTTT   & Adaptation &2.014$\pm$0.815 &388.80\% \\
\midrule
SOGA        & Adaptation  & 9.264$\pm$0.218   & 1788.42\%                       \\

\bottomrule
\end{tabular}
\end{adjustbox}
\end{center}
\end{table}

We quantify the computation time of TSA when adapting GraphSAGE on the MAG US dataset, containing the largest number of nodes and edges, and compare it with other GTTA methods:
\begin{itemize}
    \item \textbf{Inference time} refers to the time required for the model to generate predictions on the target graph $\gG^{\gT}$. This is also equivalent to the computation time for ERM.

    \item \textbf{Adaptation time} refers to the time required per epoch \emph{after} inference. It reflects the additional overhead introduced by the adaptation methods.

\end{itemize}

For a fair comparison, we report the computation time per epoch of adaptation. Note that GTrans alternates between optimizing adjacency structure and node features; in this case, we report the combined runtime of one graph feature update and one edge update as a single adaptation epoch.
Table~\ref{tab:computation_time} shows that TSA is more efficient than other GTTA baselines.
% 
% The results are shown in Table~\ref{tab:computation_time},
Notably, compared to GTTA baselines, \proj-T3A is approximately 2, 4 and 8 times more efficient than HomoTTT, GTrans and SOGA.
The time complexity of neighborhood alignment is $O(1)$ as it is a one-step edge weight assignment.
SNR adjustment updates a lightweight MLP and scalar bias terms.
While it requires backpropagation, it only introduces minimal overhead compared to full model retraining or extensive graph augmentation.

% A detailed discussion of computational efficiency is provided in Section~\ref{subsec:result_analysis}.

% \newpage
\subsection{Hyperparameter Sensitivity}
\label{app:hyperparameter_sensitivity}

\textbf{Hyperparameter analysis on $\rho_1$.} \proj employs neighborhood alignment based on the uncertainty of target pseudo labels. To evaluate the sensitivity of the hyperparameter $\rho_1$ (as defined in Section~\ref{subsec:1stalign} by the entropy threshold $-\sum_{i\in \gY}[\hat{y}]_i\ln([\hat{y}]_i)\leq\rho_1 \cdot \ln ( |\gY|)$), we vary $\rho_1$ within the range of $[-0.1, 0.1]$ centered around the value selected on the validation set. We set $\rho_1=1.0$ for the synthetic datasets and hence we vary it within $[0.9, 1.0]$. We perform evaluations on both synthetic and real-world datasets, using the best-performing variants: \proj-TENT, \proj-T3A and \proj-LAME for each case. \proj-TENT and \proj-T3A are evaluated with the GraphSAGE backbone, while \proj-LAME is evaluated with the GPRGNN backbone.
Overall, Figure~\ref{fig:hyper_rho1} 
% shows that \proj is more robust to variations in $\rho_1$ on synthetic datasets, but more sensitive on real-world datasets due to the greater task complexity. 
shows that all tested values of $\rho_1$ outperform the corresponding non-graph TTA baselines.

\textbf{Hyperparameter analysis on $\rho_2$.}
When optimizing $\alpha$ using the cross-entropy loss in Eq.~\ref{eq:loss}, real-world datasets often produce unreliable hard pseudo-labels.
To address this, we adopt the same strategy used for selecting $\rho_1$. 
Only hard pseudo-labels whose entropy satisfies $-\sum_{i\in \gY}[\hat{y}]_i\ln([\hat{y}]_i)\leq\rho_2 \cdot \ln ( |\gY|)$ are used to optimize $\alpha$.
% We vary $\rho_2$ within $\{0.1, 1.0\}$ for real-world datasets and within $[0.9, 1.0]$ for synthetic datasets.
To further analyze the effect of $\rho_2$, we vary it within the range $[-0.1, 0.1]$ centered around the value selected on the validation set.
We set $\rho_2=1.0$ for the synthetic datasets and hence we vary it within $[0.9, 1.0]$.
Figure~\ref{fig:hyper_rho2} shows that ~\proj is robust to variations in $\rho_2$, with all perturbations outperforming the corresponding non-graph TTA baselines.

\begin{figure*}[h]
\centering
\begin{tabular}{ccc}
(a) CSBM (Struct. Shift) &(b) MAG (US $\rightarrow$ FR)  &(c) Pileup (PU10$\rightarrow$30) \\
\includegraphics[width=0.30\textwidth]{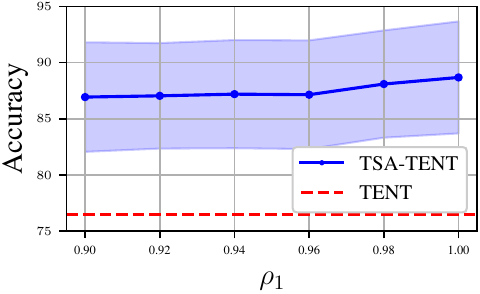}
&\includegraphics[width=0.30\textwidth]{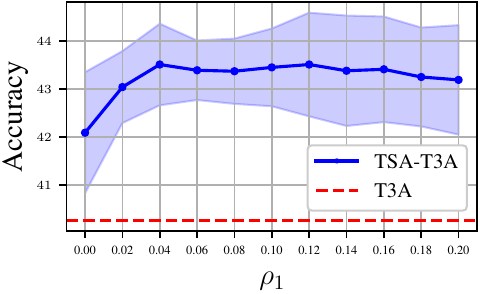}
&\includegraphics[width=0.30\textwidth]{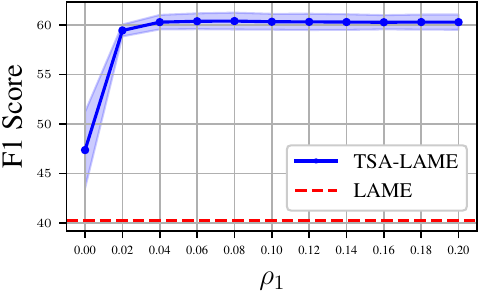} \\
\end{tabular}
\vspace{-2mm}
\caption{
Hyperparameter sensitivity analysis of $\rho_1$.}
\label{fig:hyper_rho1}
\vspace{-2mm}
\end{figure*}

\begin{figure*}[h]
\centering
\begin{tabular}{ccc}
(a) CSBM (Struct. Shift) &(b) MAG (US $\rightarrow$ FR)  &(c) Pileup (PU10$\rightarrow$30) \\
\includegraphics[width=0.30\textwidth]{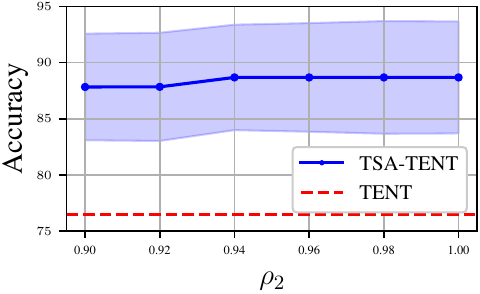}
&\includegraphics[width=0.30\textwidth]{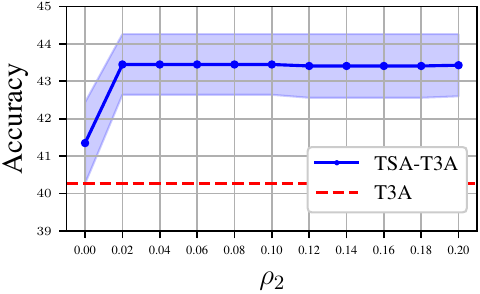}
&\includegraphics[width=0.30\textwidth]{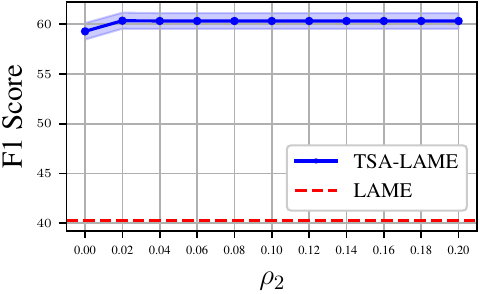}
 \\
\end{tabular}
\vspace{-2mm}
\caption{
Hyperparameter sensitivity analysis of $\rho_2$.}
\label{fig:hyper_rho2}
\vspace{-2mm}
\end{figure*}

\newpage

\begin{table}[t]
% \vspace{-2mm}
\scriptsize
\caption{GPRGNN results on MAG (accuracy). \textbf{Bold} indicates improvements in comparison to the corresponding non-graph TTA baselines. \underline{Underline} indicates the best model.}
% \vspace{-2mm}
\label{table:gprgnn_mag}
\begin{center}
\resizebox{\textwidth}{!}{%
\begin{tabular}{lcccccccccc}
\toprule
\textbf{Method} & US$\rightarrow$CN & US$\rightarrow$DE & US$\rightarrow$JP & US$\rightarrow$RU & US$\rightarrow$FR & CN$\rightarrow$US &CN$\rightarrow$DE & CN$\rightarrow$JP & CN$\rightarrow$RU & CN$\rightarrow$FR\\
\midrule
ERM &26.26$\pm$1.53 &27.09$\pm$0.74 &38.90$\pm$1.06 &22.88$\pm$1.14 &20.69$\pm$0.64 &27.58$\pm$0.55 &10.77$\pm$0.55 &17.41$\pm$0.86 &12.49$\pm$0.74 &9.05$\pm$0.42 
\\
GTrans &26.04$\pm$1.41 &26.96$\pm$0.72 &38.71$\pm$0.96 &23.07$\pm$1.08 &20.80$\pm$0.70 &27.49$\pm$0.46 &10.97$\pm$0.56 &17.41$\pm$0.78 &12.58$\pm$0.75 &9.08$\pm$0.41 \\
HomoTTT
&26.08$\pm$1.35 &28.12$\pm$0.85 &39.13$\pm$1.10 &24.16$\pm$1.17 &21.85$\pm$0.64 &25.50$\pm$0.39 &9.64$\pm$0.41 &15.95$\pm$0.62 &11.10$\pm$0.61 &8.08$\pm$0.31 
\\
SOGA &34.70$\pm$0.71 &35.18$\pm$0.33 &48.47$\pm$0.30 &33.63$\pm$0.73 &27.98$\pm$0.42 &35.16$\pm$0.57 &18.50$\pm$0.45 &28.35$\pm$1.04 &20.11$\pm$0.71 &14.73$\pm$0.58 
\\

ActMAD
&25.14$\pm$1.41 &26.77$\pm$0.69 &37.97$\pm$1.09 &22.30$\pm$0.99 &20.52$\pm$0.61 &26.89$\pm$0.57 &10.35$\pm$0.52 &16.90$\pm$0.85 &12.03$\pm$0.77 &8.71$\pm$0.42
\\
TENT &13.79$\pm$0.72 &25.67$\pm$0.50 &29.13$\pm$0.80 &18.00$\pm$0.14 &20.85$\pm$0.26 &16.29$\pm$0.82 &7.59$\pm$0.73 &13.81$\pm$0.74 &7.82$\pm$0.74 &5.94$\pm$0.34 

\\
LAME &24.54$\pm$2.55 &22.50$\pm$0.85 &35.93$\pm$1.42 &17.61$\pm$1.46 &15.71$\pm$0.92 &28.02$\pm$0.55 &9.08$\pm$0.61 &15.58$\pm$1.02 &10.87$\pm$0.70 &8.01$\pm$0.38 
\\
T3A &40.58$\pm$0.23 &39.00$\pm$0.85 &48.96$\pm$1.00 &46.13$\pm$1.56 &32.35$\pm$1.29 &38.92$\pm$2.23 &38.07$\pm$1.45 &46.23$\pm$1.73 &42.82$\pm$1.52 &28.65$\pm$0.65 
\\
\midrule
Matcha-TENT &\textbf{13.90$\pm$0.71} &\textbf{25.87$\pm$0.47} &\textbf{29.30$\pm$0.82} &\textbf{18.14$\pm$0.19} &\textbf{21.03$\pm$0.26} &\textbf{16.38$\pm$0.80} &\textbf{7.62$\pm$0.73} &\textbf{13.84$\pm$0.74} &\textbf{7.85$\pm$0.81} &\textbf{5.97$\pm$0.36}
\\
Matcha-LAME &8.93$\pm$4.76 &8.18$\pm$0.78 &15.46$\pm$1.34 &6.25$\pm$1.43 &6.69$\pm$0.81 &1.35$\pm$2.69 &\textbf{17.60$\pm$8.89} &\textbf{33.42$\pm$2.41} &\textbf{18.81$\pm$1.42} &4.23$\pm$8.02 
\\
Matcha-T3A &22.34$\pm$18.21 &32.18$\pm$1.70 &42.85$\pm$1.15 &34.77$\pm$2.62 &25.75$\pm$0.62 &\underline{\textbf{45.27$\pm$2.57}} &33.01$\pm$1.17 &42.08$\pm$0.82 &38.57$\pm$1.50 &\underline{\textbf{29.30$\pm$0.47}}
\\
\midrule
TSA-TENT 
 &\textbf{14.51$\pm$0.66} &\textbf{25.86$\pm$0.54} &\textbf{29.47$\pm$0.73} &\textbf{18.76$\pm$0.39} &\textbf{21.48$\pm$0.23} &\textbf{16.96$\pm$0.74} &\textbf{7.89$\pm$0.72} &\textbf{14.32$\pm$0.70} &\textbf{8.35$\pm$0.81} &\textbf{6.16$\pm$0.38}

\\
TSA-LAME
&\textbf{26.58$\pm$2.09} &\textbf{22.75$\pm$0.85} &\textbf{35.95$\pm$0.98} &\textbf{19.30$\pm$1.99} &\textbf{17.49$\pm$0.79} &\textbf{32.30$\pm$0.67} &\textbf{14.86$\pm$1.12} &\textbf{21.88$\pm$0.95} &\textbf{15.02$\pm$0.86} &\textbf{11.13$\pm$0.51}

\\
TSA-T3A 
&\underline{\textbf{40.77$\pm$0.15}} &\underline{\textbf{39.03$\pm$0.86}} &\underline{\textbf{49.38$\pm$0.89}} &\underline{\textbf{47.81$\pm$1.81}} &\underline{\textbf{32.43$\pm$1.24}} &\textbf{40.94$\pm$2.82} &\underline{\textbf{38.12$\pm$1.42}} &\underline{\textbf{46.26$\pm$1.81}} &\underline{\textbf{43.70$\pm$1.33}} &\textbf{29.07$\pm$0.46}

\\
\bottomrule
\end{tabular}
}
\end{center}
% \vspace{-2mm}
\end{table}

\begin{table*}[h]
\vspace{-2mm}
\caption{GPRGNN results on Pileup (f1-scores). \textbf{Bold} indicates improvements in comparison to the corresponding non-graph TTA baselines. \underline{Underline} indicates the best model.}
% \vspace{-4mm}
\label{table:gprgnn_pileup}
\begin{center}
\resizebox{0.9\textwidth}{!}{%
\small
\begin{tabular}{lcccccccc}
\toprule
\textbf{Method} & PU10$\rightarrow$30 & PU30$\rightarrow$10 & PU10$\rightarrow$50 & PU50$\rightarrow$10 & PU30$\rightarrow$140 & PU140$\rightarrow$30 &gg$\rightarrow$qq & qq$\rightarrow$gg \\
\midrule
ERM 
&50.60$\pm$3.52 &68.09$\pm$0.74  &32.60$\pm$5.55 &64.59$\pm$1.16 &0.68$\pm$0.40 &24.95$\pm$5.32 &71.00$\pm$0.61 &69.56$\pm$1.06 

\\
GTrans 
&51.08$\pm$3.06 &66.26$\pm$1.77  &33.94$\pm$5.39 &62.53$\pm$1.11 &0.75$\pm$0.52 &25.46$\pm$5.19 &70.11$\pm$1.13 &68.40$\pm$1.93 

\\
HomoTTT
&52.05$\pm$3.09 &67.46$\pm$0.94 &35.84$\pm$4.98 &63.59$\pm$1.11 &0.75$\pm$0.43 &23.03$\pm$5.46 &71.06$\pm$0.50 &69.72$\pm$1.10  

\\
SOGA 
&54.18$\pm$3.08 &66.13$\pm$0.70  &42.30$\pm$4.74 &59.97$\pm$1.70 &1.16$\pm$0.84 &18.68$\pm$5.49 &69.72$\pm$1.27 &68.39$\pm$1.75 

\\
ActMAD 
&58.80$\pm$1.58 &62.22$\pm$0.95 &49.22$\pm$2.21 &56.04$\pm$2.24 &2.16$\pm$0.75 &17.81$\pm$6.45 &70.40$\pm$0.94 &69.95$\pm$1.02  

\\
TENT 
&36.64$\pm$0.15 &2.04$\pm$1.05 &27.01$\pm$0.06 &0.62$\pm$0.81 &12.74$\pm$0.50 &3.77$\pm$3.30 &60.19$\pm$3.55 &71.19$\pm$0.58

\\
LAME 
&46.33$\pm$4.67 &67.61$\pm$0.96 &22.58$\pm$9.42 &64.58$\pm$1.00 &0.00$\pm$0.00 &20.94$\pm$7.27 &70.72$\pm$0.74 &69.67$\pm$0.88

\\
T3A 
&55.66$\pm$2.89 &71.36$\pm$0.79 &47.89$\pm$4.19 &71.01$\pm$0.38 &19.69$\pm$3.72 &54.38$\pm$2.56 &70.34$\pm$1.33 &68.63$\pm$1.51

\\
\midrule
Matcha-TENT &36.58$\pm$0.17 &0.00$\pm$0.00 &27.01$\pm$0.06 &0.00$\pm$0.00 &9.17$\pm$0.94 &0.00$\pm$0.00 &0.72$\pm$0.76 &20.30$\pm$4.03  

\\
Matcha-LAME 
&7.45$\pm$3.55 &0.00$\pm$0.00 &8.28$\pm$4.16 &0.00$\pm$0.00 &0.00$\pm$0.00 &0.00$\pm$0.00 &0.52$\pm$1.05 &6.56$\pm$4.08  

\\
Matcha-T3A &15.48$\pm$2.91 &0.00$\pm$0.00 
 &14.15$\pm$1.86 &0.00$\pm$0.00 &0.00$\pm$0.00 &9.60$\pm$19.19 &17.20$\pm$14.72 &22.77$\pm$5.23 

\\
\midrule
TSA-TENT 
% &39.90$\pm$1.18 &5.07$\pm$10.13 &27.78$\pm$0.48 &0.59$\pm$0.75 &20.31$\pm$0.65 &3.65$\pm$3.22 &67.97$\pm$1.60 &71.60$\pm$0.18 
&\textbf{39.90$\pm$1.18} &\textbf{5.07$\pm$10.13} &\textbf{27.78$\pm$0.48} &0.59$\pm$0.75 &\textbf{20.31$\pm$0.65} &3.65$\pm$3.22 &\textbf{67.97$\pm$1.60} &\underline{\textbf{71.60$\pm$0.18}} 

\\
TSA-LAME 
% &60.89$\pm$0.45 &71.54$\pm$0.33 &52.23$\pm$0.61 &70.62$\pm$0.35 &38.32$\pm$0.27 &20.49$\pm$7.76 &72.28$\pm$0.47 &71.23$\pm$0.27 
&\underline{\textbf{60.33$\pm$0.80}} &\textbf{71.54$\pm$0.33} &\textbf{52.23$\pm$0.61} &\textbf{70.62$\pm$0.35} &\underline{\textbf{38.32$\pm$0.27}} &20.49$\pm$7.76 &\underline{\textbf{72.28$\pm$0.47}} &\textbf{71.23$\pm$0.27} 

\\
TSA-T3A
% &60.76$\pm$0.49 &72.94$\pm$0.18
% &53.13$\pm$0.34 &72.63$\pm$0.23 &36.39$\pm$1.48 &62.31$\pm$0.28 &71.30$\pm$1.23 &69.70$\pm$1.04 
&\textbf{59.40$\pm$1.70} &\underline{\textbf{72.94$\pm$0.18}} &\underline{\textbf{53.13$\pm$0.34}} &\underline{\textbf{72.63$\pm$0.23}} &\textbf{36.39$\pm$1.48} &\underline{\textbf{62.31$\pm$0.28}} &\textbf{71.30$\pm$1.23} &\textbf{69.70$\pm$1.04} 

\\
\bottomrule
\end{tabular}
}
\end{center}
% \vspace{-2mm}
\end{table*}

\begin{table*}[h!]
\captionsetup{justification=centering}  % Center the caption text
\caption{GPRGNN results on Arxiv and DBLP/ACM  (accuracy). \textbf{Bold} indicates improvements in comparison to the corresponding non-graph TTA baselines. \underline{Underline} indicates the best model.}
% \vspace{-4mm}
\label{table:gprgnn_arxiv}
\begin{center}
\resizebox{0.9\textwidth}{!}{%
\small
\begin{tabular}{lcccccccc}
\toprule
& \multicolumn{2}{c}{1950-2007} &\multicolumn{2}{c}{1950-2009} &\multicolumn{2}{c}{1950-2011} &\multicolumn{2}{c}{DBLP \& ACM}\\
\textbf{Method} &2014-2016 &2016-2018 &2014-2016 &2016-2018 &2014-2016 &2016-2018 & D$\rightarrow$A & A$\rightarrow$D\\
\midrule
ERM &45.62$\pm$2.30 &40.35$\pm$4.46 &49.21$\pm$1.36 &44.78$\pm$2.70 &56.49$\pm$0.97 &54.14$\pm$1.31 &67.34$\pm$1.21 &56.47$\pm$2.84 

\\
GTrans &45.25$\pm$2.25 &40.51$\pm$4.31 &48.56$\pm$1.23 &44.57$\pm$2.55 &55.81$\pm$0.99 &53.59$\pm$1.34 &59.61$\pm$0.98 &59.27$\pm$1.79 

\\
HomoTTT &45.49$\pm$2.64 &40.03$\pm$4.80 &48.10$\pm$1.17 &42.33$\pm$2.34 &56.27$\pm$1.03 &53.25$\pm$1.24 &67.11$\pm$1.03 &57.18$\pm$2.54

\\
SOGA &48.40$\pm$2.29 &46.01$\pm$4.00 &52.59$\pm$1.38 &50.64$\pm$1.63 &\underline{58.61$\pm$1.15} &\underline{57.68$\pm$0.88} &67.00$\pm$1.03 &55.42$\pm$3.03 

\\
ActMAD &45.47$\pm$2.28 &40.12$\pm$4.42 &49.42$\pm$1.31 &44.96$\pm$2.59 &56.63$\pm$1.00 &54.25$\pm$1.25 &67.07$\pm$1.11 &55.83$\pm$2.76 

\\

TENT &44.73$\pm$2.63 &40.58$\pm$3.44 &48.47$\pm$0.94 &42.37$\pm$1.52 &57.15$\pm$0.98 &52.45$\pm$1.75 &61.56$\pm$0.77 &56.84$\pm$1.04 

\\
LAME  &45.75$\pm$3.96 &41.35$\pm$7.20 &49.52$\pm$2.24 &50.18$\pm$3.36 &55.65$\pm$1.09 &54.93$\pm$1.72 &66.92$\pm$0.99 &56.58$\pm$2.83 

\\
T3A &32.73$\pm$1.63 &31.97$\pm$1.75 &36.99$\pm$2.01 &34.10$\pm$1.32 &40.53$\pm$2.88 &39.04$\pm$3.86 &67.21$\pm$1.31 &65.32$\pm$1.69 

\\
\midrule
Matcha-TENT &43.89$\pm$2.77 &39.30$\pm$3.63 &48.01$\pm$0.90 &41.91$\pm$1.66 &56.28$\pm$1.30 &51.27$\pm$1.64 &\textbf{64.43$\pm$0.90} &\underline{\textbf{78.61$\pm$1.38}}
\\
Matcha-LAME &32.93$\pm$3.30 &24.13$\pm$5.06 &37.41$\pm$2.06 &24.16$\pm$2.92 &48.29$\pm$3.27 &44.74$\pm$4.20 &\textbf{67.83$\pm$1.71} &\textbf{70.39$\pm$2.15}

\\
Matcha-T3A &24.81$\pm$2.32 &19.06$\pm$3.09 &25.85$\pm$3.54 &19.60$\pm$4.56 &26.90$\pm$1.55 &18.74$\pm$4.37 &\textbf{67.59$\pm$1.87} &\textbf{72.91$\pm$2.69}
\\
\midrule
TSA-TENT &\textbf{45.17$\pm$2.84} &\textbf{41.07$\pm$3.67} &\textbf{49.13$\pm$0.84} &\textbf{43.19$\pm$1.24} &\textbf{57.66$\pm$1.00} &\textbf{53.00$\pm$1.65} &\textbf{62.77$\pm$0.84} &\textbf{61.71$\pm$0.99}
\\
TSA-LAME &\underline{\textbf{49.75$\pm$4.12}} &\underline{\textbf{48.67$\pm$7.74}} &\underline{\textbf{53.82$\pm$0.53}} &\underline{\textbf{55.18$\pm$1.21}}&\textbf{57.67$\pm$0.98} &\textbf{57.14$\pm$1.83} &\textbf{67.71$\pm$1.10} &\textbf{65.20$\pm$1.70}
\\
TSA-T3A &\textbf{39.11$\pm$1.73} &\textbf{39.40$\pm$3.69} &\textbf{41.15$\pm$6.33} &\textbf{44.65$\pm$2.31} &\textbf{44.14$\pm$2.63} &\textbf{40.82$\pm$5.28} &\underline{\textbf{68.02$\pm$1.40}} &\textbf{74.64$\pm$1.38}
\\
\bottomrule
\end{tabular}
}
\end{center}
\vspace{10mm}
\end{table*}

\begin{table}[t]
\scriptsize
\caption{GCN results on MAG (accuracy). \textbf{Bold} indicates improvements in comparison to the corresponding non-graph TTA baselines. \underline{Underline} indicates the best model.}
\vspace{-2mm}
\label{table:gcn_mag}
\begin{center}
\resizebox{\textwidth}{!}{%
\begin{tabular}{lcccccccccc}
\toprule
\textbf{Method} & US$\rightarrow$CN & US$\rightarrow$DE & US$\rightarrow$JP & US$\rightarrow$RU & US$\rightarrow$FR & CN$\rightarrow$US &CN$\rightarrow$DE & CN$\rightarrow$JP & CN$\rightarrow$RU & CN$\rightarrow$FR\\
\midrule
ERM &27.76$\pm$1.49 &30.70$\pm$0.96 &40.12$\pm$0.73 &27.70$\pm$0.75 &23.03$\pm$0.80 &36.46$\pm$2.00 &21.03$\pm$1.56 &28.95$\pm$2.43 &20.53$\pm$2.07 &16.68$\pm$1.54 
\\
GTrans  &27.44$\pm$1.48 &30.59$\pm$0.99 &39.65$\pm$0.97 &27.79$\pm$0.72 &22.95$\pm$0.88 &35.88$\pm$2.09 &20.80$\pm$1.59 &28.13$\pm$2.49 &20.53$\pm$2.06 &16.52$\pm$1.51

\\
HomoTTT &23.64$\pm$2.09 &28.49$\pm$1.29 &35.34$\pm$1.87 &26.21$\pm$1.62 &23.78$\pm$0.56 &29.27$\pm$2.91 &18.60$\pm$1.59 &23.10$\pm$2.01 &21.71$\pm$1.70 &16.03$\pm$1.33 

\\
SOGA &28.36$\pm$2.23 &35.61$\pm$1.50 &45.29$\pm$2.08 &36.79$\pm$1.18 &29.30$\pm$0.90 &44.09$\pm$1.92 &26.78$\pm$1.26 &37.93$\pm$2.05 &33.24$\pm$1.49 &24.23$\pm$1.01 

\\
ActMAD
&28.11$\pm$1.34 &31.25$\pm$0.90 &40.71$\pm$0.65 &28.61$\pm$0.63 &23.82$\pm$0.74 &36.07$\pm$2.01 &20.73$\pm$1.50 &28.76$\pm$2.33 &20.92$\pm$2.04 &16.65$\pm$1.46
\\
TENT &26.66$\pm$0.57 &33.38$\pm$0.32 &41.84$\pm$0.59 &33.01$\pm$0.54 &29.17$\pm$0.20 &28.82$\pm$0.85 &16.35$\pm$0.76 &25.20$\pm$0.85 &23.45$\pm$0.87 &14.96$\pm$0.56 
% 0.01
\\
LAME &28.20$\pm$3.24 &30.90$\pm$1.77 &42.25$\pm$0.97 &27.20$\pm$2.17 &20.83$\pm$1.08 &35.89$\pm$2.53 &19.31$\pm$2.11 &30.01$\pm$4.10 &17.46$\pm$2.39 &14.24$\pm$2.02 

\\
T3A &41.53$\pm$0.93 &45.41$\pm$1.76 &51.11$\pm$1.18 &48.02$\pm$1.99 &40.09$\pm$1.53 &43.72$\pm$2.31 &39.68$\pm$2.25 &47.31$\pm$0.84 &45.93$\pm$1.14 &30.95$\pm$1.20 

\\
\midrule
TSA-TENT 
% &27.46$\pm$0.53 &33.89$\pm$0.31 &42.34$\pm$0.60 &33.72$\pm$0.53 &29.59$\pm$0.21 &29.33$\pm$0.86 &16.82$\pm$0.75 &25.90$\pm$0.89 &24.03$\pm$0.81 &15.32$\pm$0.56
&\textbf{27.46$\pm$0.53} &\textbf{33.89$\pm$0.31} &\textbf{42.34$\pm$0.60} &\textbf{33.72$\pm$0.53} &\textbf{29.59$\pm$0.21} &\textbf{29.33$\pm$0.86} &\textbf{16.82$\pm$0.75} &\textbf{25.90$\pm$0.89} &\textbf{24.03$\pm$0.81} &\textbf{15.32$\pm$0.56}

\\
TSA-LAME 
% &28.64$\pm$3.24 &31.01$\pm$1.91 &42.35$\pm$0.86 &27.22$\pm$1.87 &22.27$\pm$0.52 &39.64$\pm$2.08 &22.96$\pm$1.62 &33.94$\pm$3.14 &20.08$\pm$1.93 &17.36$\pm$1.91
&\textbf{28.64$\pm$3.24} &\textbf{31.01$\pm$1.91} &\textbf{42.35$\pm$0.86} &\textbf{27.22$\pm$1.87} &\textbf{22.27$\pm$0.52} &\textbf{39.64$\pm$2.08} &\textbf{22.96$\pm$1.62} &\textbf{33.94$\pm$3.14} &\textbf{20.08$\pm$1.93} &\textbf{17.36$\pm$1.91}

\\
TSA-T3A 
% &41.61$\pm$0.95 &46.13$\pm$1.83 &51.81$\pm$1.22 &48.87$\pm$1.43 &42.98$\pm$1.90 &45.48$\pm$2.69 &40.11$\pm$2.74 &47.45$\pm$0.42 &45.95$\pm$1.17 &32.03$\pm$1.30 
&\underline{\textbf{41.61$\pm$0.95}} &\underline{\textbf{46.13$\pm$1.83}} &\underline{\textbf{51.81$\pm$1.22}} &\underline{\textbf{48.87$\pm$1.43}} &\underline{\textbf{42.98$\pm$1.90}} &\underline{\textbf{45.48$\pm$2.69}} &\underline{\textbf{40.11$\pm$2.74}} &\underline{\textbf{47.45$\pm$0.42}} &\underline{\textbf{45.95$\pm$1.17}} &\underline{\textbf{32.03$\pm$1.30}}

\\
\bottomrule
\end{tabular}
}
\end{center}
\vspace{-2mm}
\end{table}

\begin{table*}[h!]
\vspace{-2mm}
\caption{GCN results on Pileup (f1-scores). \textbf{Bold} indicates improvements in comparison to the corresponding non-graph TTA baselines. \underline{Underline} indicates the best model.}
% \vspace{-4mm}
\label{table:gcn_pileup}
\begin{center}
\resizebox{0.9\textwidth}{!}{%
\small
\begin{tabular}{lcccccccc}
\toprule
\textbf{Method} & PU10$\rightarrow$30 & PU30$\rightarrow$10 & PU10$\rightarrow$50 & PU50$\rightarrow$10 & PU30$\rightarrow$140 & PU140$\rightarrow$30 &gg$\rightarrow$qq & qq$\rightarrow$gg \\
\midrule
ERM 
&8.30$\pm$1.79 &29.20$\pm$4.57  &2.27$\pm$0.70 &26.61$\pm$6.79 &0.33$\pm$0.13 &16.08$\pm$10.46 &23.71$\pm$2.49 &26.73$\pm$4.79 

\\
GTrans 
&11.97$\pm$2.67 &28.88$\pm$4.71 &3.90$\pm$1.26 &27.20$\pm$6.88 &0.91$\pm$0.29 &16.98$\pm$11.06 &24.62$\pm$2.38 &28.88$\pm$5.14  

\\
HomoTTT
&8.94$\pm$2.05 &31.21$\pm$4.79  &2.54$\pm$0.81 &27.59$\pm$6.18 &0.55$\pm$0.18 &16.77$\pm$11.23 &24.84$\pm$3.27 &27.48$\pm$4.95 

\\
SOGA 
&5.81$\pm$2.20 &35.21$\pm$4.00  &1.00$\pm$0.44 &27.19$\pm$6.76 &0.54$\pm$0.20 &19.17$\pm$11.32 &25.38$\pm$1.66 &28.06$\pm$5.06 

\\
ActMAD 
&7.20$\pm$2.22 &30.45$\pm$4.13 &1.15$\pm$0.41 &24.35$\pm$6.69 &0.31$\pm$0.14 &16.67$\pm$9.32 &23.77$\pm$1.78 &27.36$\pm$4.75  

\\
TENT 
&2.45$\pm$0.93 &35.99$\pm$3.12 &1.63$\pm$0.75 &26.20$\pm$7.51 &0.09$\pm$0.05 &13.84$\pm$8.89 &19.96$\pm$4.06 &29.58$\pm$4.46

\\
LAME 
&7.48$\pm$2.17 &30.91$\pm$5.58  &1.46$\pm$0.75 &24.97$\pm$9.02 &0.06$\pm$0.06 &15.30$\pm$10.74 &21.50$\pm$5.24 &25.72$\pm$6.51

\\
T3A 
&41.55$\pm$13.17 &62.80$\pm$2.48  &29.92$\pm$7.29 &56.32$\pm$14.68 &4.45$\pm$2.48 &33.26$\pm$4.35 &60.15$\pm$10.72 &62.77$\pm$8.69 

\\
\midrule
TSA-TENT 
&\textbf{2.96$\pm$1.12} &\textbf{38.45$\pm$3.83} &\textbf{2.17$\pm$0.82} &\textbf{26.85$\pm$7.89} &\textbf{0.24$\pm$0.16} &\textbf{15.20$\pm$10.17} &\textbf{24.22$\pm$3.88} &\textbf{33.47$\pm$4.43} 

\\
TSA-LAME 
&\textbf{12.71$\pm$2.29} &\textbf{32.37$\pm$8.02} &\textbf{4.14$\pm$1.48} &\textbf{27.14$\pm$10.60} &\textbf{0.38$\pm$0.19} &\textbf{16.53$\pm$12.07} &\textbf{29.29$\pm$2.32} &\textbf{27.72$\pm$4.22} 

\\ 
TSA-T3A
&\underline{\textbf{43.86$\pm$13.14}} &\underline{\textbf{64.00$\pm$1.56}} &\underline{\textbf{33.31$\pm$8.50}} &\underline{\textbf{60.22$\pm$8.38}} &\underline{\textbf{7.12$\pm$4.19}} &\underline{\textbf{35.72$\pm$4.56}} &\underline{\textbf{62.03$\pm$8.01}} &\underline{\textbf{64.02$\pm$6.87}} 

\\
\bottomrule
\end{tabular}
}
\end{center}
\vspace{-2mm}
\end{table*}

\begin{table*}[h!]
% \vspace{-2mm}
\captionsetup{justification=centering}  % Center the caption text
\caption{GCN results on Arxiv and DBLP/ACM (accuracy). \textbf{Bold} indicates improvements in comparison to the corresponding non-graph TTA baselines. \underline{Underline} indicates the best model.}
% \vspace{-4mm}
\label{table:gcn_arxiv}
\begin{center}
\resizebox{0.9\textwidth}{!}{%
\small
\begin{tabular}{lcccccccc}
\toprule
& \multicolumn{2}{c}{1950-2007} &\multicolumn{2}{c}{1950-2009} &\multicolumn{2}{c}{1950-2011} &\multicolumn{2}{c}{DBLP \& ACM}\\
\textbf{Method} &2014-2016 &2016-2018 &2014-2016 &2016-2018 &2014-2016 &2016-2018 & D$\rightarrow$A & A$\rightarrow$D\\
\midrule
ERM 
&49.25$\pm$1.58 &50.64$\pm$3.11 &52.61$\pm$2.67 &51.92$\pm$4.12 &58.83$\pm$0.83 &57.73$\pm$1.27 &62.56$\pm$3.32 &47.59$\pm$1.92 

\\
GTrans 
&49.10$\pm$1.68 &50.57$\pm$3.10 &52.50$\pm$2.71 &51.87$\pm$4.26 &58.72$\pm$0.89 &57.75$\pm$1.31 &62.03$\pm$2.31 &51.65$\pm$3.09

\\
HomoTTT
&46.22$\pm$1.54 &46.11$\pm$2.79 &48.96$\pm$3.13 &46.83$\pm$5.00 &55.10$\pm$1.63 &53.18$\pm$2.52 &63.21$\pm$3.31 &48.45$\pm$1.68

\\
SOGA 
&49.82$\pm$2.38 &52.36$\pm$2.54 &52.92$\pm$1.08 &\underline{54.45$\pm$1.15} &57.93$\pm$1.71 &\underline{58.66$\pm$2.02} &\underline{64.22$\pm$2.40} &53.29$\pm$2.08 

\\
ActMAD 
&49.24$\pm$1.57 &50.68$\pm$2.94 &52.71$\pm$2.64 &52.10$\pm$3.98 &58.83$\pm$0.86 &57.85$\pm$1.25 &62.68$\pm$3.36 &47.58$\pm$1.89

\\
TENT
&48.59$\pm$1.41 &49.90$\pm$1.15 &53.39$\pm$1.82 &53.45$\pm$2.11 &58.72$\pm$0.91 &58.23$\pm$0.95 &60.80$\pm$2.51 &48.11$\pm$1.43

\\
LAME  
&49.77$\pm$2.50 &50.97$\pm$4.60 &53.35$\pm$2.72 &52.70$\pm$4.25 &58.85$\pm$0.57 &57.05$\pm$1.57 &62.51$\pm$3.33 &47.64$\pm$1.70

\\
T3A 
&43.37$\pm$1.38 &41.32$\pm$3.13 &49.95$\pm$3.76 &47.11$\pm$6.73 &55.56$\pm$1.16 &53.57$\pm$1.94 &62.60$\pm$3.65 &49.65$\pm$1.83

\\
\midrule
TSA-TENT 
&\textbf{49.63$\pm$1.26} &\textbf{50.78$\pm$0.74} &\textbf{53.78$\pm$1.57} &\textbf{53.67$\pm$1.85} &\textbf{58.90$\pm$0.94} &\textbf{58.27$\pm$0.95} &\textbf{61.31$\pm$2.39} &\textbf{53.14$\pm$2.01}
\\
TSA-LAME 
&\underline{\textbf{50.95$\pm$2.05}} &\underline{\textbf{53.04$\pm$2.86}} &\underline{\textbf{54.49$\pm$1.81}} &\textbf{54.13$\pm$3.35} &\underline{\textbf{59.26$\pm$0.70}} &\textbf{57.55$\pm$1.30} &\textbf{63.92$\pm$2.76} &\underline{\textbf{55.19$\pm$4.50}}
\\
TSA-T3A 
&\textbf{45.37$\pm$0.76} &\textbf{45.46$\pm$1.92} &\textbf{51.57$\pm$2.73} &\textbf{49.70$\pm$5.44} &\textbf{56.14$\pm$1.02} &\textbf{54.84$\pm$1.17} &\textbf{63.31$\pm$3.50} &\textbf{54.53$\pm$3.07}
\\
\bottomrule
\end{tabular}
}
\end{center}
\vspace{-2mm}
\end{table*}

\subsection{Results of More GNN Architectures}

% Compared to GraphSAGE, we observe that GPRGNN is less robust under structure shifts due to its linear combination of k-hop representations.

\textbf{TSA is Model-Agonistic.} 
~\proj employs neighborhood alignment and SNR adjustment on edge weights and do not need any assumption on the GNN architecture such as hop-aggregation parameters.
We demonstrate that ~\proj is applicable to common GNN architectures such as GraphSAGE, GPRGNN, and GCN.
While different backbones yield varying ERM results, we observe that ~\proj consistently improves model performance.

\textbf{Comparison with Baselines.}  \proj exhibits strong outcomes on real-world datasets and with both GPRGNN and GCN backbones. In all cases, the best results are obtained by GTTA methods (\underline{underlined} across Table\ref{table:gprgnn_mag}–\ref{table:gcn_arxiv}), highlighting the importance of addressing structure shift in graph data.
For GPRGNN, ~\proj consistently outperforms GTrans and HomoTTT and surpasses SOGA by an average of 10.5\% and Matcha by an average of 19.5\%.
For GCN, ~\proj similarly outperforms GTrans and HomoTTT, and exceeds SOGA by an average of 13.1\%.
Furthermore, \proj shows prevailing gains over the corresponding non-graph TTA baselines (\textbf{bold} across Table\ref{table:gprgnn_mag}–\ref{table:gcn_arxiv}).
While Matcha sometimes demonstrates strong performance, we frequently observe unstable results that may even fall below ERM.
This drop is attributed to the reliance of their proposed PIC loss on minimizing intra-variance, which depends on the quality of node hidden representations in the target domain.
Under label shift, node representations can easily degrade, leading to instability in Matcha's performance.

%8.2 better than SOGA, 5.8 better than Matchha
%6.13 better than SOGA

\clearpage

\subsection{Sensitivity of $\gamma$ to Noise in Pseudo-Labels}

In this subsection, we investigate how sensitive $\gamma$ is to noise in pseudo-labels.
Specifically, we progressively add $x\%$ additional noise to the input features and evaluate the discrepancy between the estimated $\gamma$ and the ground-truth value calculated from the labels.
We report root mean squared error (RMSE).
The results are shown in Table \ref{tab:gamma_noise_sensitivity}.
Results indicate that the estimated $\gamma$ values are robust: the relative increase in error is much smaller than the injected input noise.
Even with $100\%$ additional noise, the increase in $\gamma$ estimation error is only about $2.6\%\!\sim\!8.3\%$ across methods.

\begin{table*}[h!]
  \centering
  \caption{Sensitivity of $\gamma$ estimation to pseudo-label noise on synthetic CSBMs. The percentages in parentheses denote the relative increase compared to the "No Add" setting.}
  \vspace{-2mm}
  \label{tab:gamma_noise_sensitivity}
  \begin{center}
  \resizebox{0.75\textwidth}{!}{%
    \small
    \begin{tabular}{lcccc}
      \toprule
      Method & No Add & Add 10\% & Add 50\% & Add 100\% \\
      \midrule
      TSA-TENT & 0.6938 & 0.6984 (\textbf{+0.66\%}) & 0.7159 (\textbf{+3.18\%}) & 0.7330 (\textbf{+5.64\%}) \\
      TSA-LAME & 2.4489 & 2.5033 (\textbf{+2.22\%}) & 2.6856 (\textbf{+9.66\%}) & 2.6516 (\textbf{+8.28\%}) \\
      TSA-T3A  & 1.4768 & 1.4791 (\textbf{+0.16\%}) & 1.4943 (\textbf{+1.18\%}) & 1.5155 (\textbf{+2.69\%}) \\
      \bottomrule
    \end{tabular}
  }
  \end{center}
\end{table*}

\subsection{TSA on the Heterophilic Graphs}

TSA is in principle applicable to heterophilic datasets because it makes no homophily-specific assumptions. 
We include an experiment adapting between homophilic and heterophilic CSBMs to demonstrate the model’s adaptability empirically. 
Specifically, we adapt a source model first trained on a heterophilic dataset to a homophilic dataset, and vice versa. 
We control neighborhood shifts to synthesize data with high heterophily. 
The results are shown in Table \ref{tab:hetero_homo_adapt}.
Across both directions, TSA consistently improves performance and attains the best accuracy. 
By contrast, SOGA relies on homophily assumptions and hence fails to adapt effectively.
Its accuracy decreases under both settings.

\begin{table*}[h!]
  \centering
  \caption{We report performance when adapting between heterophilic and homophilic synthetic CSBMs. A higher node homophily ratio \cite{pei2020geom} indicates that a larger fraction of each node’s neighbors share its label. }
  \vspace{-2mm}
  \label{tab:hetero_homo_adapt}
  \resizebox{0.9\textwidth}{!}{%
    \small
    \begin{tabular}{lccccccc}
      \toprule
      Homophily ratio & ERM & GTrans & HomoTTT & SOGA & TSA-TENT & TSA-LAME & TSA-T3A \\
      \midrule
      $0.041 \rightarrow 0.775$ &
      29.22$\pm$18.20 & 29.98$\pm$17.46 & 29.75$\pm$18.28 & 26.88$\pm$17.94 &
      \textbf{49.47$\pm$16.55} & 33.12$\pm$14.63 & 39.51$\pm$17.68 \\
      \midrule
      $0.775 \rightarrow 0.041$ &
      36.17$\pm$1.98 & 35.29$\pm$3.00 & 36.02$\pm$2.63 & 35.14$\pm$1.88 &
      \textbf{39.11$\pm$3.88} & 37.13$\pm$1.04 & 36.95$\pm$1.87 \\
      \bottomrule
    \end{tabular}
  }
\end{table*}

% \newpage
\section{DATASET DETAILS}

\label{app:dataset}

\subsection{Dataset Statistics} 
Here we summarize the statistics of the four real-world graph datasets in Table \ref{table:arxiv_datastats}, \ref{table:mag_datastats}, and \ref{table:pileup_datastats}.
The edges are counted twice in the edge list as they are undirected graphs. 

\begin{table}[h!]
\caption{Arxiv and ACM/DBLP dataset statistics}
\vspace{-2mm}
\begin{center}
\begin{tabular}{lccccccc}
\toprule
            & ACM          & DBLP           & 1950-2007         &1950-2009  &1950-2011  &1950-2016          &1950-2018 \\
\midrule
Nodes          & $7410$  & $5578$ & $4980$ &  $9410$  &$17401$  & $69499$ & $120740$ \\
Edges        & $11135    $ & $7341$ & $5849$    &$13179$   &    $30486$     & $232419$ & $615415$ \\
Features    & $7537    $ & $7537  $ & $128$   &    $128$  & $128$   &    $128$   &    $128$   \\
Classes        & $6    $ & $6  $ & $40$    &$40$  & $40$    &$40$     &    $40$ \\
\bottomrule
\label{table:arxiv_datastats}
\end{tabular}
\end{center}
\vskip -0.7cm
\end{table}

% \begin{table}[h!]
% \caption{real dataset statistics}
% \vspace{2mm}
% \begin{center}
% % \begin{sc}
% \begin{tabular}{lcccccc}
% \toprule
%             & ACM          & DBLP           & Arxiv-2007          &Arxiv-2009  & Arxiv-2016          &Arxiv-2018\\
% \midrule
% $\#$nodes          & $7410$  & $5578$ & $4980$ &  $9410$ & $69499$ & $120740$\\
% $\#$edges        & $11135    $ & $7341$ & $5849$    &$13179$       & $232419$ & $615415$\\
% Node feature dimension    & $7537    $ & $7537  $ & $128$   &    $128$  & $128$   &    $128$   \\
% $\#$labels        & $6    $ & $6  $ & $40$    &$40$  & $40$    &$40$      \\
% \bottomrule
% \label{table:datastats}
% \end{tabular}
% % \end{sc}
% \end{center}
% \vskip -0.7cm
% \end{table}

\begin{table}[h!]
\caption{MAG dataset statistics}
\vspace{-2mm}
\begin{center}
% \begin{sc}
\begin{tabular}{lcccccc}
\toprule
            & US          & CN           & DE          &JP  & RU          &FR\\
\midrule
Nodes          & $132558$  & $101952$ & $43032$     &  $37498$  & $32833$     &  $29262$   \\
Edges        & $697450$ & $285561$ & $126683$    &$90944$    & $67994$     &  $78222$    \\
Features    & $128    $ & $128  $ & $128$   &    $128$  & $128$   &    $128$   \\
Classes        & $20$ & $20  $ & $20$    &$20$  & $20$    &$20$ \\
\bottomrule
\label{table:mag_datastats}
\end{tabular}
% \end{sc}
\end{center}
\vskip -0.7cm
\end{table}

\begin{table}[h!]
\caption{Pileup dataset statistics}
\vspace{-2mm}
\begin{center}
% \begin{sc}
\begin{tabular}{lcccccc}
\toprule
            & gg-10          & qq-10           & gg-30           &gg-50  & gg-140          \\
\midrule
Nodes          & $18611$  & $17242$ & $41390$     & $60054$     &  $154750$     \\
Edges        & $ 53725 $ & $42769$ & $173392$      & $341930$    &$2081229$    \\
Features    & $28    $ & $28  $ & $28$    & $28$   &    $28$   \\
Classes        & $2    $ & $2  $ & $2$      & $2$    &$2$      \\
\bottomrule
\label{table:pileup_datastats}
\end{tabular}
% \end{sc}
\end{center}
\vskip -0.7cm
\end{table}

\subsection{Dataset Shift Statistics}
\label{app:shift_stats}
Below we provide two metrics indicating the neighborhood and label shift in real-world dataset.
We measure the CSS by computing the weighted average of the total variation distance between of the neighborhood distribution:
\begin{align*}
    \text{CSS}\coloneqq\frac{1}{2}\sum_{j\in\gY}\sum_{k\in\gY}\probt(Y_u=j)\abs{\probs(Y_v=k|Y_u=j,v\in \neighbor_u) - \probt(Y_v=k|Y_u=j,v\in \neighbor_u)}
\end{align*}

The label shift is measured as the total variation distance between the label distribution, which corresponds to the first term in Theorem \ref{theory:decomposeerror}.

\begin{align*}
    \text{Label Shift}\coloneqq\frac{1}{2}\sum_{j\in\gY}\abs{\probs(Y_u=j) - \probt(Y_u=j)}
\end{align*}

\begin{table}[h!]
\caption{MAG dataset shift metrics}
\vspace{-2mm}
\begin{center}
\begin{adjustbox}{width = 1\textwidth}
\begin{sc}
\begin{tabular}{lcccccccccc}
\toprule
  &   $US \rightarrow CN$ &  $US\rightarrow DE$  &     $US\rightarrow JP$    & $US\rightarrow RU$     & $US\rightarrow FR$ & $CN \rightarrow US$ &  $CN\rightarrow DE$  &     $CN\rightarrow JP$  & $CN\rightarrow RU$     & $CN\rightarrow FR$\\
\midrule
CSS        & $0.2137$ & $0.1902$ & $0.1476$    &$0.2855$    & $0.2238$     &  $0.1946$ & $0.2053$ &$0.1361$ &$0.2025$ & $0.2039$    \\
Label Shift        & $0.2734$ & $0.1498$ & $0.1699$    &$0.3856$    & $0.1706$     &  $0.2734$ & $0.2691$ &$0.1522$ &$0.2453$ & $0.2256$    \\

\bottomrule
\label{table:magstats}
\end{tabular}
\end{sc}
\end{adjustbox}
\end{center}
\vskip -0.7cm
\end{table}

\begin{table}[h!]
\caption{Pileup dataset shift metrics}
\vspace{-2mm}
\begin{center}
\begin{adjustbox}{width = 0.95\textwidth}
\begin{sc}
\begin{tabular}{lcccccccc}
\toprule
& \multicolumn{6}{c}{Pileup Conditions} & \multicolumn{2}{c}{Physical Processes} \\
 Domains  & $\text{PU}10 \rightarrow 30$    & $\text{PU}30 \rightarrow 10$ & PU$10\rightarrow 50$ & PU$50\rightarrow 10$ & PU$30\rightarrow 140$ & PU$140\rightarrow 30$ &$gg \rightarrow qq$      &     $qq \rightarrow gg$      \\
\midrule
CSS        & $0.1493$ & $0.1709$ & $0.2039$    &$0.2595$    & $0.1293$     &  $0.1635$  & $0.0214$     &  $0.0205$   \\
Label Shift       & $0.2425$ & $0.2425$ & $0.2982$    &$0.2982$    & $0.1377$     &  $0.1377$  & $0.0341$     &  $0.0348$  \\

\bottomrule
\label{table:hepstats}
\end{tabular}
\end{sc}
\end{adjustbox}
\end{center}
\vskip -0.7cm
\end{table}

\begin{table}[h!]
\caption{Real dataset shift metrics}
\vspace{-2mm}
\begin{center}
\begin{adjustbox}{width = 0.95\textwidth}
\begin{sc}
\begin{tabular}{lcccccccc}
\toprule
 & \multicolumn{2}{c}{1950-2007} & \multicolumn{2}{c}{1950-2009} & \multicolumn{2}{c}{1950-2011} & \multicolumn{2}{c}{DBLP/ACM}\\
 Domains  & $2014-2016$    & $2016-2018$ & $2014-2016$ & $2016-2018$ & $2014-2016$ & $2016-2018$   & $D\rightarrow A$ & $A\rightarrow D$\\
\midrule
CSS         
& $0.3005$ & $0.3450$ & $0.2583$    &$0.3120$    & $0.1833$     &  $0.2567$   & $0.1573$     &  $0.2227$   \\
Label Shift        
& $0.2938$ & $0.4396$ & $0.2990$    &$0.4552$    & $0.2853$     &  $0.4438$   & $0.3434$     &  $0.3435$ \\

\bottomrule
\label{table:realstats}
\end{tabular}
\end{sc}
\end{adjustbox}
\end{center}
\vskip -0.7cm
\end{table}

\newpage
\section{DETAILED EXPERIMENTAL SETUP}
\label{app:detailed experimental setup}

\subsection{Datasets Setup}
For Arxiv, MAG, and DBLP/ACM we follow the setting of \cite{liu2024pairwise}. Hence we discuss the setting of synthetic dataset and pileup in the following:

\paragraph{Synthetic Dataset.}
The synthetic dataset is generated by the contextual stochastic block model (CSBM) \cite{deshpande2018contextual}. 
The generated graph contains 6000 nodes and 3 classes.
We alter the edge connection probability matrix $\mathbf{B}$ and the label ratio $\prob_Y$
Specifically, the edge connection probability matrix $\mathbf{B}$ is a symmetric matrix given by
\[
\mathbf{B} = 
\begin{bmatrix}
p & q & q \\
q & p & q \\
q & q & p
\end{bmatrix}
\]
where $p$ and $q$ indicate the intra- and the inter-class edge probability respectively.
We assume the node features are sampled from a Gaussian distribution $X_u \sim \gN(\bm{\mu}_u, \sigma\mI)$ where we set $\sigma=0.3$ as well as $\bm{\mu}_u=[1,0,0]$ for class 0, $\bm{\mu}_u=[0,1,0]$ for class 1, and $\bm{\mu}_u=[0,0,1]$ for class 2.

\begin{itemize}
    \item The source graph for setting 1-6 has $\prob_Y = \left[ 0.1,0.3,0.6 \right]$, with $p = 0.01$ and $q = 0.0025$.
    
    \item For neighborhood shift we fix the same class ratio but change the connection probability:
    \begin{itemize}
        \item Condition 1: $p = 0.005$, $q = 0.00375$.
        \item Condition 2: $p = 0.005$, $q = 0.005$.
    \end{itemize}

    \item For condition 3 and 4 we  introduce degree shift for SNR discrepancy.
    \begin{itemize}
        \item Condition 3: $p = \frac{0.005}{2}$, $q = \frac{0.00375}{2}$.
        \item Condition 4: $p = \frac{0.005}{2}$, $q = \frac{0.005}{2}$.
    \end{itemize}

    \item For condition 5 and 6 we investigate structure shift under training from give training is from the imbalanced source label ratio.
    \begin{itemize}
        \item Condition 5: $\mathbb{P}_Y =  \left[ 1/3, 1/3, 1/3\right]$ $p = \frac{0.005}{2}$, $q = \frac{0.00375}{2}$.
        \item Condition 6: $\mathbb{P}_Y =  \left[ 1/3, 1/3, 1/3\right]$ $p = \frac{0.005}{2}$, $q = \frac{0.005}{2}$.
    \end{itemize}

    \item The source graph for Condition 7-8 has $\prob_Y = \left[ 1/3, 1/3, 1/3\right]$, with $p = 0.01$ and $q = 0.0025$.

    \item For condition 7 and 8 we investigate structure shift under training from balanced source ratio. The $p$ and $q$ in Condition 7 and 8 is the same as in Condition 5 and 6, respectively.
    \begin{itemize}
        \item Condition 7: $\mathbb{P}_Y = [0.1, 0.3, 0.6]$, $p = \frac{0.005}{2}$, $q = \frac{0.00375}{2}$.
        \item Condition 8: $\mathbb{P}_Y = [0.1, 0.3, 0.6]$, $p = \frac{0.005}{2}$, $q = \frac{0.005}{2}$.
    \end{itemize}
\end{itemize}

\paragraph{Pileup.} Under the study of different PU levels we use the data from signal $gg$. Compared to the Pileup datasets used in the \cite{liu2024pairwise}, we explicitly label all types of particles to better align the neighborhood distribution. Note that the charged particle can achieve all most perfect recall as their label information is encoded in their feature, hence the four-class classification still reduces into a binary classification between LC neutral particles and OC neutral particles.
For the study of different data generation process, we use PU10 from signal $gg\rightarrow qq$ and $qq \rightarrow gg$.

\subsection{Pretraining Setup}
\paragraph{Model Architecture and Pretraining.} 
We use GraphSAGE as the model backbone, with a 3-layer mean pooling GNN as the encoder and a 2-layer MLP with a batch normalization layer as the classifier.
The hidden dimension of the encoder is set to 300 for Arxiv and MAG, 50 for Pileup, 128 for DBLP/ACM, and 20 for the synthetic datasets.
The hidden dimension of the classifier is set to 300 for Arxiv and MAG, 50 for Pileup, 40 for DBLP/ACM, and 20 for the synthetic datasets.
For the GPRGNN backbone, we follow the same configuration but increase the number of encoder layers to 5.
For the GCN backbone, we use a 3-layer GCN as the feature encoder, followed by the same 2-layer MLP classifier as in the other backbones.
All models are trained for 400 epochs with the learning rate set to 0.003. 
We use Adam~\cite{kingma2014adam} as the optimizer for all models, and apply a learning rate scheduler with a decay rate of 0.9.
All experiments are repeated 5 times under different initializations and data splits to ensure consistency.

\paragraph{Hardware and Frameworks.} 
All experiments are conducted on a server with a single NVIDIA RTX 6000 GPU and an AMD EPYC 7763 64-Core Processor. We use PyTorch 2.2.2 \cite{paszke2017automatic} with CUDA 12.1 as the framework. For graph neural networks, we use PyTorch Geometric 2.5.2 \cite{fey2019fast}.

\subsection{Evaluation and Metric}
The source graph is splitted into train/val/test 60/20/20 during the pretraining stage. The target graph is splitted into labeled/unlabeled 3/97 during test time. 
We use the the labeled nodes in the target graph to do hyperparameter tuning and select the model with the optimal validation score.
We follow the metrics in~\cite{liu2024pairwise} for evaluation. MAG, Arxiv, DBLP/ACM, and synthetic datasets are evaluated with accuracy.
For the MAG datasets, only the top 19 classes are evaluated.
For the Pileup datasets, we evaluate the performance with f1 score.

\subsection{Hyperparameter Tuning}
Below we introduce the search space of the hyperparameters. LAME is a parameter-free approach so we do not tune it.

\begin{itemize}
\item TSA introduces three hyperparameters (1) the learning rate $lr$ for optimizing $\alpha$, (2) the ratio $\rho_1$ for reliable assignment of $\gamma$ based on entropy $H(\hat{y})\leq \rho_1 \cdot \ln ( |\gY|)$, and (3) the ratio $\rho_2$ for filtering out unreliable hard pseudo-labels in Eq.~\ref{eq:loss} based on entropy $H(\hat{y})\leq \rho_2 \cdot \ln ( |\gY|)$.
We select $lr$ from $\{0.001, 0.01, 0.05, 0.1 \}$. 
For real world datasets, we select $\rho_1$ from $\{0, 0.01, 0.1, 0.5\}$, and $\rho_2$ from $\{0.1, 1.0\}$.
For synthetic datasets, we set $\rho_1$ and $\rho_2$ to $1.0$.
We present the results updated by one epoch.
We observe that given the same domain shifts, the optimal hyperparameters of $\rho_1$ may differ due to different performance of the boundary refinement approaches (TENT, T3A, and LAME). 

\item Matcha introduces two hyperparameters learning rate $lr$ and the number of epochs $T$. Based on their hyperparameter study, We select $lr$ from $\{0.1, 1, 5, 10 \}$ and $T$ from $\{2, 10\}$.

\item GTrans introduces four hyperparameters learning rate of feature adaptation $lr_f$, learning rate of structure adaptation $lr_A$, the number of epochs $T$, and the budget to update. We select $lr_f$ from $\{0.001, 0.01 \}$, $lr_A=\{0.01,0.1,0.5\}$ from $\{0.01, 0.05 \}$, $T$ from $\{5, 10\}$, and budget from $\{0.01, 0.05, 0.1 \}$.

\item HomoTTT introduces two hyperparameter $lr$ and $k0$. The hyperparameter $k0$ is a scaling factor for the model selection score. We present the results updated by one epoch. We select $lr$ from $\{0.001, 0.01 \}$ and $k0$ from $\{0.001, 0.01 \}$.

\item SOGA introduces one hyperparameter $lr$. We present the results updated by one epoch and select $lr$ from $\{0.001, 0.01 \}$.

\item TENT introduces one hyperparameter $lr$.  We present the results updated by one epoch and select $lr$ from $\{0.001, 0.01, 0.05\}$. Note that TENT is applied only to the classifier in our implementation, as GNNs typically do not include batch normalization layers.

\item T3A introduces one hyperparameter $M$ for deciding the number of supports to restore. We select $M$ from $\{5, 20, 50, 100\}$

\item ActMAD introduces one hyperparameter $lr$. We present the results updated by one epoch and select $lr$ from $\{0.001, 0.01 \}$. We adapt ActMAD to GNNs by aligning the outputs of each GNN layer before the classifier.
\end{itemize}

\end{document}